\documentclass[letterpaper]{article} 
\usepackage{aaai23}  
\usepackage{times}  
\usepackage{helvet}  
\usepackage{courier}  
\usepackage[hyphens]{url}  
\usepackage{graphicx} 
\usepackage{natbib}  
\usepackage{caption} 
\usepackage{algorithm}
\usepackage{algorithmic}

\usepackage{amsmath,amsfonts,bm}

\def\1{\bm{1}}

\usepackage{microtype}
\usepackage{graphicx}
\usepackage{subfigure}
\usepackage{tabularx,booktabs} 

\usepackage{algorithm,algorithmic}

\usepackage{epsfig,amsmath,amssymb,amsfonts,amstext,amsthm,mathrsfs}
\usepackage{latexsym,graphics,epsf,epsfig,psfrag}
\usepackage{dsfont,color,epstopdf,fixmath}
\usepackage{enumitem}
\usepackage{dsfont}

\usepackage[utf8]{inputenc} 
\usepackage[T1]{fontenc}    

\usepackage{booktabs}       

\usepackage{hhline}

\usepackage{threeparttable}
\usepackage{booktabs, caption, makecell}

\usepackage{amsfonts}       
\usepackage{nicefrac}       
\usepackage{microtype}      
\usepackage{subfigure}
\usepackage{mathtools}
\usepackage{amssymb}
\usepackage{amsthm}
\usepackage{comment}

\newtheorem{theorem}{Theorem}
\newtheorem{assumption}{Assumption}
\newtheorem{lemma}{Lemma}
\newtheorem{definition}{Definition}{}
\newtheorem{remark}{Remark}
\newtheorem{proposition}{Proposition}

\newtheorem{example}{Example}

\allowdisplaybreaks[4]

\usepackage{newfloat}
\usepackage{listings}
\DeclareCaptionStyle{ruled}{labelfont=normalfont,labelsep=colon,strut=off} 
\floatstyle{ruled}
\newfloat{listing}{tb}{lst}{}
\floatname{listing}{Listing}
\title{Efficient Gradient Approximation Method for Constrained Bilevel Optimization}
\author {
    Siyuan Xu\textsuperscript{\rm 1},
    Minghui Zhu\textsuperscript{\rm 1}\thanks{Corresponding author.}
}
\affiliations {
    \textsuperscript{\rm 1} School of Electrical Engineering and Computer Science\\ The Pennsylvania State University, University Park, USA\\
    \{spx5032, muz16\}@psu.edu
}

\usepackage{bibentry}

\begin{document}

\maketitle

\begin{abstract} 
Bilevel optimization has been developed for many machine learning tasks with large-scale and high-dimensional data. This paper considers a constrained bilevel optimization problem, where the lower-level optimization problem is convex with equality and inequality constraints and the upper-level optimization problem is non-convex. The overall objective function is non-convex and non-differentiable. To solve the problem, we develop a gradient-based approach, called gradient approximation method, which determines the descent direction by computing several representative gradients of the objective function inside a neighborhood of the current estimate. We show that the algorithm asymptotically converges to the set of Clarke stationary points, and demonstrate the efficacy of the algorithm by the experiments on hyperparameter optimization and meta-learning. 

\end{abstract}

\section{Introduction}

A general constrained bilevel optimization problem is formulated as follows:
\begin{align}
\label{opt_intr}
&\min_{x \in \mathbb{R}^{d_x}}  \   \Phi(x)=f\left(x, y^{*}(x)\right) \nonumber \\
& \text { s.t. } \  r\left(x, y^{*}(x)\right) \leq 0; \ s\left(x, y^{*}(x)\right) = 0; \\
&y^{*}(x)=\underset{y \in \mathbb{R}^{d_y}}{\arg\min } \{ g(x, y):  p\left(x, y\right) \leq 0; q\left(x, y\right) = 0\}. \nonumber
\end{align}
The bilevel optimization minimizes the overall objective function $\Phi(x)$ with respect to (w.r.t.) $x$, where $y^{*}(x)$ is the optimal solution of the lower-level optimization problem and parametric in the upper-level decision variable $x$. In this paper, we assume that $y^{*}(x)$ is unique for any $x \in \mathbb{R}^{d_x}$.

Existing methods to solve problem \eqref{opt_intr} can be categorized into two classes: single-level reduction methods \cite{BARD198277, Bard1990, shi2005extended} and descent methods \cite{SAVARD1994265, dempe1998implicit}.
Single-level reduction methods use the KKT conditions to replace the lower-level optimization problem when it is convex.
Then, they reformulate the bilevel optimization problem \eqref{opt_intr} as a single-level constrained optimization problem.
Descent methods aim to find descent directions in which the new point is feasible and meanwhile reduces the objective function.
Paper \cite{SAVARD1994265} computes a descent direction of the objective function by solving a quadratic program.
Paper \cite{dempe1998implicit} applies the gradient of the objective function computed in \cite{kolstad1990derivative, fiacco1990nonlinear} to compute a generalized Clarke Jacobian, and uses a bundle method \cite{schramm1992version} for the optimization. 
When applied to machine learning, bilevel optimization faces additional challenges as the dimensions of decision variables in the upper-level and lower-level problems are high \cite{liu2021investigating}.

Gradient-based methods have been shown to be effective in handling large-scale and high-dimensional data in a variety of machine learning tasks \cite{bottou2008}. 
They have been extended to solve the bilevel optimization problem where there is no constraint in the lower-level optimization.
The methods can be categorized into
the approximate implicit differentiation (AID) based approaches \cite{pedregosa2016hyperparameter,gould2016differentiating,ghadimi2018approximation,grazzi2020iteration} and the iterative differentiation (ITD) approaches \cite{grazzi2020iteration,franceschi2017forward,franceschi2018bilevel,shaban2019truncated,ji2021bilevel}. 
The AID based approaches evaluate the gradients of $y^{*}(x)$ and $\Phi (x)$ based on implicit differentiation \cite{bengio2000gradient}.
The ITD based approaches treat the iterative optimization steps in the lower-level optimization 
as a dynamical system, impose $y^{*}(x)$ as its stationary point, and compute
$\nabla y^{*}(x)$ at each iterative step. 
The gradient-based algorithms have been applied to solve several machine learning tasks,
including meta-learning \cite{franceschi2018bilevel,rajeswaran2019meta,JiLLP20}, 
hyperparameter optimization \cite{pedregosa2016hyperparameter,franceschi2017forward, franceschi2018bilevel}, reinforcement learning \cite{hong2020two,konda2000actor}, and network architecture search \cite{liu2018darts}.
The above methods are limited to unconstrained bilevel optimization and require the objective function to be differentiable.
They cannot be directly applied when constraints are present in the lower-level optimization, as the objective function is non-differentiable.

\textbf{Contributions.  }
In this paper, we consider a special case of problem \eqref{opt_intr} where the upper-level constraints $r$ and $s$ are not included.
In general, the objective function $\Phi$ is nonconvex and non-differentiable, even if the upper-level and lower-level problems are convex and functions $f$, $g$, $p$, $q$ are differentiable \cite{10.1137.0913069,liu2021investigating}. 
Most methods for this bilevel optimization problem are highly complicated and computationally expensive, especially when the dimension of the problem is large \cite{liu2021investigating,10.1007/s10589-015-9795-8}. 
Addressing the challenge, we determine the descent direction by computing several gradients which can represent the gradients of the objective function of all points in a ball, and develop a computationally efficient algorithm with convergence guarantee for the constrained bilevel optimization problem.
The overall contributions are summarized as follows.
(i) Firstly, we derive the conditions under which the lower-level optimal solution $y^{*}(x)$ is continuously differentiable or directional differentiable. In addition, we provide analytical expressions for the gradient of $y^*(x)$ when it is continuously differentiable and the directional derivative of $y^{*}(x)$ when it is directional differentiable.
(ii) Secondly, we propose the gradient approximation method, which applies the Clarke subdifferential approximation of the non-convex and non-differentiable objective function $\Phi$ to the line search method. In particular, a set of derivatives is used to approximate the gradients or directional derivatives on all points in a neighborhood of the current estimate.  Then, the Clarke subdifferential is approximated by the derivatives, and the approximate Clarke subdifferential is employed as the descent direction for line search.
(iii) It is shown that, the Clarke subdifferential approximation errors are small, the line search is always feasible, and the algorithm asymptotically converges to the set of Clarke stationary points. 
(iv) We empirically verify the efficacy of the proposed algorithm by conducting experiments on hyperparameter optimization and meta-learning.

\textbf{Related Works. }
Differentiation of the optimal solution of a constrained optimization problem has been studied for a long time.
Sensitivity analysis of constrained optimization \cite{lemke1985introduction,fiacco1990sensitivity,fiacco1990nonlinear} shows the optimal solution $y^*(x)$ of a convex optimization problem is directional differentiable but may not differentiable at all points.
It implies that the objective function $\Phi(x)$ in problem \eqref{opt_intr} may not be differentiable.
Based on the implicit differentiation of the KKT conditions, the papers also compute $\nabla y^*(x)$ when $y^*$ is differentiable at $x$.
Optnet \cite{amos2017input,amos2017optnet,agrawal2019differentiable} applies the gradient computation to the constrained bilevel optimization, where a deep neural network is included in the upper-level optimization problem.
In particular, the optimal solution $y^*(x)$ serves as a layer in the deep neural network and $\nabla y^*(x)$ is used as the backpropagation gradients to optimize the neural network parameters. 
However, all the above methods do not explicitly consider the non-differentiability of $y^{*}(x)$ and $\Phi(x)$, and cannot guarantee convergence.
Recently, papers \cite{liu2021towards,sow2022constrained} consider that the lower-level optimization problem has simple constraints, such that projection onto the constraint set can be easily computed, and require that the constraint set is bounded. In this paper, we consider inequality and equality constraints, which are more general than those in \cite{liu2021towards,sow2022constrained}.

\textbf{Notations. }
Denote $a>b$ for vectors $a, b \in \mathbb{R}^{n}$, when $a_i>b_i$ for all $1 \leq i \leq n$. Notations $a \geq b$, $a=b$, $a \leq b$, and $a<b$ are defined in an analogous way. Denote the $l_2$ norm of vectors by $\| \cdot \|$.
The directional derivative of a function $f$ at $x$ on the direction $d$ with $\|d\|=1$ is defined as $\nabla_{d} f({x}) \triangleq \lim _{h \rightarrow 0^{+}} \frac{f(x+h d)-f(x)}{h}$. 
A ball centered at $x$ with radius $\epsilon$ is denoted as $\mathcal{B}(x,\epsilon)$.
The complementary set of a set $S$ is denoted as $S^C$.
The distance between the point $x$ and the set $S$ is defined as $d(x, S)  \triangleq \inf \{\|x-a\| \mid a \in S\}$.
The convex hull of $S$ is denoted by $\operatorname{conv} S$. 
For set $S$ and function $f$, we define the image set $f(S) \triangleq \{f(x)\mid x \in S \}$.
For a finite positive integer set $I$ and a vector function $p$, we denote the subvector function $p_I \triangleq [p_{k_1}, \cdots , p_{k_j}, \cdots]^{\top} $ where $k_j \in I$. 

\section{Problem Statement}
\label{section_1}
Consider the constrained bilevel optimization problem: 
\begin{align}
\label{biopt}
&\min_{x \in \mathbb{R}^{d_x}}   \   \Phi(x)=f\left(x, y^{*}(x)\right) \\
&\text{s.t. }
y^{*}(x)=\underset{y \in \mathbb{R}^{d_y}}{\arg\min } \{ g(x, y):  p\left(x, y\right) \leq 0; q\left(x, y\right) = 0\}, \nonumber
\end{align}
where $f, g : \mathbb{R}^{d_x} \times \mathbb{R}^{d_y} \rightarrow \mathbb{R}$; $p : \mathbb{R}^{d_x} \times \mathbb{R}^{d_y} \rightarrow \mathbb{R}^{m}$; $q : \mathbb{R}^{d_x} \times \mathbb{R}^{d_y} \rightarrow \mathbb{R}^{n}$. 
Given $x \in \mathbb{R}^{d_x}$, we denote the lower-level optimization problem in \eqref{biopt}
as $P(x)$.
The feasible set of $P(x)$ is defined as $K\left(x\right) \triangleq \{y \in \mathbb{R}^{d_y}: p\left(x, y\right) \leq 0, q\left(x, y\right) = 0\}$.
Suppose the following assumptions hold. 

\begin{assumption}
\label{a1} The functions $f$, $g$, $p$ and $q$ are twice continuously differentiable.
\end{assumption}

\begin{assumption}
\label{a2}
For all ${x} \in \mathbb{R}^{d_x}$, the function $g(x,y)$ is $\mu$-strongly-convex w.r.t. $y$; $p_j(x,y)$ is convex w.r.t. $y$ for each $j$; $q_i(x,y)$ is affine w.r.t. $y$ for each $i$. 
\end{assumption}

Note that the upper-level objective function $f(x,y)$ and the overall objective function $\Phi(x)$ are non-convex.
The lower-level problem $P(x)$ is convex and its Lagrangian is 
$
\mathcal{L}(y, \lambda, \nu, x) \triangleq g(x, y)+ \lambda^{\top} p(x, y)+ \nu^{\top} q(x, y)
$,
where $(\lambda, \nu)$ are Lagrange multipliers and $\lambda \geq 0$.

\begin{definition}
\label{def1}
Suppose that the KKT conditions hold at ${y}$ for $P(x)$ with the Lagrangian multipliers ${\lambda}$ and ${\nu}$. 
The set of {active inequality constraints} at ${y}$ for $P(x)$ is defined as:
$J(x,{y}) \triangleq \{j: 1 \leq j \leq m, \  p_{j}(x,{y})=0\}$.
An inequality constraint is called inactive if it is not included in $J(x,{y})$ and the set of inactive constraints is denoted as $J(x,y)^C$.
The set of {strictly active inequality constraints} at ${y}$ is defined as:
$J^{+}(x,{y},{\lambda}) \triangleq \{j: j \in J\left(x,{y}\right), \  {\lambda}_{j}>0\}$.
The set of {non-strictly active inequality constraints} at ${y}$ is defined as:
$J^{0}(x,{y},{\lambda}) \triangleq J(x,{y}) \setminus J^{+}(x,{y},{\lambda})$.
Notice that ${\lambda}_j \geq 0$ for $j \in J(x,{y})$ and ${\lambda}_j = 0$ for $j \in J^{0}(x,{y},{\lambda})$.
\end{definition} 

\begin{definition} 
The Linear Independence Constraint Qualification (LICQ) holds at ${y}$ for $P(x)$ if the vectors 
$
\left\{\nabla_y p_{j}\left(x,{y}\right), j \in J\left(x, {y}\right) ; \nabla_y q_{i}\left(x, {y}\right),  1 \leq i \leq n \right\} 
$
are linearly independent. 
\end{definition} 

\begin{assumption}
\label{a3}
Suppose that for all ${x} \in \mathbb{R}^{d_x}$, the solution $y^{*}(x)$ exists for $P\left(x\right)$, and the LICQ holds at $y^{*}(x)$ for $P(x)$.
\end{assumption}

\section{Differentiability and Gradient of $y^{*}(x)$}
\label{section3}
In this section, we provide sufficient conditions under which the lower-level optimal solution $y^{*}(x)$ is continuously differentiable or directional differentiable. We compute the gradient of $y^*(x)$ when it is continuously differentiable and the directional derivative of $y^{*}(x)$ when it is directional differentiable. 
Moreover, we give a necessary condition that $y^{*}(x)$ is not differentiable and illustrate it by a numerical example.

In problem \eqref{biopt},
if the upper-level objective function $f$ and the solution of lower-level problem $y^{*}$ are continuously differentiable, so is $\Phi$, and by the gradient computation of composite functions, we have
\begin{equation}
    \label{compositiond}
    \nabla \Phi(x)=\nabla_{x} f(x, y^{*}(x))+\nabla  y^{*}(x)^{\top} \nabla_{y} f(x, y^{*}(x)).
\end{equation}
It is shown in \cite{domke2012generic} that, when $p$ and $q$ are absent, $y^*$ and $\Phi$ are differentiable under certain assumptions.
The differentiability of $y^*$ and $\Phi$ is used by the AID based approaches in \cite{pedregosa2016hyperparameter,gould2016differentiating,ghadimi2018approximation,grazzi2020iteration,domke2012generic} to approximate $\nabla y^*$  and minimize $\Phi$ by gradient descent.
However, it is not the case as the lower-level problem \eqref{biopt} is constrained. 

Theorem \ref{th0} states the conditions under which $y^{*}(x)$ is directional differentiable.
\begin{theorem}
\label{th0}
Suppose Assumptions \ref{a1}, \ref{a2}, \ref{a3} hold. The following properties hold for any $x$.
\begin{itemize}
\item[(\romannumeral1)] The global minimum $y^{*}(x)$ of $P\left(x\right)$ exists and is unique. The KKT conditions hold at $y^{*}(x)$ with unique Lagrangian multipliers $\lambda(x)$ and $\nu(x)$.
\item[(\romannumeral2)] 
The vector function $z(x) \triangleq [y^{*}(x)^{\top}, \lambda(x)^{\top}, \nu(x)^{\top}]^{\top}$ is continuous and locally Lipschitz. 
The directional derivative of $z(x)$ on any direction exists. 
\end{itemize}
\end{theorem}

As shown in part (i) of Theorem \ref{th0}, 
$y^{*}(x)$, ${\lambda}(x)$ and ${\nu}(x)$ are uniquely determined by $x$. 
So we simplify the notations of Definition \ref{def1} in the rest of this paper: $J(x,y^{*}(x))$ is denoted as $J(x)$, $J^{+}(x,y^{*}(x),{\lambda}(x))$ is denoted as $J^{+}(x)$, and $J^{0}(x,y^*(x),\lambda(x))$ is denoted as $J^{0}(x)$. 
In part (ii), the computation of the directional derivative of $z(x)$ is given in Theorem \ref{th2} in Appendix \ref{proof_app}.

\begin{definition}
\label{def3}
Suppose that the KKT conditions hold at ${y}$ for $P(x)$ with the Lagrangian multipliers ${\lambda}$ and ${\nu}$.
The Strict Complementarity Slackness Condition (SCSC) holds at ${y}$ w.r.t. ${\lambda}$ for $P(x)$, if ${\lambda}_{j}>0$ for all ${j} \in J(x,{y})$. 
\end{definition}

\begin{remark}
\label{remark_scsc}
The KKT conditions include the Complementarity Slackness Condition (CSC). The SCSC is stronger than the CSC, which only requires that ${\lambda}_{j} \geq 0$ for all ${j} \in J(x,{y})$. 
\end{remark}

Theorem \ref{th1} states the conditions under which $y^{*}(x)$ is continuously differentiable and derives $\nabla y^{*}(x)$.

\begin{theorem}
\label{th1}
Suppose Assumptions \ref{a1}, \ref{a2}, \ref{a3} hold. 
If the SCSC holds at $y^{*}(x)$ w.r.t. $\lambda(x)$, then $z(x)$ is continuously differentiable at $x$ and the gradient is computed as
\begin{equation}
\label{eq8}
\left[
\nabla_{x} y^{*}(x)^{\top},
\nabla_{x} \lambda_{J(x)}^{\top}(x),
\nabla_{x} \nu(x)^{\top}
\right]^{\top}
= -M_{+}^{-1}(x) N_{+}(x)
\end{equation}
and $\nabla_{x} \lambda_{{J(x)}^C}(x)=0$,
where $M_{+}(x) \triangleq $
$$
\left[\begin{array}{ccccccc}
\nabla_{y}^{2} \mathcal{L}  & \nabla_{y} p_{J^{+}(x)}^{\top} & \nabla_{y} q^{\top} \\
\nabla_{y} p_{J^{+}(x)}  & 0 & 0 \\
\nabla_y q & 0 & 0 
\end{array}\right](x,y^*(x),\lambda(x),\nu(x))
$$ 
is nonsingular and $N_{+}(x)\triangleq$
$$
[\nabla_{x y}^{2} \mathcal{L}^{\top}, \nabla_{x} p_{J^{+}(x)}^{\top}, \nabla_{x} q^{\top}]^{\top}(x,y^*(x),\lambda(x),\nu(x)).
$$

\end{theorem}

Theorem \ref{th1} shows that, if $z(x)$ is not continuously differentiable, then the SCSC does not hold at $y^*(x)$ w.r.t. $\lambda(x)$. 
Definition \ref{def3} implies that the SCSC holds at ${y}$ w.r.t. ${\lambda}$ for $P(x)$ if and only if $J^{0}(x)  =\emptyset$. 
It concludes that if $y^*(x)$ is not continuously differentiable at $x$, $J^0(x) \neq \emptyset$, i.e., the non-differentiability of $y^{*}(x)$  occurs at points with non-strictly active constraints. 
Example \ref{example1} illustrates such claim. 

\begin{example}
\label{example1}
Consider a bilevel optimization problem $\Phi(x)=y^*(x)$ and the lower-level problem $P(x)$: $y^*(x)= {\arg\min}_y \{(y-x^2)^2 : p_1(x,y)=-x-y \leq 0\}$, where $x$, $y \in \mathbb{R}$.
The analytical solution of $z(x)=[y^*(x), \lambda(x)]$ is given by: $y^*(x)=x^2$, $\lambda(x)=0$ when $x \in (-\infty,-1] \cup [0,+\infty)$; $y^*(x)=-x$, $\lambda(x)=-2x(1+x)$ when $x \in [-1,0]$. 
Correspondingly, when $x \in (-1,0)$, $J(x)=\{1\}$, $J^+(x)=\{1\}$, $J^0(x)=\emptyset$; when $x \in (-\infty,-1) \cup (0,+\infty)$, $J(x)=\emptyset$, $J^+(x)=\emptyset$, $J^0(x)=\emptyset$; when $x \in \{ -1,0 \}$, $J(x)=\{1\}$, $J^+(x)=\emptyset$, $J^0(x)=\{1\}$.
As shown in Fig. \ref{ssss1}, $y^*(x)$ is continuously differentiable everywhere except when $J^0(x) \neq \emptyset$.
\end{example}

\begin{figure}[hbt] 
\centering 
\includegraphics[width=0.38\textwidth]{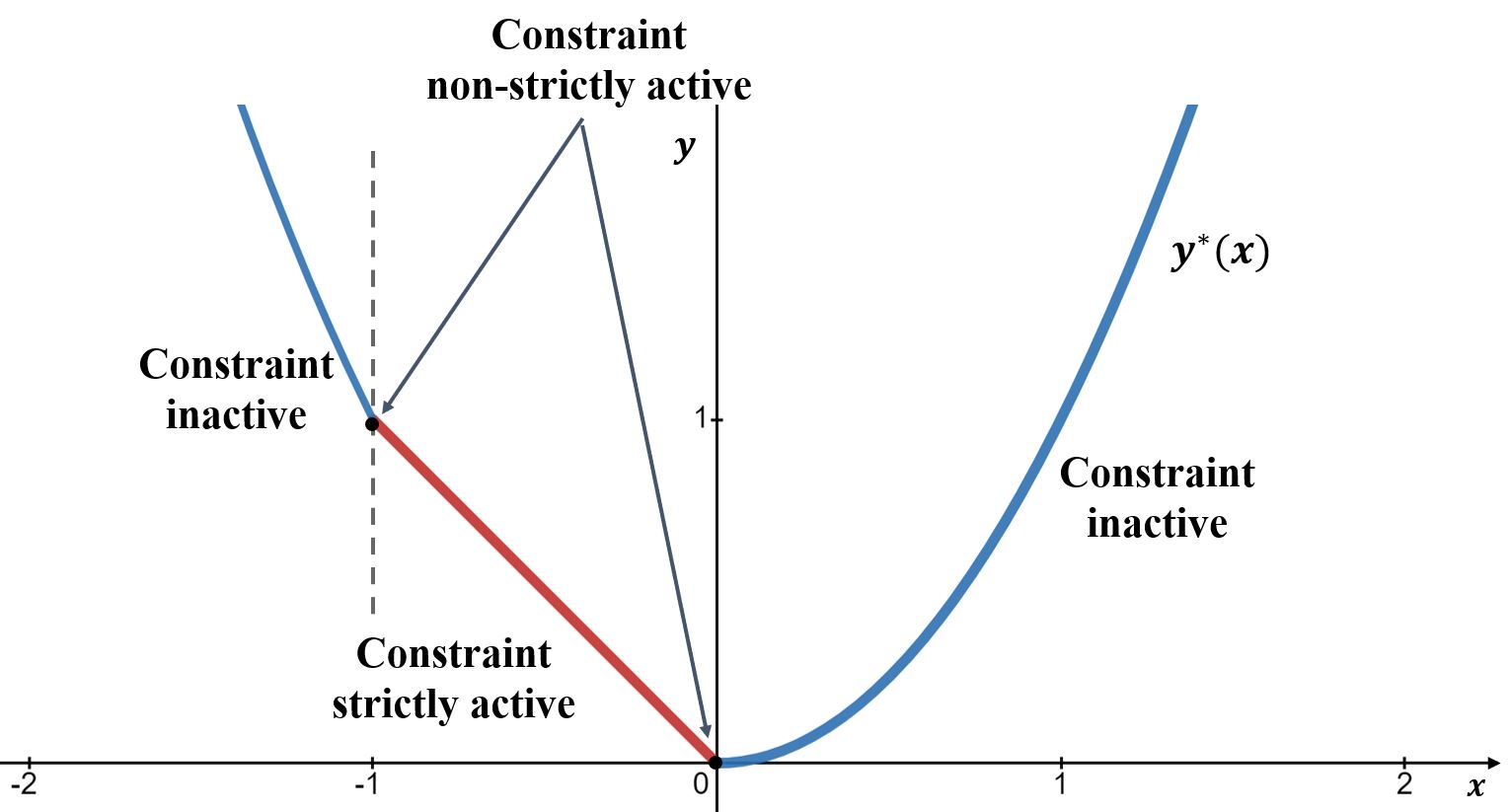} 
\caption{Occurrence of non-differentiability. }
\label{ssss1} 
\end{figure}

The computation of the gradient of $z(x)$ in \eqref{eq8} is derived from the implicit differentiation of the KKT conditions of problem $P(x)$, which is also used in \cite{fiacco1990sensitivity, amos2017optnet, agrawal2019differentiable}. 
Compared with these papers, Theorem \ref{th1} directly determines $\nabla_{x} \lambda_{{J(x)}^C}(x)=0$ and excludes $\lambda_{{J(x)}^C}(x)$ from the computation of the inverse matrix in \eqref{eq8}, when $z(x)$ is continuously differentiable. Theorem \ref{th2} in Appendix \ref{proof_app} derives the directional derivative of $z(x)$ when it is not differentiable.

Consider a special case where the lower-level optimization problem $P(x)$ is unconstrained. 
Since the SCSC is not needed anymore, the assumptions in Theorem \ref{th1} reduce to that $g$ is twice continuously differentiable and $g(x,y)$ is $\mu$-strongly-convex w.r.t. $y$ for ${x} \in \mathbb{R}^{d_x}$. 
By Theorem \ref{th1}, the optimal solution $y^{*}(x)$ is continuously differentiable, the matrix $\nabla_{y}^{2} g({x}, {y})$ is non-singular, and the gradient is computed as
$\nabla  y^{*}(x) =-[\nabla_{y}^{2} g({x}, {y})]^{-1} \nabla_{x y}^{2} g({x}, {y})$.
These results are well-known and widely used in unconstrained bilevel optimization analysis and applications \cite{pedregosa2016hyperparameter,franceschi2017forward, franceschi2018bilevel, ji2021bilevel}.

\section{The Gradient Approximation Method}
In this section, we develop the gradient approximation method to efficiently solve problem \eqref{biopt}, whose objective function is
non-differentiable and non-convex.
First, we define the Clarke subdifferential (Section \ref{sectionA}) and efficiently approximate the Clarke subdifferential of the objective function
$\Phi(x)$ (Section \ref{sectionB}). Next, we propose the gradient approximation algorithm, provide its convergence guarantee (Section \ref{sectionC}), and present its implementation details (Section \ref{sectionD}).

\subsection{Clarke Subdifferential of $\Phi$}
\label{sectionA}
As shown in Section \ref{section_1} and also shown in \cite{10.1007/s10589-015-9795-8,liu2021investigating}, the objective function $\Phi\left(x\right)$ of problem (\ref{biopt}) is usually non-differentiable and non-convex.
To deal with the non-smoothness and non-convexity,
we introduce Clarke subdifferential and Clarke stationary point.

\begin{definition}[Clarke subdifferential and Clarke stationary point \cite{clarke1975generalized}]
\label{defc}
For a locally Lipschitz function $f: \mathbb{R}^{n} \rightarrow \mathbb{R}$, the Clarke subdifferential of $f$ at $x$ is defined by the convex hull of the limits of gradients of $f$ on sequences converging to $x$, i.e.,
$\bar{\partial} f(x) \triangleq \operatorname{conv}\left\{\lim _{j \rightarrow \infty} \nabla f\left(y^{j}\right):\left\{y^{j}\right\} \rightarrow x\right.$ where $f$ is differentiable at $y^{j}$ for all $\left.j \in \mathbb{N}\right\}$.
The Clarke $\epsilon \text {-subdifferential}$ of $f$ at $x$ is defined by $
\bar{\partial}_{\epsilon} f(x)\triangleq\operatorname{conv} \{ \bar{\partial} f(x^{\prime}): x^{\prime} \in \mathcal{B}(x,\epsilon) \}$.
A point $x$ is Clarke stationary for $f$ if $0 \in \bar{\partial} f(x)$.

\end{definition}
If $y^*$ is differentiable at $x$, we have $\bar{\partial} y^*(x)=\left\{ \nabla y^*(x) \right\}$
and $\bar{\partial} \Phi(x)=\{ \nabla_{x} f(x, y^{*}(x))+\nabla  y^{*}(x)^{\top} \nabla_{y} f(x, y^{*}(x)) \}$;
otherwise, 
$\bar{\partial} \Phi(x)= \{\nabla_{x} f(x, y^{*}(x))+w^{\top} \nabla_{y} f(x, y^{*}(x)) : w \in  \bar{\partial}  y^{*}(x) \}$.
Take the functions shown in Example \ref{example1} and Fig. \ref{ssss1} as an example, $\bar{\partial}_{\epsilon} \Phi(-1)=\bar{\partial}_{\epsilon} y^{*}(-1)= \operatorname{conv} \{[-2-2\epsilon,-2] \cup \{-1\}\}=[-2-2\epsilon,-1]$, and $\bar{\partial}_{\epsilon} \Phi(0)=\bar{\partial}_{\epsilon} y^{*}(0)= \operatorname{conv} \{[0,2\epsilon] \cup \{-1\}\}=[-1,2\epsilon]$.

\subsection{Clarke Subdifferential Approximation}
\label{sectionB}
Gradient-based methods have been applied to convex and non-convex optimization problems \cite{hardt2016train,nemirovski2009robust}. The convergence requires that the objective function is differentiable.
If there exist points where the objective function is not differentiable, the probability for the algorithms to visit these points is non-zero and the gradients at these points are not defined \cite{bagirov2020numerical}.
Moreover, oscillation may occur even if the objective function is differentiable at all visited points \cite{bagirov2014introduction}.

To handle the non-differentiability, the gradient sampling method \cite{burke2005robust,kiwiel2007convergence,bagirov2020numerical} uses gradients in a neighborhood of the current estimate to approximate the Clarke subdifferential and determine the descent direction. 
Specifically, the method samples a set of points inside the neighborhood $\mathcal{B}(x^0,\epsilon)$, select the points where the objective function is differentiable, and then compute the convex hull of the gradients on the sampled points.

However, in problem (\ref{biopt}), the point sampling is highly computationally expensive. 
For each sampled point $x^j$, to check the differentiability of $\Phi$, we need to solve the lower-level optimization $P(x^j)$ to obtain $y^*(x^j)$, $\lambda(x^j)$ and $\nu(x^j)$, and check the SCSC. 
Moreover, after the points are sampled, the gradient on each point is computed by \eqref{eq8}. 
As the dimension $d_x$ increases, the sampling number increases to ensure the accuracy of the approximation.
More specifically, as shown in \cite{kiwiel2007convergence}, the algorithm is convergent if the sampling number is large than $d_x+1$. 
The above procedure is executed in each optimization iteration.

Addressing the computational challenge, we approximate the Clarke $\epsilon \text {-subdifferential}$ by a small number of gradients, which can represent the gradients on all points in the neighborhood.
The following propositions distinguish two cases: $\Phi$ is continuously differentiable on $\mathcal{B}(x^0,\epsilon)$ (Proposition \ref{prop2_0}) and it is not (Proposition \ref{prop3}).
\begin{proposition}
\label{prop2_0}
Suppose Assumptions \ref{a1}, \ref{a2}, \ref{a3} hold. 
Consider $x^0 \in R^{d_x}$. There is sufficiently small $\epsilon>0$ such that,
if the SCSC holds at $y^{*}(x)$ w.r.t. $\lambda(x)$ for any $x \in \mathcal{B}(x^0,\epsilon)$, then $\nabla \Phi(x^0) \in \bar{\partial}_{\epsilon} \Phi(x^0)$ and
$$|\|\nabla \Phi(x^0)\| - d(0, \bar{\partial}_{\epsilon} \Phi(x^0))|< o(\epsilon).$$
\end{proposition}
Proposition \ref{prop2_0} shows that the gradient $\nabla \Phi$ at a single point $x^0$ can be used to approximate the Clarke $\epsilon \text {-subdifferential}$ $\bar{\partial}_{\epsilon} \Phi(x^0)$ and the approximation error is in the order of $\epsilon$.
Recall that the gradient $\nabla \Phi(x^0)$ can be computed by \eqref{compositiond} and \eqref{eq8}. Fig. \ref{ssss1.5} illustrates the approximation on the problem in Example \ref{example1}. The SCSC holds at $y^{*}(x)$ and $\Phi(x)$ is continuously differentiable on $\mathcal{B}(x^0,\epsilon)$, then $\bar{\partial}_{\epsilon} \Phi(x^0)=[2x^0-2\epsilon, 2x^0+2\epsilon]$ can be approximated by $\nabla \Phi(x^0)= 2x^0$, and the approximation error is $2\epsilon$.
The approximations of $\bar{\partial}_{\epsilon}\Phi(x^1)$ and $\bar{\partial}_{\epsilon}\Phi(x^2)$ can be done in an analogous way.

\begin{figure*}[hbt] 
\begin{minipage}[t]{0.33\textwidth}
\centering 
\includegraphics[width=1.0\textwidth]{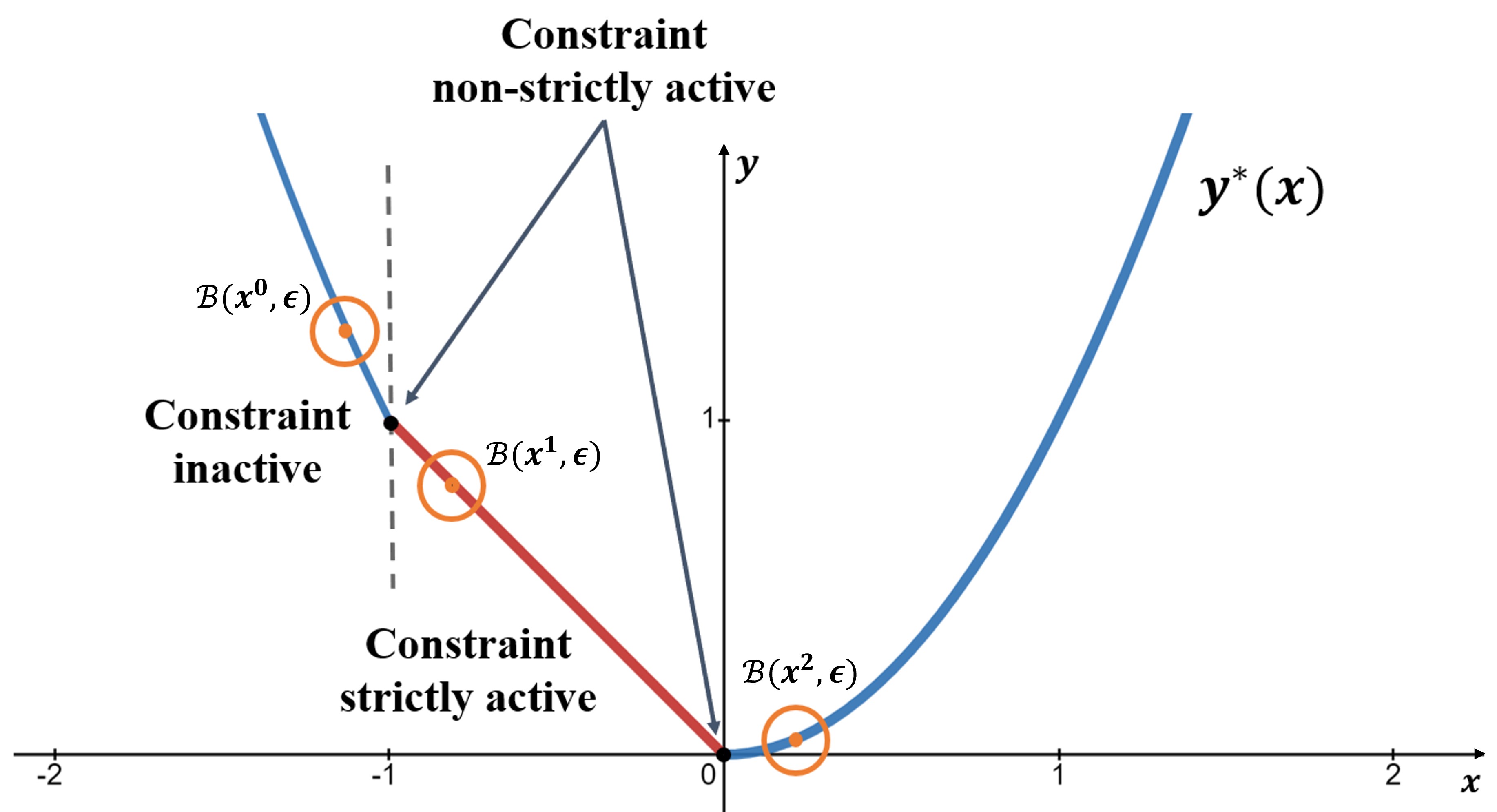} 
\caption{The SCSC holds on $\mathcal{B}(x^0,\epsilon)$} 
\label{ssss1.5} 
\end{minipage}
\begin{minipage}[t]{0.33\textwidth}
\centering 
\includegraphics[width=1.0\textwidth]{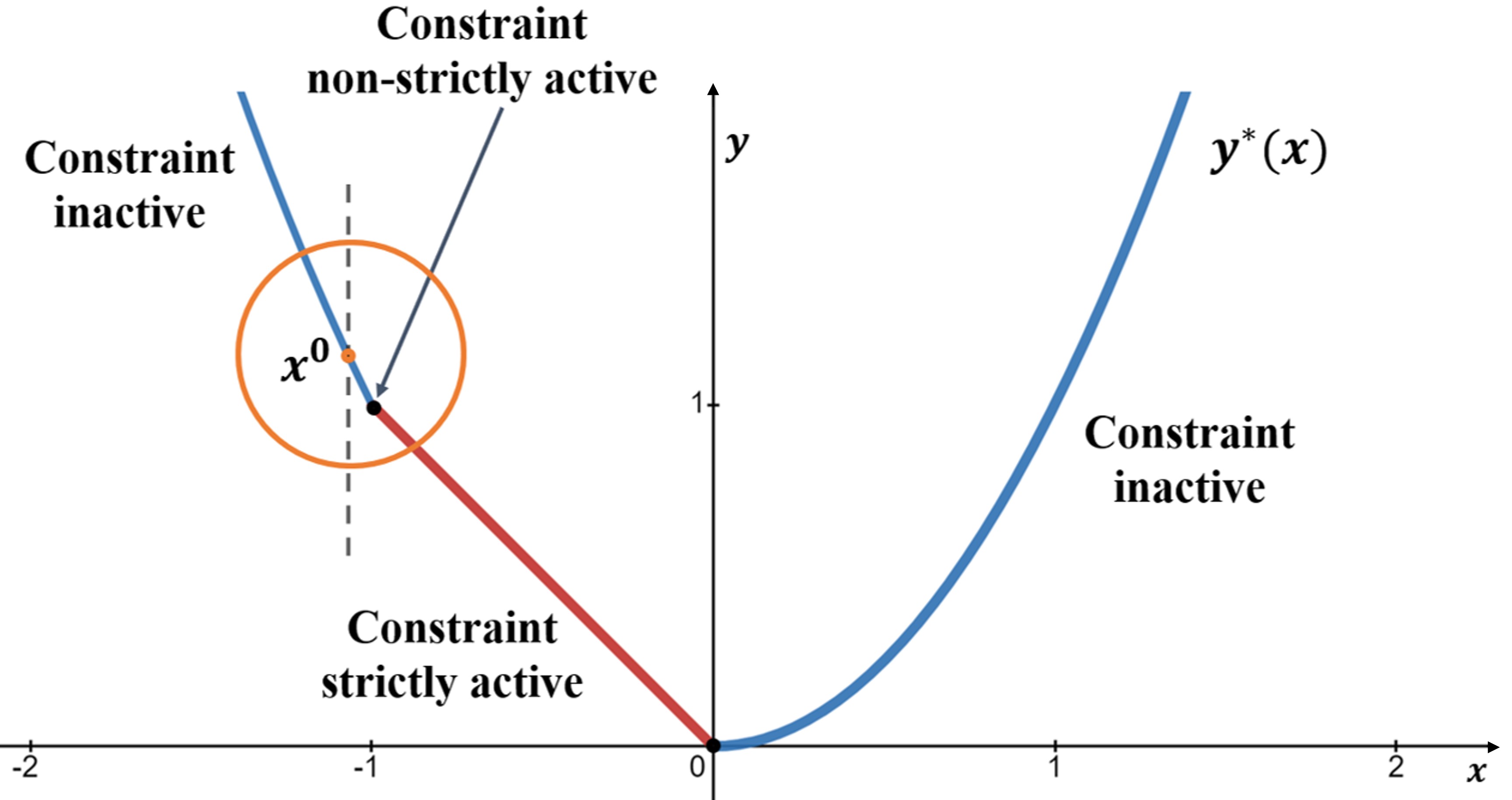} 
\caption{Subsets inside $\mathcal{B}(x^0,\epsilon)$} 
\label{ssss2} 
\end{minipage}
\begin{minipage}[t]{0.32\textwidth}
\centering 
\includegraphics[width=0.9\textwidth]{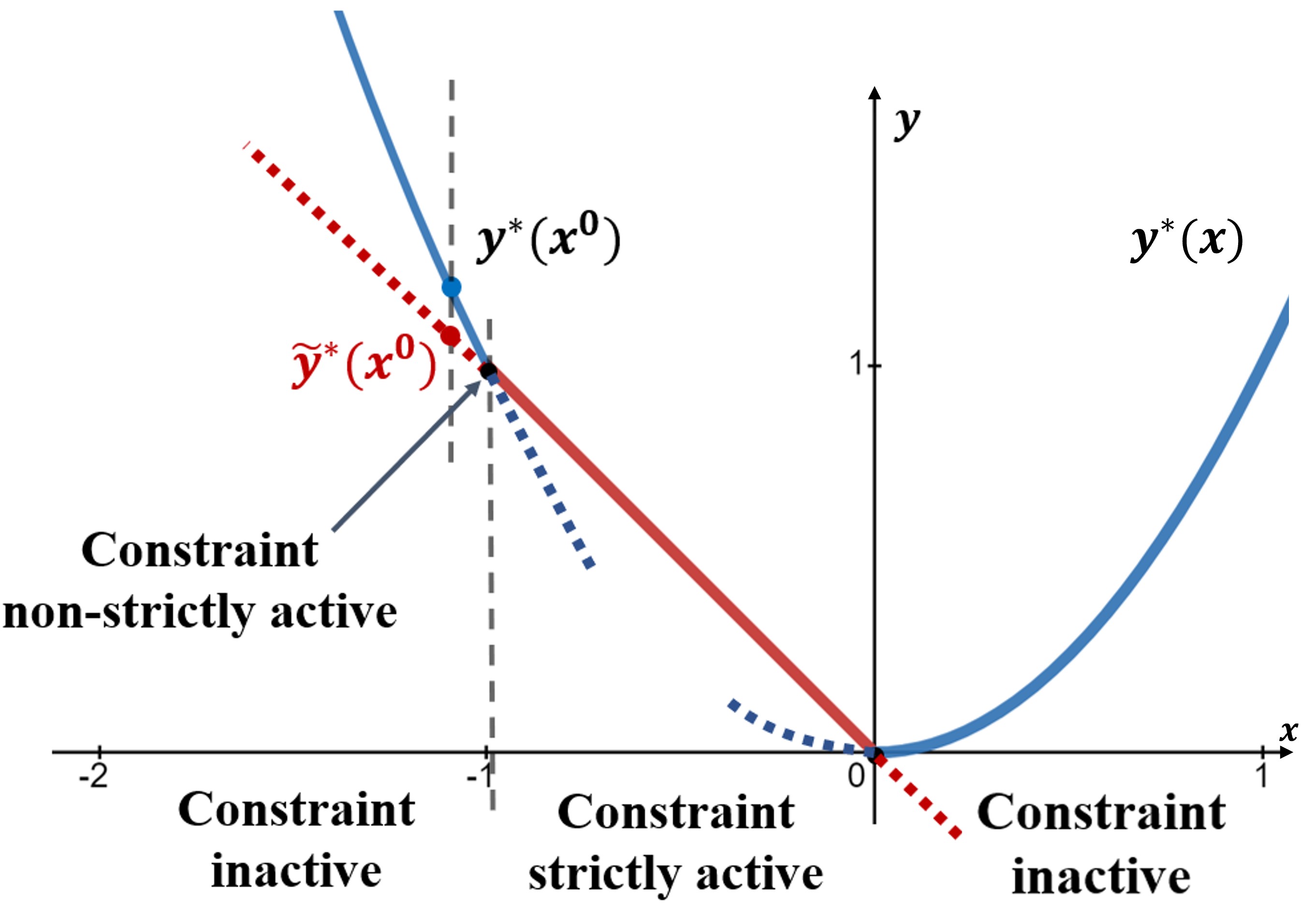} 
\caption{Approximate $\bar{\partial}_{\epsilon} y^*(x^0)$} 
\label{ssss3} 
\end{minipage}
\end{figure*}

Consider $\Phi(x)$ is not continuously differentiable at some points in $\mathcal{B}(x^0,\epsilon)$.
Define the set $I^\epsilon(x^0)$ which contains all $j$ such that there exist $x^{\prime}$, $x^{\prime\prime} \in \mathcal{B}(x^0,\epsilon)$ with 
$
j\in J^+(x^{\prime})^C \text{  and  } j\in J^+(x^{\prime\prime})
$.
Define the set $I_+^{\epsilon}(x^0)$ which contains all $j$ such that $j\in J^+(x)$ for any $x \in \mathcal{B}(x^0,\epsilon)$.
If $I^{\epsilon}(x^0)$ is not empty, there exists a point $x \in \mathcal{B}(x^0,\epsilon)$ such that the SCSC does not hold at $y^*(x)$.
The power set of $I^\epsilon(x^0)$ partitions $\mathcal{B}(x^0,\epsilon)$ into a number of subsets, where $\Phi(x)$ and $y^*(x)$ are continuously differentiable in each subset. 
An illustration on the problem in Example \ref{example1} is shown in Fig. \ref{ssss2}. 
The point $x^{\prime}=-1 $ belongs to $(x^0-\epsilon, x^0+\epsilon)$ and the SCSC does not hold at $y^*(x^{\prime})$. Notice that $I^{\epsilon}_{+}(x^0)=\emptyset$ and $I^\epsilon(x^0)=\{1\}$. Then, $I^\epsilon(x^0)=\{1\}$ has the power set $\{S_{(1)},S_{(2)}\}$ with $S_{(1)}=\emptyset$ and $S_{(2)}=\{1\}$.
Then, $\mathcal{B}(x^0,\epsilon)$ is partitioned to two subsets: the subset where the
constraint $p_1$ is inactive (blue side in the ball) which corresponds to $S_{(1)}$, and the subset where the constraint $p_1$ is strictly active (red side in the ball) which corresponds to $S_{(2)}$. Their boundary is the point $x^{\prime}$ where the constraint $p_1$ is non-strictly active.
It can be seen that $y^*(x)$ is continuously differentiable on each subset and the gradient variations are small inside the subset when $\epsilon$
is small. In contrast, the gradient variations between two subsets are large.
Inspired by Proposition \ref{prop2_0}, we compute a representative gradient to approximate $\nabla y^*(x)$ inside each subset of $\mathcal{B}(x^0,\epsilon)$.

Now we proceed to generalize the above idea. 
Recall that $\nabla \Phi(x)$ is computed by \eqref{compositiond} and $f$ is twice continuously differentiable.
Define 
\begin{equation}
\begin{aligned}
\label{eq17}
G(x^{0},\epsilon)\triangleq\{ \nabla_{x} f\left(x^0, y^{*}\left(x^0\right)\right)+{w}^S(x^0)^{\top} \\ 
\nabla_{y} f\left(x^0, y^{*}\left(x^0\right)\right): S \subseteq I^\epsilon(x^0)\},
\end{aligned}
\end{equation}
where
${w}^S(x^0)$ is obtained by extracting the first $d_x$ rows from matrix $-M^{S}_{\epsilon}(x^0,y^*(x^0))^{-1} N^{S}_{\epsilon}(x^0,y^*(x^0))$, with
$$M^{S}_{\epsilon} \triangleq
\left[\begin{array}{ccccccc}
\nabla_{y}^{2} \mathcal{L} & \nabla_y p_{I^{\epsilon}_{+}(x^0)}^{\top} & \nabla_{y} q^{\top}  & \nabla_y p^{\top}_{S}\\
\nabla_y p_{I^{\epsilon}_{+}(x^0)} & 0 & 0 &0\\
\nabla_y q & 0 & 0 &0\\
\nabla_y p_{S} & 0 & 0 &0
\end{array}\right],$$
and 
$N^{S}_{\epsilon} \triangleq \left[\nabla_{x y}^{2}\mathcal{L}^{\top}, \nabla_{x} p_{I^{\epsilon}_{+}(x^0)}^{\top}, \nabla_{x} q^{\top}, \nabla_{x} p_{S}^{\top} \right]^{\top}$.
Here, $S$ is a subset of $I^\epsilon(x^0)$, and $w^S(x^0)$ is the representative gradient to approximate $\nabla y^*(x)$ inside the subset of $\mathcal{B}(x^0,\epsilon)$ which corresponds $S$. Proposition \ref{prop3} shows that the Clarke $\epsilon \text {-subdifferential}$ $\bar{\partial}_{\epsilon} y^*(x^0)$ can be approximated by representative gradient set $G(x^{0},\epsilon)$, and the approximation error is in the order of $\epsilon$. 

\begin{proposition}
\label{prop3}
Suppose Assumptions \ref{a1}, \ref{a2}, \ref{a3} hold. 
Consider $x^0 \in \mathbb{R}^{d_x}$, and assume there exists a sufficiently small $\epsilon>0$ such that, there exists $x \in \mathcal{B}(x^0,\epsilon)$ such that $y^{*}(x)$ is not continuously differentiable at $x$. 
Then, the following inequality holds
for any $z \in \mathbb{R}^{d_x}$, 
$$| d(z, \operatorname{conv} G(x^{0},\epsilon)) - d(z, \bar{\partial}_{\epsilon} \Phi(x^0))|< o(\epsilon).$$
\end{proposition}



The computation of the representative gradient $w^S(x^0)$ of Example \ref{example1} is demonstrated in Fig. \ref{ssss3}.
Since $x^0$ is near the boundary of two subsets, Proposition \ref{prop3} employs $w^{S_{(1)}}(x^0) = \nabla y^{*}(x^0)$ to approximate the gradients of the subset with the inactive constraint (blue side), and $w^{S_{(2)}}(x^0) =\nabla \Tilde{y}^{*}(x^0)$ to approximate the gradients in the subset with the strictly active constraint (red side). The twice-differentiable function $\Tilde{y}^{*}(x)$ is an extension of $y^{*}(x)$ (refer to the definition of $x^I(\cdot)$ in (12.8) of \cite{dempe1998implicit}).
The gradients $\nabla {y}^{*}(x^0)$ and $\nabla \Tilde{y}^{*}(x^0)$ are computed in the matrices $-{M^{S_{(1)}}_{\epsilon}}^{-1} N^{S_{(1)}}_{\epsilon}$ and $-{M^{S_{(2)}}_{\epsilon}}^{-1} N^{S_{(2)}}_{\epsilon}$, respectively.
Then, the representative gradients $w^{S_{(1)}}(x^0)$ and $w^{S_{(2)}}(x^0)$ are used to approximate $\bar{\partial}_{\epsilon} y^*(x^0)$.
Then, we can compute $G(x^0,\epsilon)=\{2x^0,-1\}$ and $\bar{\partial}_{\epsilon} \Phi(x^0)= [2x^0-2\epsilon,-1] $ with $-1 \in [x^0-\epsilon,x^0+\epsilon]$.
The approximation error $| d(z, \operatorname{conv} G(x^{0},\epsilon)) - d(z, \bar{\partial}_{\epsilon} \Phi(x^0))|$ is smaller than or equal to $2\epsilon$ for any $z$.

\begin{algorithm}[bt]
	\caption{Gradient Approximation Method} 
	\label{alg:framework0}
	\begin{algorithmic}[1] 
	    \REQUIRE Initial point $x^{0}$; 
	    Initial approximation radius $\epsilon_{0} \in(0, \infty)$; Initial stationarity target $\nu_{0} \in(0, \infty)$; Line search parameters $(\beta, \gamma) \in(0,1) \times(0,1)$; Termination tolerances $\left(\epsilon_{\mathrm{opt}}; \nu_{\mathrm{opt}}\right) \in[0, \infty) \times[0, \infty)$; Discount factors $\left(\theta_{\epsilon}, \theta_{\nu}\right) \in(0,1) \times(0,1)$.
	    \FOR{$k \in \mathbb{N}$}
	    \STATE \label{line2} Solve the lower-level optimization problem $P(x^k)$ and obtain $y^{*}(x^k)$, $\lambda(x^k)$, and $\nu(x^k)$
	    \STATE \label{line3} Check the differentiability of $y^{*}$ on $\mathcal{B}\left(x^{k}, \epsilon_{k}\right)$ by \eqref{check_dif} and \eqref{check}
	     \IF{$y^{*}$ is continuously differentiable on $\mathcal{B}\left(x^{k}, \epsilon_{k}\right)$} 
	    \STATE  \label{line5} Compute $g^{k}=\nabla \Phi(x^k)$ by \eqref{compositiond}
	     \ELSE
	    \STATE \label{line7} Compute $G(x^{k},\epsilon_k)$ by \eqref{eq17}
	    \STATE \label{line8} $\bar{\partial}_{\epsilon_k}\Phi(x^k )=\operatorname{conv} G(x^{k},\epsilon_k)$ 
	    \STATE \label{line9} Compute $g^{k}= \min \{ \|g\|: g \in \operatorname{conv} G(x^{k},\epsilon_k) \}$
	     \ENDIF 
	    \IF{$\left\|g^{k}\right\| \leq \nu_{\mathrm{opt}}$ and $\epsilon_{k} \leq \epsilon_{\mathrm{opt}}$}
	    \STATE {\bfseries Output:} $x^{k}$ and \textbf{terminate}
	    \ENDIF 
	    \IF{$\left\|g^{k}\right\| \leq \nu_{k}$}
	    \STATE $\nu_{k+1} \leftarrow \theta_{\nu} \nu_{k}$, $\epsilon_{k+1} \leftarrow \theta_{\epsilon} \epsilon_{k}$, and $t_{k} \leftarrow 0$
	    \ELSE
	    \STATE \label{line17} Compute $t_k$ by the line search:
	    $t_{k} = \max \{t \in\{\gamma, \gamma^{2}, \ldots\}:  \Phi(x^{k}-t g^{k})< \Phi(x^{k}) -\beta t\|g^{k}\|^{2}\}$
	    \STATE $\nu_{k+1} \leftarrow \nu_{k}$ and $\epsilon_{k+1} \leftarrow \epsilon_{k}$
	    \ENDIF 
	    \STATE $x^{k+1} \leftarrow x^{k}-t_{k} g^{k}$
	    \ENDFOR 
	\end{algorithmic}
\end{algorithm} 

\subsection{The Gradient Approximation Algorithm}
\label{sectionC}

Our proposed gradient approximation algorithm, summarized in Algorithm \ref{alg:framework0}, is a line search algorithm. It uses the approximation of the Clarke subdifferential as the descent direction for line search.
In iteration $k$, we firstly solve the lower-level optimization problem $P(x^k)$ and obtain $y^{*}(x^k)$, $\lambda(x^k)$ and $\nu(x^k)$. To reduce the computation complexity, the solution in iteration $k$ serves as the initial point to solve $P(x^{k+1})$ in iteration $k+1$. Secondly, we check the differentiability of $y^{*}$ on $\mathcal{B}\left(x^{k}, \epsilon_{k}\right)$ and its implementation details are shown in Section \ref{sectionD1}. If $y^{*}$ is continuously differentiable on $\mathcal{B}\left(x^{k}, \epsilon_{k}\right)$, we use $\nabla \Phi(x^k)$ to approximate $\bar{\partial}_{\epsilon} \Phi(x^k)$ which corresponds to Proposition \ref{prop2_0}. Otherwise, $G(x^{k},\epsilon_k)$ is used which corresponds to Proposition \ref{prop3}. The  details of computing $G(x^{k},\epsilon_k)$ are shown in \eqref{eq17} and Section \ref{sectionD2}.
Thirdly, the line search direction $g^k$ is determined by a vector which has the smallest norm over all vectors in the convex hull of $G(x^{k},\epsilon_k)$.
During the optimization steps, as the iteration number $k$ increases, the approximation radius $\epsilon_k$ decreases. According to Propositions \ref{prop2_0} and \ref{prop3}, the approximation error of the Clarke subdifferential is diminishing.
We next characterize the convergence of Algorithm \ref{alg:framework0}.

\begin{theorem}
\label{th_converge}
{
Suppose Assumptions \ref{a1}, \ref{a2}, \ref{a3} hold and $\Phi(x)$ is lower bounded on $\mathbb{R}^{d_x}$.
Let $\{x^{k}\}$ be the sequence generated by Algorithm \ref{alg:framework0} with $\nu_{\mathrm{opt}}=\epsilon_{\mathrm{opt}}=0$. Then, 
\begin{itemize}
    \item[(\romannumeral1)] For each $k$, the line search in line \ref{line17} has a solution $t_k$.
    \item[(\romannumeral2)] $\lim_{k\rightarrow\infty} \nu_{k} = 0$, $\lim_{k\rightarrow\infty} \epsilon_{k} = 0$.
    \item[(\romannumeral3)] $\liminf_{k\rightarrow\infty} d(0, \bar{\partial} \Phi(x^k))= 0$.
    \item[(\romannumeral4)] Every limit point of $\{x^{k}\}$ is Clarke stationary for $\Phi$.
\end{itemize}}
\end{theorem}

If the objective function $\Phi$ is non-convex but smooth, property (iii) reduces to $\liminf_{k\rightarrow\infty}\|\nabla \Phi(x^k)\| = 0$, which is a widely used convergence criterion for smooth and non-convex optimization \cite{nesterov1998introductory,jin2021nonconvex}.
A sufficient condition for the existence of limit point of $\{x^k\}$ is that the sequence is bounded.

\subsection{Implementation Details}
\label{sectionD}
\subsubsection{Check differentiability of $y^{*}$ on $\mathcal{B}\left(x^{0}, \epsilon_{0}\right)$} 
\label{sectionD1}
We propose Proposition \ref{prop2} to check differentiability of $y^{*}$ on $\mathcal{B}\left(x^{0}, \epsilon_{0}\right)$, which is required by line \ref{line3} of Algorithm \ref{alg:framework0}.

\begin{proposition}
\label{prop2}
Consider $x^0 \in \mathbb{R}^{d_x}$ and $\epsilon>0$.
Suppose Assumptions \ref{a1}, \ref{a2}, \ref{a3} hold. 
Then, Lipschitz constants of functions $\Phi(x)$, $\lambda_j(x)$ and $p_j(x,y^*(x))$ on $\mathcal{B}(x^0,\epsilon)$ exist and are denoted by 
$l_{\Phi}(x^0,\epsilon)$, $l_{\lambda_j}(x^0,\epsilon)$ and $l_{p_j}(x^0,\epsilon)$, respectively.
Further, suppose the SCSC holds at $y^{*}(x^0)$ w.r.t. $\lambda(x^0)$. If there exists $\epsilon_1 > 0$ such that
\begin{equation}
    \label{check_dif}
    \begin{aligned}
     &\lambda_j(x^0) > l_{\lambda_j}(x^0,\epsilon_1) \epsilon_1 \  \text{ for all }  j \in J(x^0), \\
     &p_j(x^0,y^*(x^0)) < -l_{p_j}(x^0,\epsilon_1) \epsilon_1  \  \text{ for all }  j \not\in J(x^0),
    \end{aligned}
\end{equation}
then $y^{*}$ is continuously differentiable on $\mathcal{B}(x^0,\epsilon_1)$.

\end{proposition}

Proposition \ref{prop2} shows that, $y^{*}$ is continuously differentiable on a neighborhood of $x^0$, if for any $j$, either (i) $\lambda_j$ is larger than zero with non-trivial amount when the constraint $p_j(x^0,y^*(x^0))$ is active; or (ii) the satisfaction of $p_j(x^0,y^*(x^0))$ is non-trivial when it is inactive.
For case (i), $\lambda_j(x)>0$ and the constraint is strictly active for all $x \in \mathcal{B}(x^0,\epsilon)$;
for case (ii), $p_j(x,y^*(x))<0$ and the constraint is inactive for all $x \in \mathcal{B}(x^0,\epsilon)$.
As a illustration on the problem in Example \ref{example1} shown in Fig. \ref{ssss1.5}, $y^{*}$ is continuously differentiable on $\mathcal{B}(x^0,\epsilon)$, $\mathcal{B}(x^1,\epsilon)$ and $\mathcal{B}(x^2,\epsilon)$, and the constraint is inactive or strictly active in each ball.

{\color{black}
We evaluate the differentiability of $y^*(x)$ and $\Phi(x)$ on $\mathcal{B}(x^0,\epsilon)$ by Proposition \ref{prop2}.
In particular, we approximatively regard that $y^*$ and $\Phi$ is continuously differentiable on $\mathcal{B}(x^0,\epsilon)$ if \eqref{check_dif} is satisfied; otherwise, there exists $x \in \mathcal{B}(x^0,\epsilon)$ such that $y^*$ and $\Phi$ is not continuously differentiable at $x$.
The Lipschitz constants $l_{\lambda_j}(x^0,\epsilon)$ and $l_{p_j}(x^0,\epsilon)$ are computed as 
\begin{equation}
\label{check}
    \begin{aligned}
    &l_{\lambda_j}(x^0,\epsilon)= \| \nabla \lambda_j(x^0)\|+\delta , \\
    &l_{p_j}(x^0,\epsilon)= \|\nabla_x p_j(x^0,y^*(x^0))+    \\ & \quad \quad\quad \nabla y^{*}(x^0)^{\top}\nabla_y p_j(x^0,y^*(x^0))\|+\delta,
    \end{aligned}
\end{equation}
where $\delta$ is a small parameter, and $\nabla_x p_j(x^0,y^*(x^0))$ and $\nabla \lambda_j(x^0)$ are given in \eqref{eq8}.
Here, for a function $f$, we approximate its Lipschitz constant on $\mathcal{B}(x^0,\epsilon)$, which is defined as $l_{f}(x^0,\epsilon) \triangleq \sup_x \{ \| \nabla f(x)\|: x \in \mathcal{B}(x^0,\epsilon)\}$, as $l_{f}(x^0,\epsilon) 	\approx \| \nabla f(x^0)\|+\delta$.
As $\epsilon$ decreases, $f$ in $\mathcal{B}(x^0,\epsilon)$ is approximating to an affine function, and then the approximation error of $l_{f}(x^0,\epsilon)$ decreases.}

\subsubsection{Computation of $G(x^{0},\epsilon)$}
\label{sectionD2}
To compute $G(x^{0},\epsilon)$ in line \ref{line7} of Algorithm \ref{alg:framework0}, we need to compute the sets $I_{+}^\epsilon(x^0)$ and $I^\epsilon(x^0)$ defined in Proposition \ref{prop3}.
Similar to the idea in Proposition \ref{prop2},
we evaluate $I_{+}^\epsilon(x^0)$ and $I^\epsilon(x^0)$ as
\begin{align}
\label{check_I}
& I_+^\epsilon(x^0) = \left\{ j \in J(x^0): \lambda_j(x^0) > l_{\lambda_j}(x^0,\epsilon) \epsilon \right\} , \nonumber \\
& I^\epsilon_{-}(x^0) = \left\{ j \not\in J(x^0): \right. \nonumber \left. p_j(x^0,y^*(x^0)) < -l_{p_j}(x^0,\epsilon) \epsilon \right\}, \nonumber \\
& I^\epsilon(x^0) = \left\{ j : j \not\in I^\epsilon_+(x^0) \cup I^\epsilon_{-}(x^0) \right\}.
\end{align} 
Recall that the KKT conditions hold at $y^*(x^0)$ for problem $P(x^0)$, then for any $x \in \mathcal{B}(x^0,\epsilon)$, $p_j(x,y^*(x))=0$ for $j \in I_+^\epsilon(x^0)$ and $\lambda_j(x)=0$ for $j \in I_-^\epsilon(x^0)$.
Here, we also use $l_{\lambda_j}$ and $l_{p_j}$ given in \eqref{check} as the Lipschitz constants.
When $y^*$ and $\lambda$ are not differentiable at $x^0$, we sample a point $x^{\prime}$ near $x^0$ such that $y^*$ and $\lambda$ are differentiable at $x^{\prime}$, then replace $\nabla \lambda(x^0)$ and $\nabla y^*(x^0)$ in \eqref{check} by $\nabla \lambda(x^{\prime})$ and $\nabla y^*(x^{\prime})$. 

\section{Experiments }
\subsection{Hyperparameter Optimization}

Hyperparameter optimization has been widely studied \cite{pedregosa2016hyperparameter,franceschi2017forward,ji2021bilevel}.
However, existing methods cannot handle hyperparameter optimization of constrained learning problems, such as the supported vector machine (SVM) classification \cite{cortes1995support}, constrained reinforcement learning \cite{achiam2017constrained,chen2021primal,xu2021crpo}.
We apply the proposed algorithm to hyperparameter optimization of constrained learning. 

\begin{figure}[hbt]
\begin{center}
\begin{tabular}{cc}
\hspace{-6mm}\includegraphics[height=0.2\textwidth]{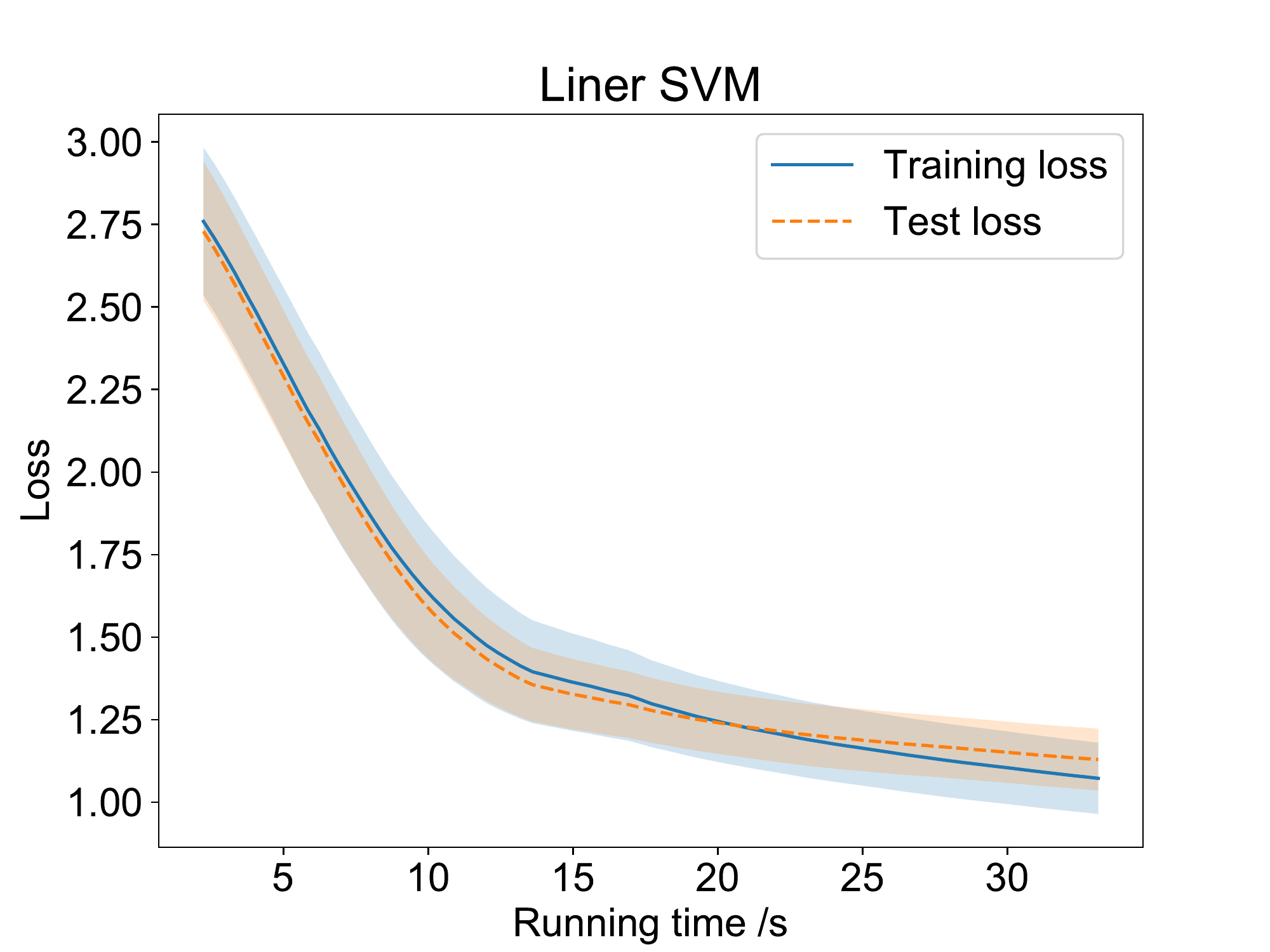} & \hspace{-8mm}
\includegraphics[height=0.2\textwidth]{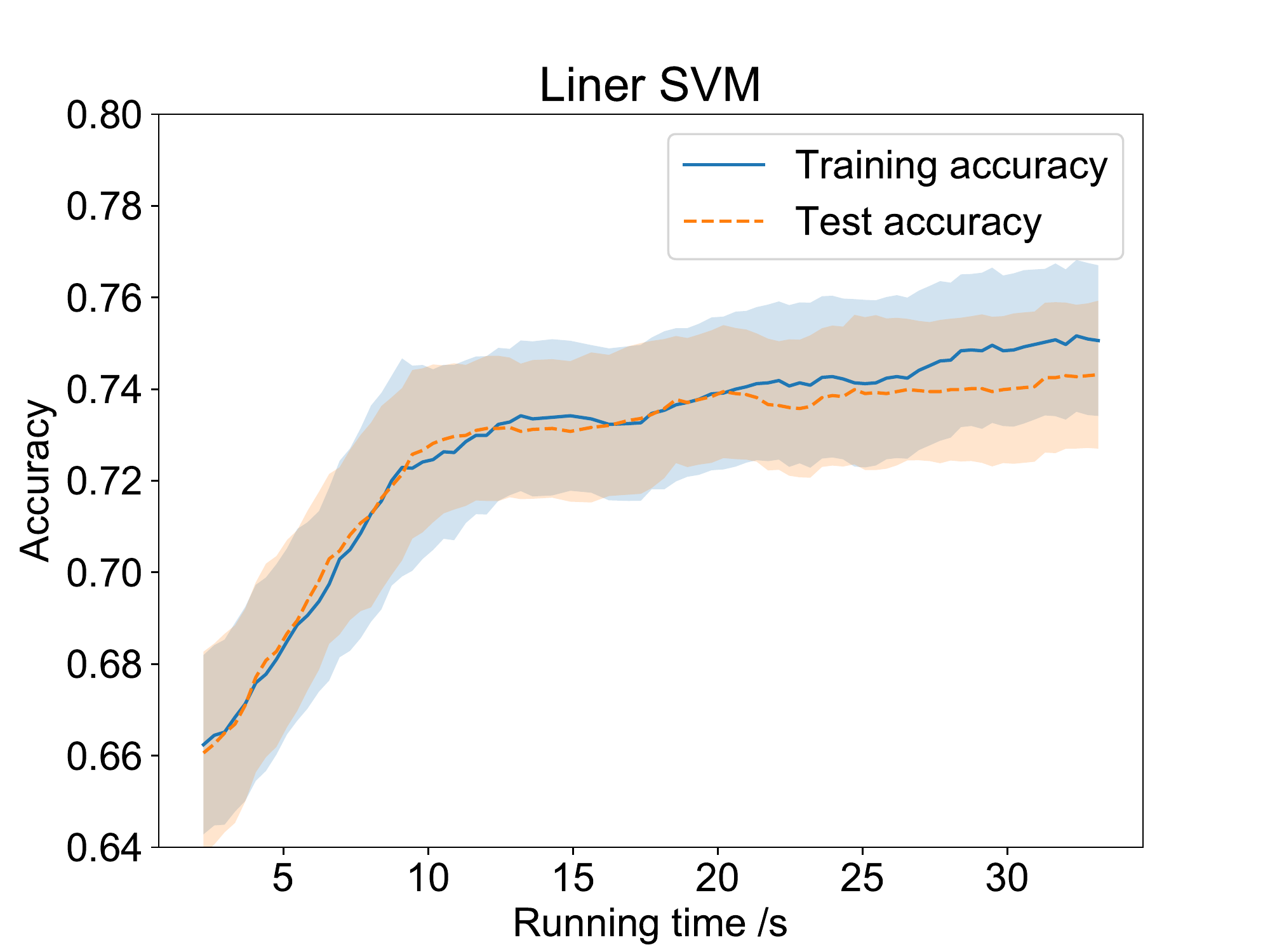}  \hspace{-8mm} \\
\hspace{-6mm}\includegraphics[height=0.2\textwidth]{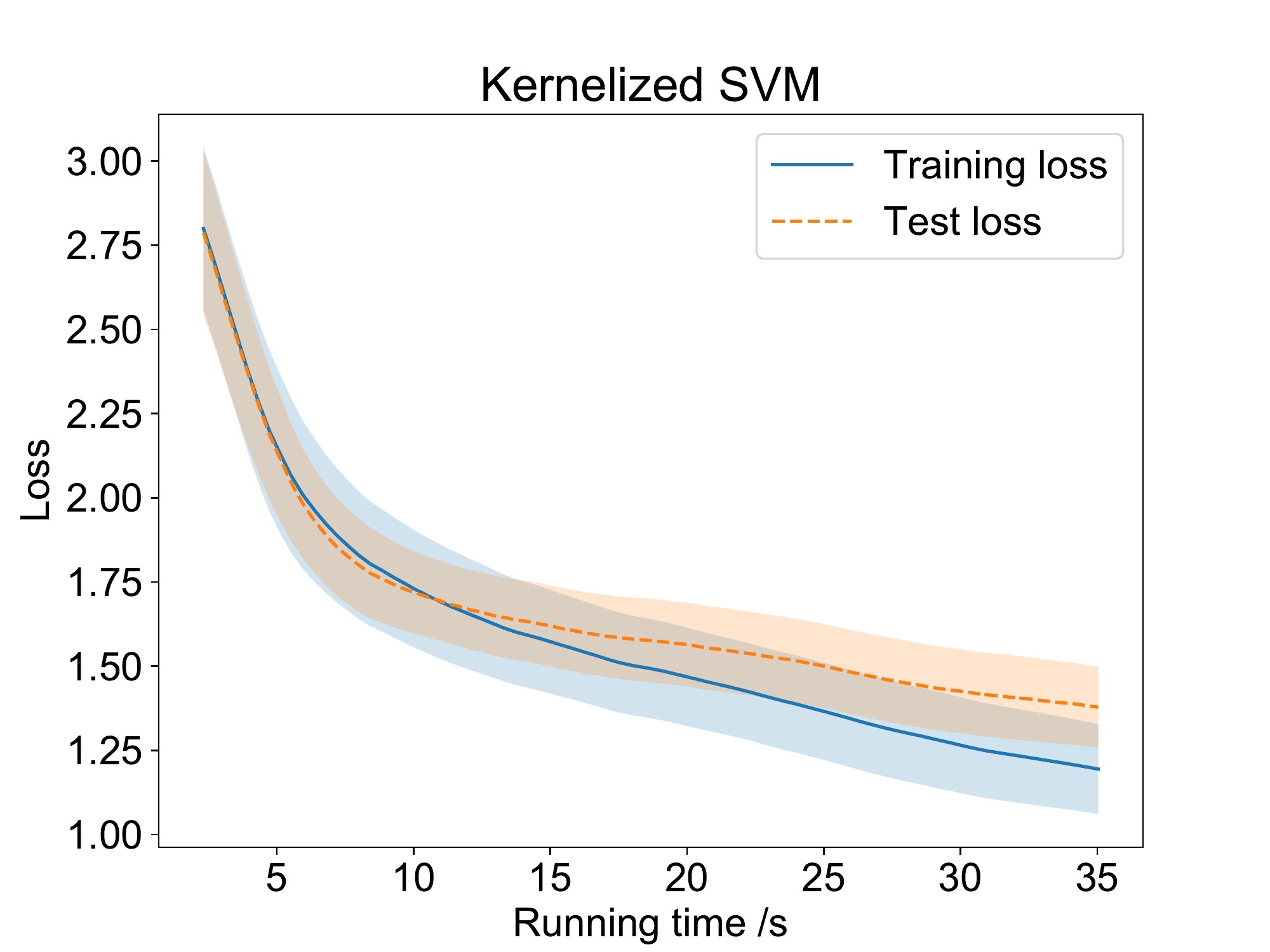} & \hspace{-8mm}
\includegraphics[height=0.2\textwidth]{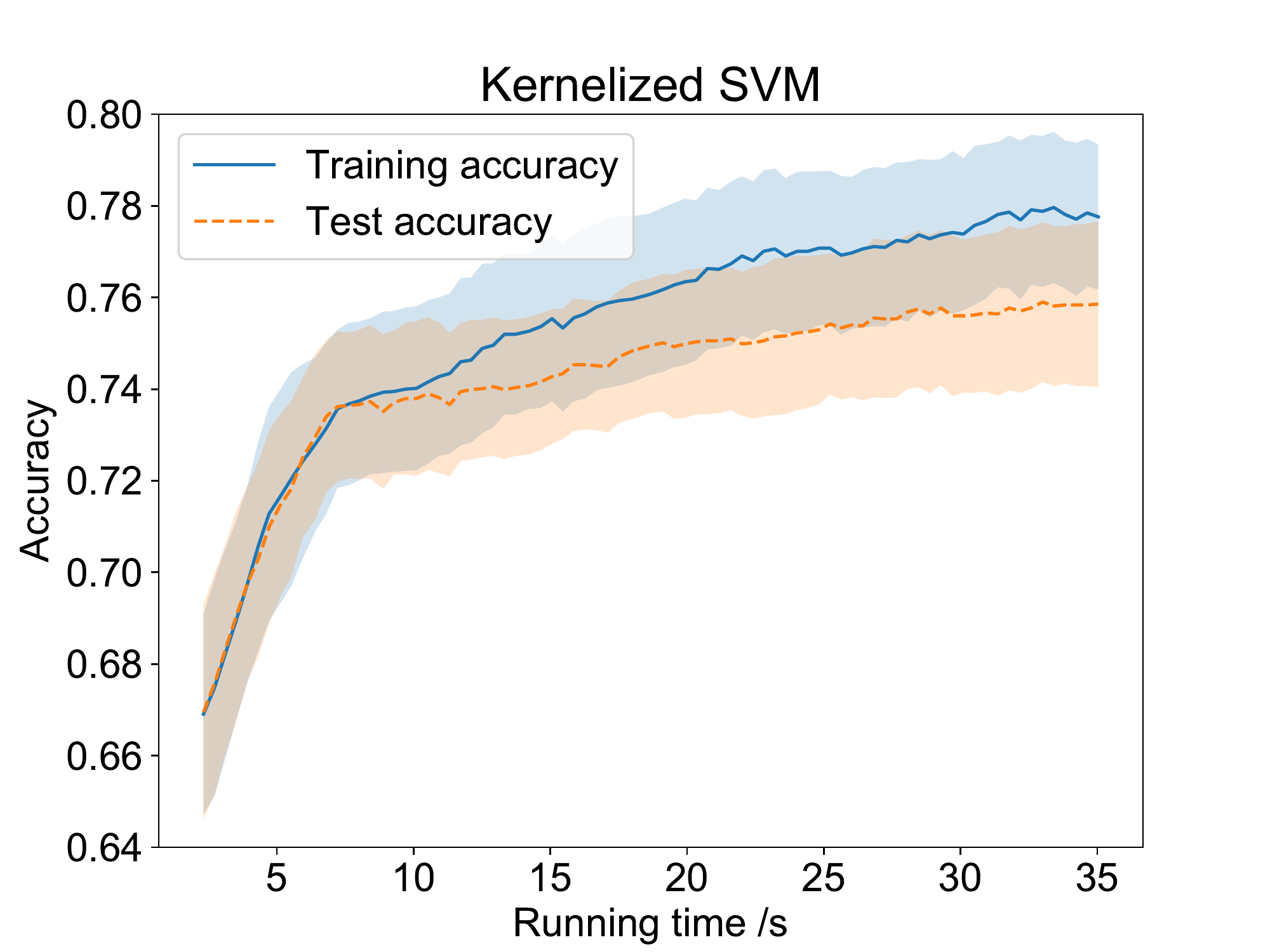}  \hspace{-8mm}
\end{tabular}
\caption{ Loss and accuracy v.s. running time in hyperparameter optimization of linear and kernelized SVM}
\label{fig:ho}
\end{center}
\end{figure}

\subsubsection{Hyperparameter Optimization of SVM}
We optimize the hyperparameter in the SVM optimization, i.e., the penalty terms of the separation violations. 
We conduct the experiment on linear SVM and kernelized SVM on the dataset of diabetes in \cite{Dua2019}.
It is the first time to solve hyperparameter optimization for SVM.
We provide details of the problem formulation and the implementation setting in Appendix \ref{A1}.
As shown in Fig. \ref{fig:ho}, the loss is nearly convergent for both linear and kernelized SVM, and the final test accuracy is much better than that of randomly selected hyperparameters, which is the initial point of the optimization.

\subsubsection{Data Hyper-Cleaning}
Data hyper-cleaning \cite{franceschi2017forward, shaban2019truncated} is to train a classifier in a setting where the labels of training data are corrupted with a probability $p$ (i.e., the corruption rate).
We formulate the problem as a hyperparameter optimization of SVM and conduct experiments on a breast cancer dataset provided in \cite{Dua2019}.
The problem formulation and the implementation setting are provided in Appendix \ref{A1}.
We compare our gradient approximation method with directly gradient descent used in \cite{amos2017optnet,lee2019meta}. 
It is shown in Fig. \ref{fig:dataclean} that, our method converges faster
than the benchmark method in terms of the loss and accuracy in both the training and test stages. Moreover, both the two methods achieve the test accuracy $96.2 \%$ using the corrupt data ($p=0.4$). The accuracy is comparable to the test accuracy $96.5\%$ of an SVM model where the data is not corrupted. 

\begin{figure}[bt]
\begin{center}
\begin{tabular}{cc}
\hspace{-6mm}\includegraphics[height=0.2\textwidth]{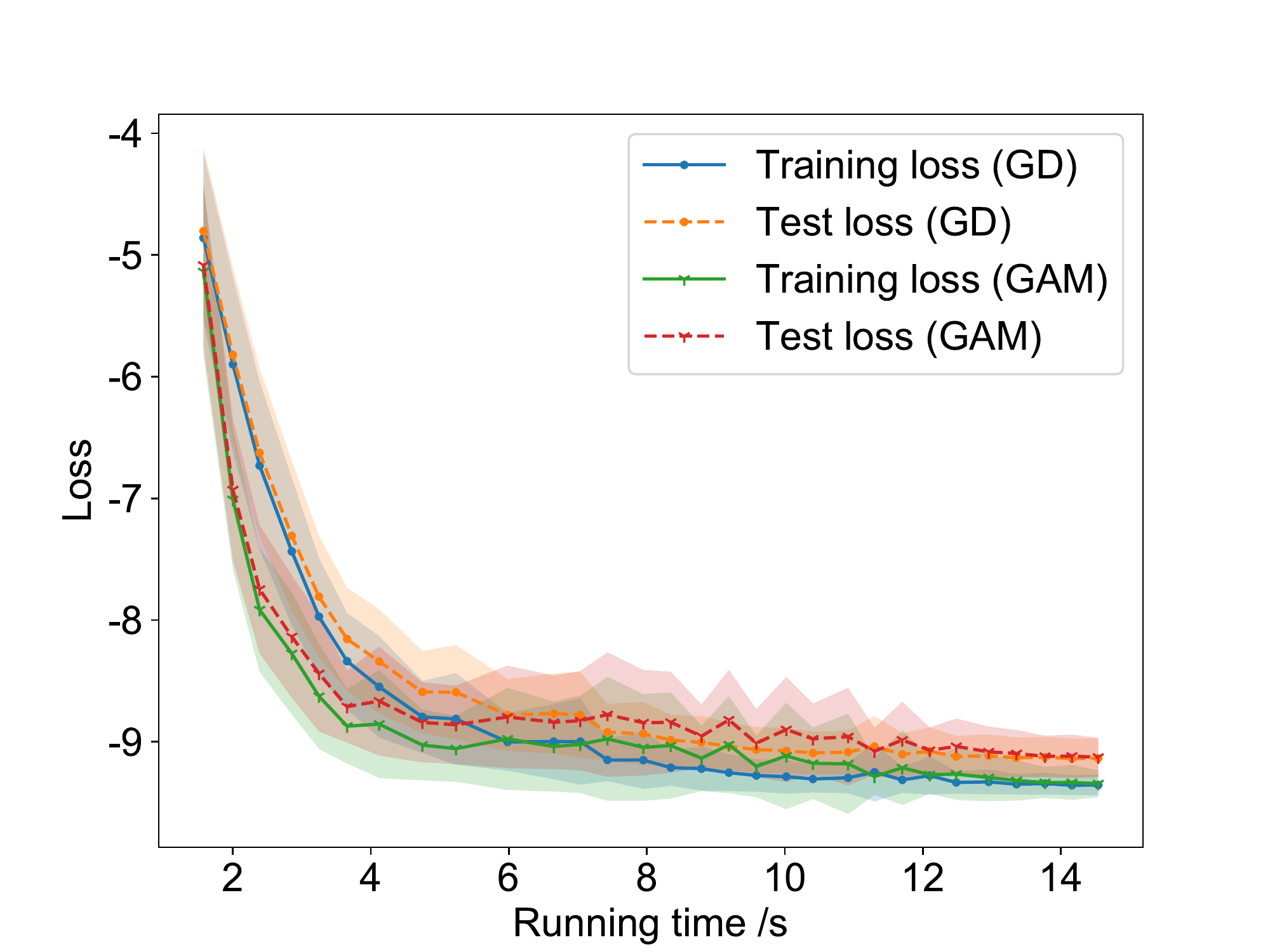} & \hspace{-8mm}
\includegraphics[height=0.2\textwidth]{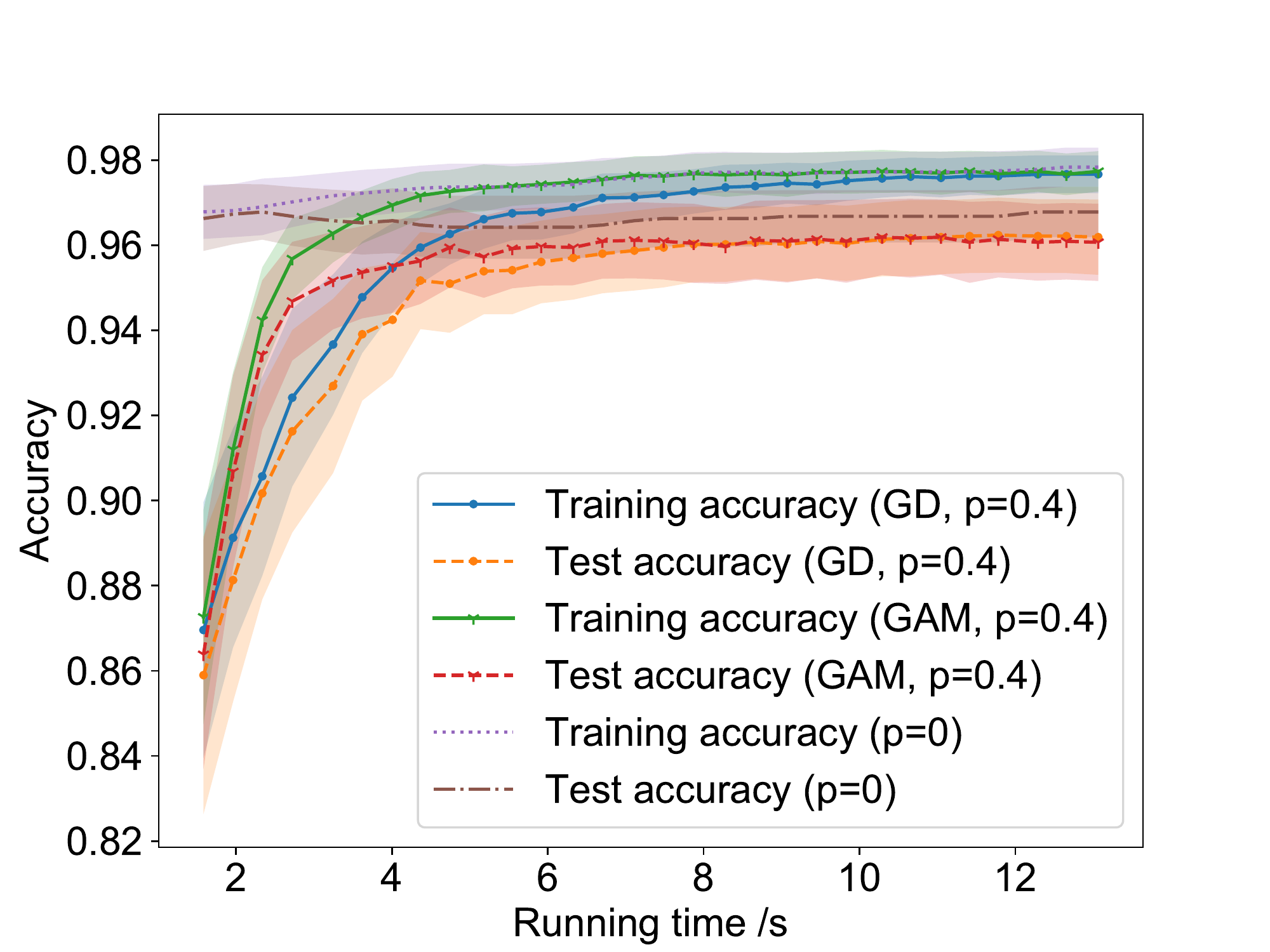}  \hspace{-8mm} 
\end{tabular}
\caption{Comparison of gradient descent (GD) and gradient approximation method (GAM) in data hyper-cleaning with the corruption rate $p=0.4$. Left: training and test losses of GD and GAM v.s. running time; right: training and test accuracy of GD and GAM with $p=0.4$ and training and test accuracy with $p=0$ v.s. running time. }
\label{fig:dataclean}
\end{center}
\end{figure}

\begin{figure}[hbt]
\begin{center}
\begin{tabular}{cc}
\hspace{-6mm}\includegraphics[height=0.2\textwidth]{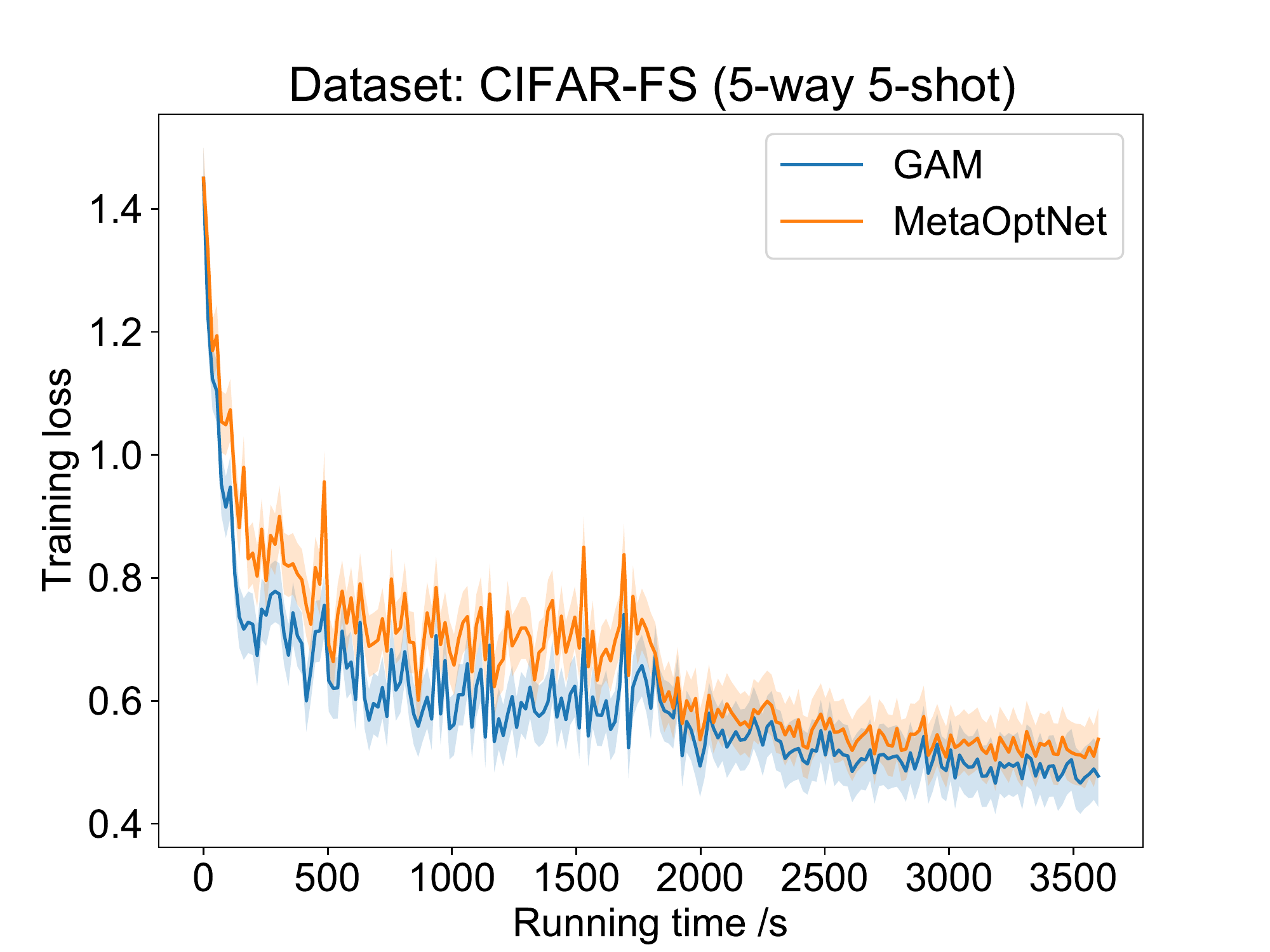} & \hspace{-8mm}
\includegraphics[height=0.2\textwidth]{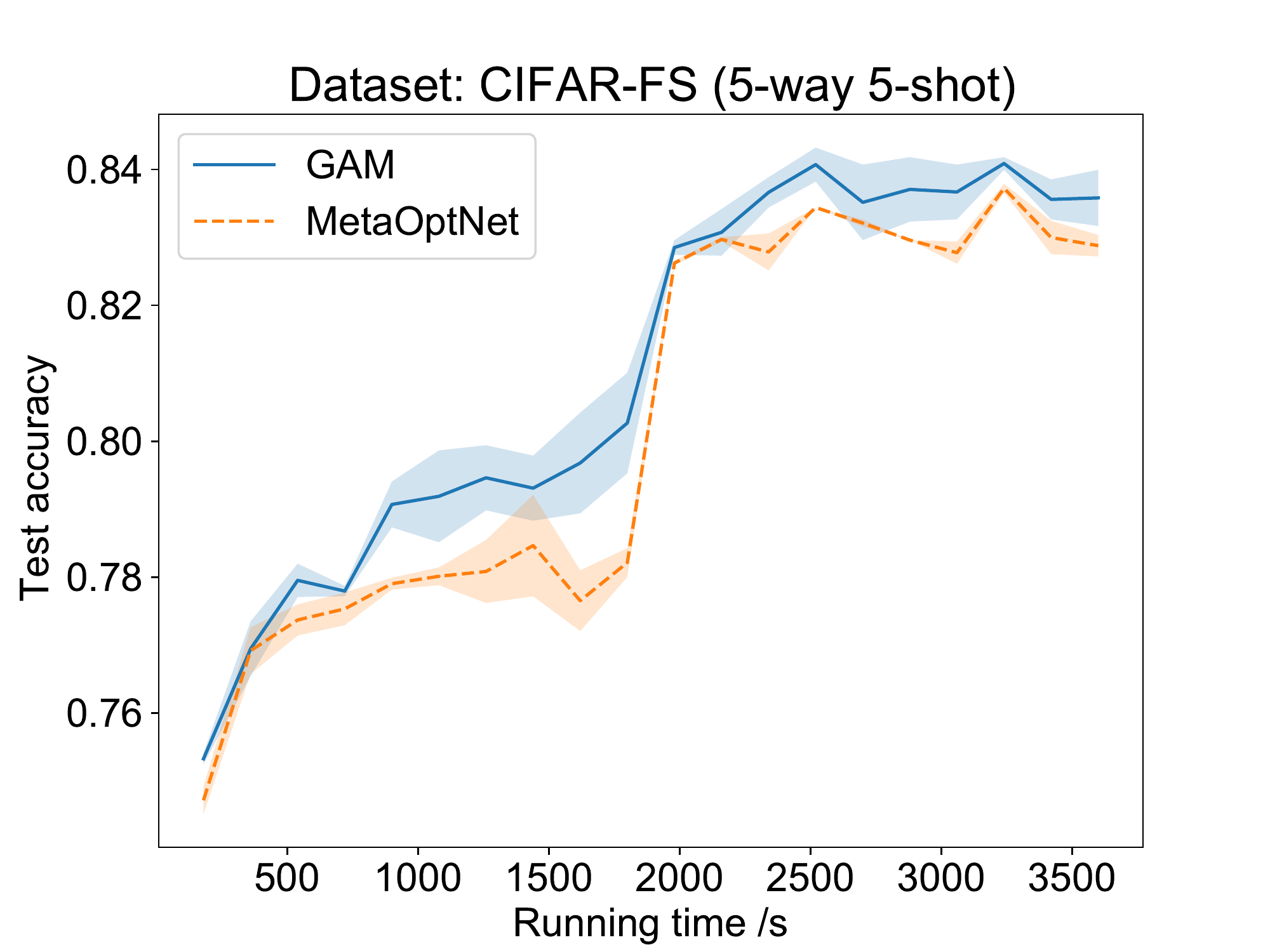}  \hspace{-8mm} \\
\hspace{-6mm}\includegraphics[height=0.2\textwidth]{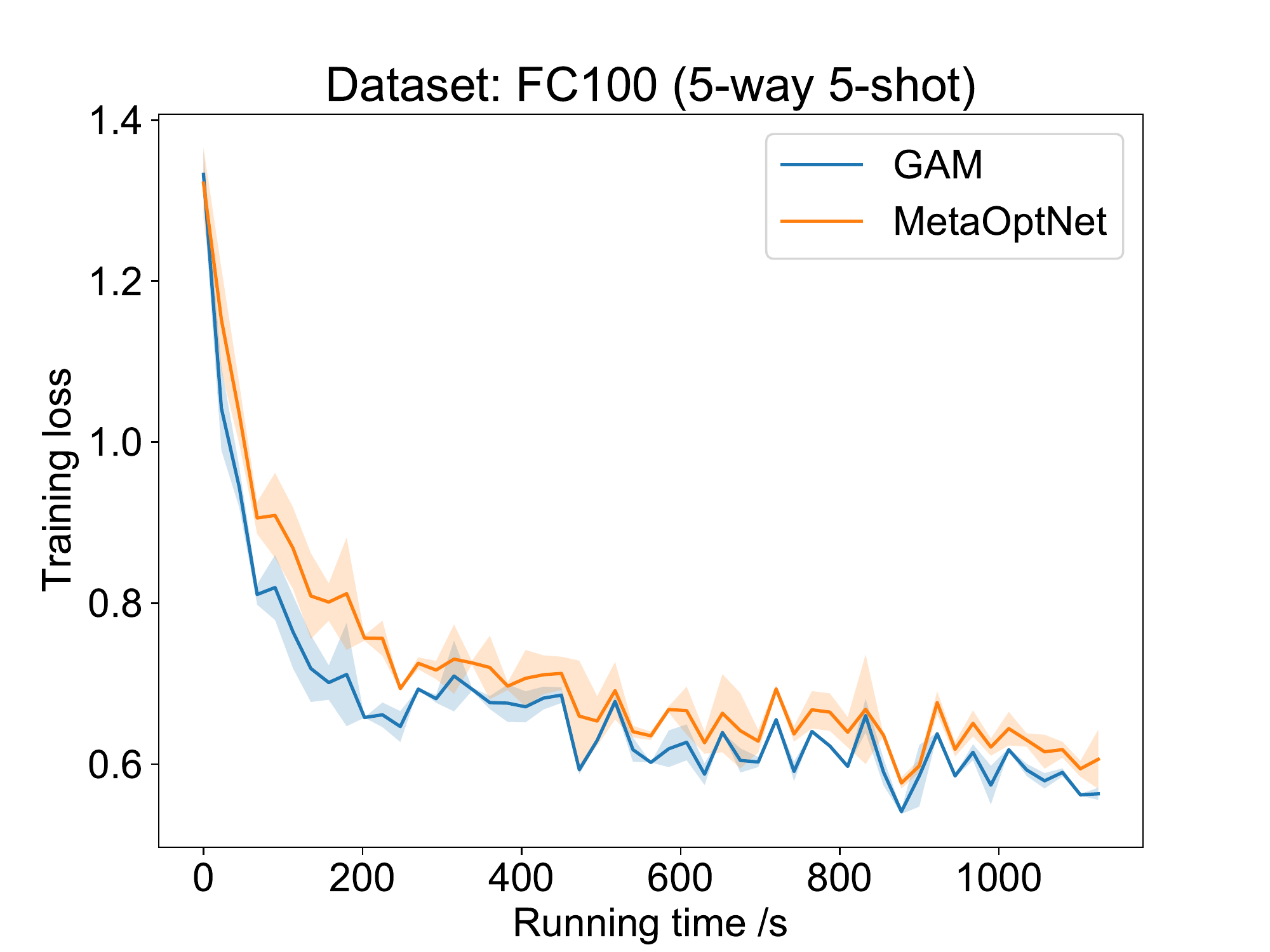} & \hspace{-8mm}
\includegraphics[height=0.2\textwidth]{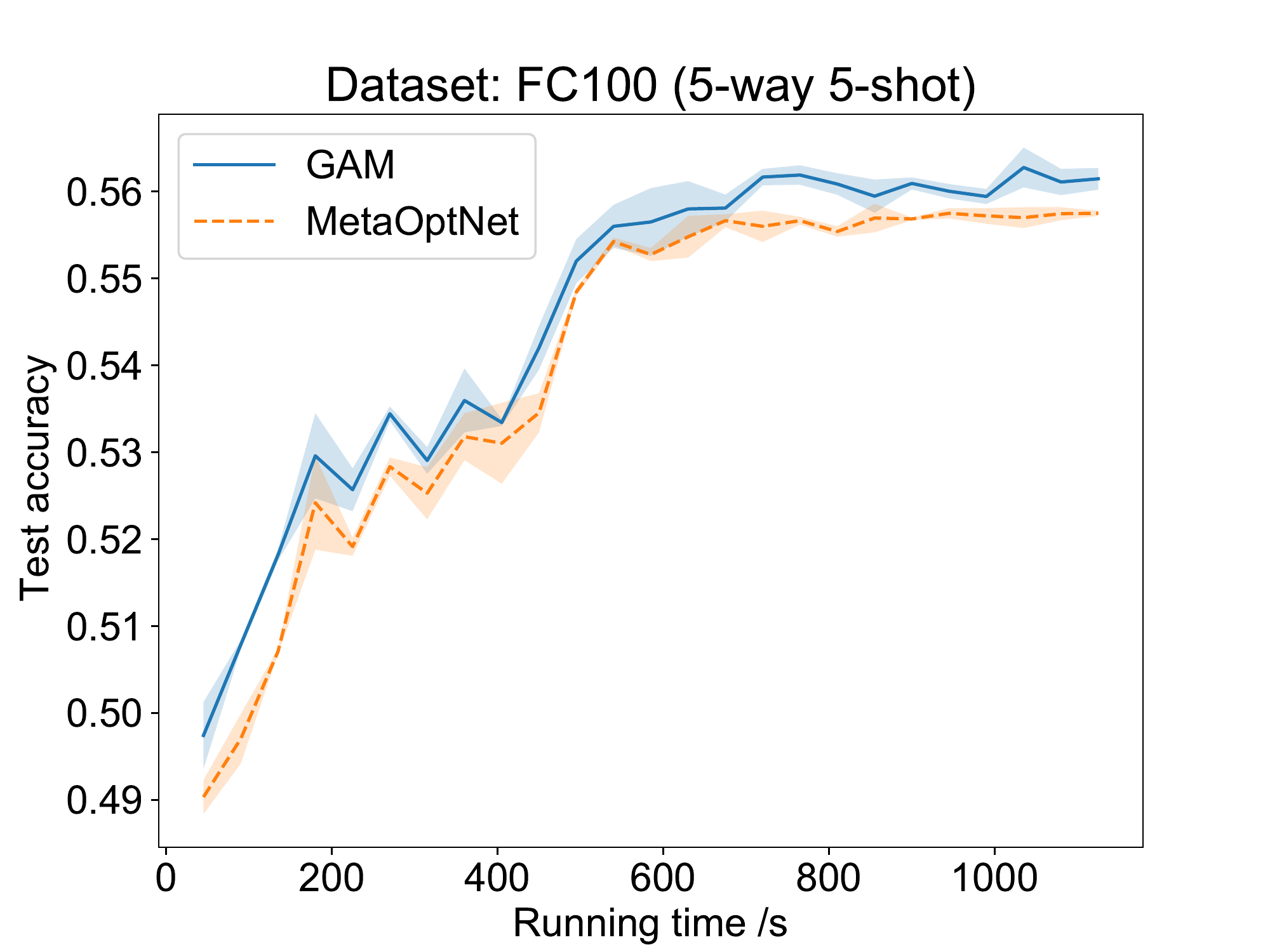}  \hspace{-8mm}
\end{tabular}
\caption{ Comparison of MetaOptNet and gradient approximation method (GAM). For each dataset, left: training loss v.s. running time; right: test accuracy v.s. running time.}
\label{fig:meta-training-shot}
\end{center}
\end{figure}

\subsection{Meta-Learning}
Meta-learning approaches for few-shot learning have been formulated as bilevel optimization problems in \cite{rajeswaran2019meta,lee2019meta, ji2021bilevel}. In particular, the problem in MetaOptNet \cite{lee2019meta} has the form of problem \eqref{biopt} with the lower-level constraints. However, its optimization does not explicitly consider the non-differentiability of the objective function and cannot guarantee convergence.
In the experiment, we compare our algorithm with the optimization in MetaOptNet on datasets CIFAR-FS \cite{R2D2} and FC100 \cite{TADAM}, which are widely used for few-shot learning. Appendix \ref{A2} provides details of 
the problem formulation and the experiment setting.

Fig. \ref{fig:meta-training-shot} compares our gradient approximation method and the direct gradient descent in MetaOptNet \cite{lee2019meta}.
The two algorithms share all training configurations, including
the network structure, the learning rate in each epoch and the batch size. 
For both CIFAR-FS
and FC100 datasets, our method converges faster
than the optimization in MetaOptNet in terms of the training loss and test accuracy, and achieves a higher final test accuracy.
Note that the only difference between the two algorithms in this experiment is the computation of the descent direction. The result shows the Clarke subdifferential approximation in our algorithm works better than the gradient as the descent direction. 
This is consistent with Proposition \ref{prop3}, where a set of representative gradients instead one gradient is more suitable to approximate 
the Clarke subdifferential.
More comparison results with other meta-learning approaches are given in Appendix \ref{A2}.

\section{Conclusion}
We develop a gradient approximation method for the bilevel optimization where the lower-level optimization problem is convex with equality and inequality constraints and the upper-level optimization is non-convex. The proposed method efficiently approximates the Clarke Subdifferential of the non-smooth objective function, and theoretically guarantees convergence. Our experiments validate our theoretical analysis and demonstrate the superior effectiveness of the algorithm. 

\section{Acknowledgements}
This work was partially supported by NSF awards ECCS 1846706 and ECCS 2140175.

\bibliography{aaai23}

\onecolumn
\appendix
\noindent {\Large \textbf{Supplementary Materials}}

\section{Implementation Supplement}
\subsection{Computation of gradient matrices}
In lines \ref{line5} and \ref{line7} of Algorithm \ref{alg:framework0}, we need to compute the gradient matrices $\nabla \Phi(x^k)$ and the set of gradients $G(x^{k},\epsilon_k)$.
Notice that both $-M_{+}^{-1}(x) N_{+}(x)$ in the computation of $\nabla \Phi(x^k)$ in \eqref{eq8} and $-M^{S}_{\epsilon}(x^0,y^*(x^0))^{-1} N^{S}_{\epsilon}(x^0,y^*(x^0))$ in the computation of $G(x^{0},\epsilon)$ in \eqref{eq17} have a form of
\begin{equation}
\label{form+}
-\left[\begin{array}{ccccccc}
\nabla_{y}^{2} \mathcal{L}  & \nabla_{y} r^{\top}\\
\nabla_{y} r & 0
\end{array}\right]^{-1}
\left[\begin{array}{c}
\nabla_{x y}^{2} \mathcal{L}\\
\nabla_{x} r
\end{array}\right],
\end{equation}
where $\nabla_{y}^{2} \mathcal{L}^{-1}$ is positive definite at $(y^*(x), \lambda(x), \nu(x), x)$ for any $x$ (shown in the proof (ii) of Theorem \ref{th2} in Appendix).
We can compute \eqref{form+} as follows. First, as $\nabla_{y}^{2} \mathcal{L}^{-1}$ is positive definite, we can use the conjugate gradient (CG) method \cite{hestenes1952methods} to compute $A = \nabla_{y}^{2} \mathcal{L}^{-1} \nabla_{x y}^{2} \mathcal{L}$ and $B= \nabla_{y}^{2} \mathcal{L}^{-1} \nabla_{y} {r}^{\top}$. Second, \eqref{form+} can be written as
\begin{equation}
\label{form1+}
\left[\begin{array}{c}
-A+B(\nabla_{y}{r} B)^{-1} (\nabla_{y}{r} A - \nabla_{x} r) \\
-(\nabla_{y}{r} B)^{-1} (\nabla_{y}{r} A - \nabla_{x} r)
\end{array}\right].
\end{equation}
Let the number of strictly active constraints be $m^{+}$ and the number of non-strictly active constraints be $m^{-}$ in $P(x)$, and $n$ is the number of equality constraints.
Then, $r(x,y)$ is a vector function with at most $n+m^{+}+m^{-}$ dimensions.
The dimension of $\nabla_{y}^{2} \mathcal{L}$ is $d_y$. In machine learning applications, $d_x$ is usually large and $n+m^{+}+m^{-}$ is relatively small. It is shown in \cite{ji2021bilevel} that the computation of $A$ and $B$ in the first step is achievable, and thus so is the computation of \eqref{form1+}.

Paper \cite{amos2017optnet} provides a highly efficient solver to compute the gradient of the solution of a quadratic program, which exploits fast GPU-based batch solves within a primal-dual interior point method. The tool can also be used to compute $-M_{+}^{-1}(x) N_{+}(x)$ and $-M^{S}_{\epsilon}(x^0,y^*(x^0))^{-1} N^{S}_{\epsilon}(x^0,y^*(x^0))$ in Algorithm \ref{alg:framework0}, when the lower-level optimization problem is a quadratic program.

\subsection{Satisfaction of Assumptions}

Here are two reminders of Assumption \ref{a3}.
\begin{remark}
A sufficient condition of Assumption \ref{a3} which is easier to check is that, the solution $y^{*}(x)$ exists for $P\left(x\right)$ and the LICQ holds at $y$ for $P(x)$ for all ${x} \in \mathbb{R}^{d_x}$ and $y \in \mathbb{R}^{d_y}$.
\end{remark}
\begin{remark}
The equality constraint cannot be replaced by two inequality constraints, i.e., $p = 0$ is replaced by $p \leq 0$ and $-p \leq 0$. Otherwise, the LICQ is violated.
\end{remark}
Note that all optimization problems in the experiments of this paper satisfy Assumptions \ref{a1}, \ref{a2}, \ref{a3}. The details are shown in Appendix \ref{A0}.

\section{Experimental Supplement}\label{sc: expsetting}
\label{A0}
All experiments are executed on a computer with a 4.10 GHz Intel Core i5 CPU and an RTX 3080 GPU.

\subsection{Hyperparameter Optimization}
\label{A1}

In a machine learning problem, given a hyperparameter $\Lambda$, the learner is to minimize the training error and the optimal parameter is denoted as $w^*(\Lambda)$.
Hyperparameter optimization is to search for the best hyperparameter $\Lambda^{*}$ for the learning problem, which can be formulated as a bilevel  optimization problem. 
In particular, it is to minimize the validation error of the learner’s parameter $w^{*}(\Lambda)$ in the upper-level optimization, where $w^{*}(\Lambda)$ is the minimizer of training error in the lower-level optimization under the hyperparameter $\Lambda$. 

Hyperparameter optimization has been widely studied in \cite{pedregosa2016hyperparameter,franceschi2017forward, franceschi2018bilevel,lorraine2020optimizing,ji2021bilevel}.
However, these methods cannot handle with hyperparameter optimization of constrained learning problems, such as the supported vector machine (SVM) classification \cite{cortes1995support}, safe reinforcement learning \cite{achiam2017constrained,chen2021primal,xu2021crpo}.
We apply the proposed algorithm to hyperparameter optimization of constrained learning problem, which is formulated as
$$
\begin{aligned}
&\min _{\Lambda}  \ \Phi(\Lambda)  = \mathcal{L}_{\mathcal{D}_{\text {val }}}(w^{*})=\frac{1}{\left|\mathcal{D}_{\text {val }}\right|} \sum_{z \in \mathcal{D}_{\text {val }}} \mathcal{L}\left(w^{*} ; z \right) \\
& \text { s.t. }  w^{*}=\underset{w}{\arg \min } \{  \mathcal{F}_{\mathcal{D}_{\text {tr }}}(\Lambda, w): 
p\left(\Lambda, w\right) \leq 0; q\left(\Lambda, w\right) = 0 \},
\end{aligned}
$$
where $\mathcal{D}_{\text {val }}$ and $\mathcal{D}_{\text {tr }}$ are validation and training data, $\mathcal{L}_{\mathcal{D}_{\text {val }}}$ is the loss function of model parameter $\omega$ on data $\mathcal{D}_{\text {val }}$.
The lower-level optimization is the training of model parameter $w$, where $\mathcal{F}_{\mathcal{D}_{\text {val }}}$ is the training objective on $\mathcal{D}_{\text {tr }}$ and $p$, $q$ are the constraints.

\subsubsection{Hyperparameter Optimization of SVM}
The optimization problem for SVM is: 
\begin{equation}
\label{eqsvm_opt}
\begin{aligned}
(w^{*}, b^{*}, {\xi}^{*}) =&\arg\min_{w, b, \xi} \ \frac{1}{2}\|w\|^{2}+ \frac{1}{2}\sum_{i=1}^{N} {e^{c_i}} \xi_{i}^2 \\
& \text { s.t. } \  l_{i}\left(w^{\top} \phi(z_{i})+b\right) \geq 1-\xi_{i}, \ i=1,2, \ldots N,\\
\end{aligned}
\end{equation}
Here, $z_i$ is the data point and $y_i$ is the label, and $(z_i;l_i) \in \mathcal{D}_{\text {tr }}$ for all $1 \leq i \leq N$. and $(z_i;y_i) \in \mathcal{D}_{\text {tr }}$ for all $1 \leq i \leq N$. The vector function $\phi(z_i)$ is the high dimension feature for point $z_i$. The kernel function is defined as $K\left(z_{i}, z_{j}\right)=\phi\left(z_{i}\right)^{T} \phi\left(z_{j}\right)$.
The hyperparameter optimization of SVM is formulated as 
\begin{equation}
\label{ho_svm_upper}
    \min _{c}  \ \Phi(c)  =  \mathcal{L}_{\mathcal{D}_{\text {val }}}(w^{*},b^{*}),
\end{equation}
where $w^{*},b^{*}$ are given in \eqref{eqsvm_opt} and the optimized hyperparameter is $c=[c_1, \dots, c_N]$.
To satisfy Assumption \ref{a2}, we set the objective function of \eqref{eqsvm_opt} as $\frac{1}{2}\|w\|^{2}+ \frac{1}{2} \sum_{i=1}^{N} {e^{c_i}} \xi_{i}^2 + \frac{1}{2} \mu b^2$ where $\mu$ is a small positive number. Then, the objective function is strongly-convex w.r.t $(w, b, \xi)$. It is easy to justify the LICQ in Assumption \ref{a3} is satisfied. 

The upper-level objective function is defined as 
$\mathcal{L}_{\mathcal{D}_{\text {val }}}(w^{*},b^{*})= \frac{1}{\left|\mathcal{D}_{\text {val }}\right|} \sum_{(z,l) \in \mathcal{D}_{\text {val }}} \mathcal{L}\left(w^{*},b^* ; z,l \right)$,
where $\mathcal{L}(w^{*},b^{*} ; \mathcal{D}_{\text {val }})$ is defined as 
$\mathcal{L}\left(w^{*},b^* ; z,l \right)=\sigma((\frac{-l(z^{\top}{w^{*}}+b)}{\|w^{*}\|})$ and $\sigma(x)=\frac{1-e^{-x}}{1+e^{-x}}$.
Here, $\frac{l(z^{\top}{w^{*}}+b)}{\|w^{*}\|}$ is the signed distance between point $z$ and the decision plane $z^{\top}{w^{*}}+b=0$, where $\frac{l(z^{\top}{w^{*}}+b)}{\|w^{*}\|}>0$ when the prediction is correct and $\frac{l(z^{\top}{w^{*}}+b)}{\|w^{*}\|}<0$ when the prediction is incorrect. Thus, $\mathcal{L}_{\mathcal{D}_{\text {val }}}(w^{*},b^{*})$ is a differentiable surrogate function of the validation accuracy.

When the feature function $\phi$ is not tractable, the hyperparameter $c$ of \eqref{eqsvm_opt} can not be directly optimized.
For example, in kernelized SVM \cite{10.1214/009053607000000677}, under most kernel functions, e.g., Gaussian kernel and polynomial kernel $\phi$ are  unknown or very complex. 
Then, it is hard to compute $\nabla w^{*} (c)$ by Theorem \ref{th1}. 
Therefore, in kernelized SVM, we solve the dual problem of \eqref{eqsvm_opt}:
\begin{equation}
\label{eqsvm_opt_dual}
\begin{aligned}
\min _{\alpha} \ & \frac{1}{2} \alpha^{\top} (Q + C^{-1}) \alpha -\sum_{i=1}^{n} \alpha_{i}\\
\text { s.t. } & \ \sum_{i=1}^{n} \alpha_{i} y_{i}=0 \\
&\alpha_{i} \geq 0, \ i=1,2, \ldots N,
\end{aligned}
\end{equation}
where $Q$ is an $n$ by $n$ positive semi-definite matrix with $Q_{i j} \equiv y_{i} y_{j} K\left(z_{i}, z_{j}\right)$, and $C=\operatorname{diag}\left(e^{c_{1}}, \ldots, e^{c_{n}}\right)$.
Since the strong duality holds for problem (\ref{eqsvm_opt}), we have 
$$
w^{*}=\sum_{i=1}^{n} \alpha_{i}^{*} y_{i} \phi\left(z_{i}\right)
$$
and
$$
b^{*}=y_{i}(1-e^{-c_i}\alpha_i)-\sum_{j=1}^{n} \alpha_{j}^{*} y_{j}\phi\left(z_{j}\right)^{\top} \phi\left(z_{i}\right)=y_{i}(1-e^{-c_i}\alpha_i)-\sum_{j=1}^{n} \alpha_{j}^{*} y_{j} K\left(z_{j}, z_{i}\right)
$$
for any support vector $z_{i}$ with $\alpha_{i}^{*}>0$.
Following the kernel method, the computation of $\phi$ is not required for both processes of model training and prediction.
The prediction of $z^{new}$ is
$$\operatorname{sign}\{ {w^{*}}^{\top}\phi(z^{new})+b^*\}=\operatorname{sign}\{\sum_{i=1}^{n} \alpha_{i}^{*} y_{i} K\left(z_{i},z_{new}\right)+b^{*}\}.$$
Assumption \ref{a2} is satisfied, since the objective function $\frac{1}{2} \alpha^{\top} (Q+C^{-1}) \alpha -\sum_{i=1}^{n} \alpha_{i}$ where $Q$ is a positive semi-definite matrix and $C^{-1}=\operatorname{diag}\left(e^{-c_{1}}, \ldots, e^{-c_{n}}\right)$ is positive definite, then the objective function is strongly-convex. Since there exists $i$ such that $\alpha_{i}^{*} > 0$, then it is easy to justify the LICQ in Assumption \ref{a3} is satisfied.

In the experiment, we consider hyperparameter optimization of the linear SVM model and the kernelized  SVM model.
\textbf{Linear SVM: } The feature function is $\phi(x)=x$. Both lower-level problems (\ref{eqsvm_opt}) and (\ref{eqsvm_opt_dual}) works for hyperparameter optimization of Linear SVM. Here, we solve the bilevel problem \eqref{eqsvm_opt}.
\textbf{kernelized SVM: } We apply the polynomial kernel, i.e., $K(z, z^{\prime})=\phi(z)^{T} \phi(z^{\prime})=(\gamma z^{\top}z^{\prime}+r)^{d}$, where $\gamma=1$ and $d=3$. We test our algorithm on a diabetes dataset in \cite{Dua2019}.
For Algorithm \ref{alg:framework0}, we set $\gamma=0.3$, $\epsilon_0=0.3$, $\beta=0.5$ and fix the total iteration number as 60.

\subsubsection{Data Hyper-Cleaning}
We formulate the data hyper-cleaning as the hyperparameter optimization of SVM, where the upper-level optimization problem is shown in \eqref{ho_svm_upper} and the lower-level optimization problem is shown in \eqref{eqsvm_opt}. After the optimization of the hyperparameter $c$, the penalty term $e^{c_i}$ which corresponds to the corruption data $(z_i,y_i)$ will be close to $0$. Thus, the corruption data $(z_i,y_i)$ is detected and almost does no affect training and the prediction of the classifier model. We conduct experiments on a dataset of breast cancer provided in \cite{Dua2019}.
For Algorithm \ref{alg:framework0}, we set $\gamma=0.3$, $\epsilon_0=0.3$, $\beta=0.5$ and fix the total iteration number as 30.

\subsection{Meta-learning}
\label{A2}
Meta-learning for few-shot learning is to learn a shared prior parameter across a distribution of tasks, such that a simple learning step with few-shot data based on the prior leads to a good adaptation to the task in the distribution. In particular, the training task $\mathcal{T}_{i}$ is sampled from distribution $P_{\mathcal{T}}$. Each task $\mathcal{T}_{i}$ is characterized by its training data $\mathcal{D}_{i}^{tr}$ and test data $\mathcal{D}_{i}^{test}$. 
The goal of meta-learning is to find a good parameter ${\phi}$ and a base learner $\mathcal{A}$, such that the task-specific parameter $w^{i}=\mathcal{A}({\phi}, \mathcal{D}_{i}^{tr})$ has a small test loss $\mathcal{L}(w^{i}, \phi, \mathcal{D}_{i}^{test})$. 

The training of meta-learning can be formulated as a constrained bilevel optimization problem \cite{lee2019meta}. The upper-level optimization is to extract features from the input data. The multi-class SVM served as the base learner $\mathcal{A}$ in the lower-level optimization to classify the data on its extracted features. In particular, the feature extraction model $f_{{\phi}}$ maps from the image $x_n$ to its features denoted as $f_{{\phi}}(x_n)$.
The multi-class SVM in the lower-level optimization is a constrained optimization problem:
\begin{equation}
\label{metalearningopt}
\begin{aligned}
w^{i}=\mathcal{A}\left(\mathcal{D}_i^{\text {tr }} ; {\phi}\right)=\underset{\{{w}_{k}\},\{\xi_{n}\}}{\arg \min }  \frac{1}{2} \sum_{k}\left\|{w}_{k}\right\|_{2}^{2}+C \sum_{n} \xi_{n} \\
\text{s.t. }
{w}_{y_{n}} \cdot f_{{\phi}}\left({x}_{n}\right)-{w}_{k} \cdot f_{{\phi}}\left({x}_{n}\right) \geq 1-\delta_{y_{n}, k}-\xi_{n}, \forall n, k 
\end{aligned}
\end{equation}
where $\mathcal{D}^{ {tr }}_i=\left\{\left({x}_{n}, y_{n}\right)\right\}$ with image ${x}_{n}$ and its label $y_n$, $C$ is the regularization parameter and $\delta_{\cdot,\cdot}$ is the Kronecker delta function. Here, $k$ is the index of feature $f_{{\phi}}(x_n)$, and $n$ is the index of the data.
The upper-level optimization is
\begin{equation}
\label{metalearningupper}
\min_{\phi} \sum_{\mathcal{T}_{i} \in P_{\mathcal{T}}} \mathcal{L}\left(w^i, \phi, \mathcal{D}^{ {test }}_i \right).
\end{equation}
where
$$
\mathcal{L}\left(w^i, \phi, \mathcal{D}^{ {test }}_i \right)= 
\sum_{({x}, y) \in \mathcal{D}^{ {test }}}\left[-\gamma {w}_{y}^i \cdot f_{\phi}({x})+\log \sum_{k} \exp \left(\gamma {w}_{k}^i \cdot f_{\phi}({x})\right)\right]
$$
and $w^i$ is given in \eqref{metalearningopt}. Here, $\mathcal{L}\left(w^i, \phi, \mathcal{D}^{ {test }}_i \right)$ is the negative log-likelihood loss under the feature extraction parameter ${{\phi}}$ and the SVM parameter ${w}^i$ optimized in lower-level optimization \eqref{metalearningopt}.
Then, the upper-level optimization \eqref{metalearningupper} to find the best feature extraction parameter ${{\phi}}$.

Following the experiment setting in MetaOptNet \cite{lee2019meta}, we use a ResNet-12 network as the feature extraction mapping $f_{{\phi}}()$. Since the line search in Algorithm \ref{alg:framework0} is not convenient for the meta-learning problem, we compute the descent direction in \ref{alg:framework0} and use the SGD method with Nesterov momentum of 0.9 and weight decay of 0.0005 to solve the problem. Each mini-batch
consists of 8 episodes. The model was meta-trained for 30 epochs, with each epoch consisting of 1000 episodes. The
learning rate was initially set to 0.1, and then changed to
0.006, 0.0012 at epochs 10, 20 and 25, respectively. We use the configurations for both our method and the method in MetaOptNet \cite{lee2019meta}.

The comparison to previous work on CIFAR-FS and FC100 in the aspect of prediction accuracy is shown in Table \ref{tab:CIFAR}. It is shown that the final test accuracy of our optimization algorithm is slightly better than that of MetaOptNet-SVM \cite{lee2019meta}. 

\begin{table*}[htb]
\caption{\textbf{Comparison to prior work on CIFAR-FS and FC100.} Average few-shot classification accuracies (\%) with 95\% confidence intervals on CIFAR-FS and FC100. \textsuperscript{$\ast$}CIFAR-FS results from \cite{R2D2}. FC100 result from \textsuperscript{$\dagger$}\cite{TADAM} and \textsuperscript{$\ddag$}\cite{ji2021bilevel}. All models are trained on the original training data of CIFAR-FS and FC100 in \cite{lee2019meta}, and validation data are not included.} 
\label{tab:CIFAR}
\begin{center}
\begin{small}
\begin{tabular}{@{}llc@{}cc@{}c@{}cc@{}}
\hline
\toprule
&  & \multicolumn{2}{c}{\textbf{CIFAR-FS 5-way}} & \phantom{ab} & \multicolumn{2}{c}{\textbf{FC100 5-way}} \\
\cmidrule{3-4} \cmidrule{6-7}
\textbf{model} && \textbf{1-shot} & \textbf{5-shot} && \textbf{1-shot} & \textbf{5-shot}  \\
\hline
MAML\textsuperscript{$\ast$}\textsuperscript{$\ddag$} \cite{MAML}  &&  58.9 $\pm$ 1.9 \quad & \quad 71.5 $\pm$ 1.0  && - \quad & \quad 47.2 \\
Prototypical Networks\textsuperscript{$\ast$}\textsuperscript{$\dagger$} \cite{proto-net} && 55.5 $\pm$ 0.7 \quad & \quad 72.0 $\pm$ 0.6 && 35.3 $\pm$ 0.6 \quad & \quad 48.6 $\pm$ 0.6 \\
Relation Networks\textsuperscript{$\ast$} \cite{sung2018learning}  && 55.0 $\pm$ 1.0 \quad & \quad 69.3 $\pm$ 0.8 && - \quad & \quad - \\
R2D2 \cite{R2D2}  && 65.3 $\pm$ 0.2 \quad & \quad 79.4 $\pm$ 0.1 && - \quad & \quad -\\
TADAM \cite{TADAM}  && - \quad & \quad - && 40.1 $\pm$ 0.4 \quad & \quad 56.1 $\pm$ 0.4\\
ProtoNets \cite{proto-net} && \textbf{72.2 $\pm$ 0.7} \quad & \quad 83.5 $\pm$ 0.5 && 37.5 $\pm$ 0.6 \quad & \quad 52.5 $\pm$ 0.6 \\
MetaOptNet-RR \cite{lee2019meta}   && \textbf{72.6 $\pm$ 0.7} \quad & \quad \textbf{84.3 $\pm$ 0.5} && 40.5 $\pm$ 0.6 \quad & \quad 55.3 $\pm$ 0.6\\
MetaOptNet-SVM \cite{lee2019meta}  && \textbf{72.0 $\pm$ 0.7} \quad & \quad \textbf{84.2 $\pm$ 0.5} && 41.1 $\pm$ 0.6 \quad & \quad 55.5 $\pm$ 0.6 \\
MetaOptNet-SVM-GAE (ours)  && \textbf{72.2 $\pm$ 0.7} \quad & \quad \textbf{84.6 $\pm$ 0.6} && \textbf{41.9 $\pm$ 0.6} \quad & \quad \textbf{56.4 $\pm$ 0.8} \\
\bottomrule
\hline
\end{tabular}
\end{small}
\end{center}
\end{table*}

\section{Proof and Analysis}

Firstly, we clarify notations used in this section.
All notations used in this section are the same as in the main body of the paper, except $J(x)$, $J^{+}(x)$ and $J^{0}(x)$ in Section \ref{section3}. In Section \ref{section3}, 
we simplify notations in \eqref{def1} and denote $J(x,y^{*}(x))$ as $J(x)$, $J^{+}(x,y^{*}(x),{\lambda}(x))$ as $J^{+}(x)$, and $J^{0}(x,y^*(x),\lambda(x))$ as $J^{0}(x)$, because the optimal solution $y^{*}(x)$ and the Lagrangian multipliers ${\lambda}$ and ${\nu}$ are uniquely determined by $x$. 

This section keeps all definitions in Definition \ref{def1}, and simplify $J^{+}(x,y,{\lambda})$ as $J^{+}(x,y)$, and $J^{0}(x,y,\lambda)$ as $J^{0}(x,y)$, where the KKT conditions hold at $y$ for $P(x)$.
We can do this because Lagrangian multipliers $\lambda$ and $\mu$ are unique and determined by $x$ and $y$ when 
the KKT conditions hold at $y$ for $P(x)$ and the LICQ holds (Assumption \ref{a3}) \cite{kyparisis1985uniqueness}.

\label{proof_app}
\subsection{Proof of Theorems \ref{th0} and \ref{th1}}

We first list Definition \ref{def4}, Theorems \ref{thm1} and \ref{thm1+}, which are shown in \cite{lemke1985introduction,Giorgi2018ATO,kojima1980strongly}, then introduce Lemmas \ref{lemma0}, \ref{thm4} and \ref{thm3}. Finally, we prove Theorem \ref{th2}, which is the full version of the combination of Theorems \ref{th0} and \ref{th1}.

\begin{definition}
\label{def4}
Suppose that the KKT conditions hold at $\hat{y}$ for $P(x)$ with the Lagrangian multipliers $\hat{\lambda}$ and $\hat{\nu}$, the Strong Second Order Sufficient Conditions (SSOSC) hold at $\hat{y}$ if
$$
z^{\top} \nabla_{y}^{2} \mathcal{L}\left(\hat{y}, \hat{\lambda}, \hat{\nu}, x\right) z>0
$$
for all $z \neq 0, z \in Z\left(\hat{y},x\right)$, where $\mathcal{L}$ is the Lagrangian associated with $P(x)$, and $Z\left(\hat{y},x\right)$ is defined by 
$$
Z\left(\hat{y},x\right) \triangleq \{  z \in \mathbb{R}^{d_y}: 
\nabla_y p_{j}\left(x,\hat{y}\right) z=0, \ \hat{\lambda}_j>0;
\nabla_y q_{i}\left(x,\hat{y}\right) z=0, \ 1 \leq i \leq n \}.
$$
\end{definition}

\begin{theorem}[\cite{lemke1985introduction,Giorgi2018ATO}]
\label{thm1} 
Consider the problem $P\left(x^{0}\right)$. Suppose that Assumption \ref{a1} holds. Suppose that $y^{0} \in K\left(x^{0}\right)$ and the KKT conditions hold at $y^{0}$ with the Lagrangian multipliers $\lambda^{0}, {\nu}^{0}$. Moreover, suppose that the LICQ holds at $y^{0}$, the SCSC holds at $y^{0}$ w.r.t. $\lambda^{0}$, and the SSOSC hold at $y^{0}$, with $\left(\lambda^{0}, {\nu}^{0}\right)$. Then, 

\begin{itemize}
\item[(\romannumeral1)] $y^{0}$ is a locally unique local minimum of $P\left(x^{0}\right)$, i.e., there exists $\delta > 0$ such that, for all $y \in \mathcal{B}(y^0,\delta)$, $y^{0}$ is the unique local minimum of $P\left(x^{0}\right)$. The associated Lagrangian multipliers $\lambda^{0}$ and ${\nu}^{0}$ are unique.
\item[(\romannumeral2)] There exists $\epsilon > 0$ such that, there exists a unique continuously differentiable vector function
$$
z(x) \triangleq [y(x)^{\top}, \lambda(x)^{\top}, \nu(x)^{\top}]^{\top},
$$
which is defined on $\mathcal{B}(x^0,\epsilon)$, and
$y(x)$ is a locally unique local minimum of $P(x)$. The KKT conditions hold at $y(x)$ with unique associated Lagrangian multipliers $\lambda(x)$ and $\nu(x)$.
\item[(\romannumeral3)] The LICQ and the SCSC hold at $y(x)$ for $P(x)$ for all $x \in \mathcal{B}(x^0,\epsilon)$.
\item[(\romannumeral4)] The gradient of $z(x)$ is given as 
$$
\left[\begin{array}{c}
\nabla_{x} y(x^0) \\
\nabla_{x} \lambda(x^0) \\
\nabla_{x} \nu(x^0)
\end{array}\right]=-M(x^0,y^0,\lambda^{0},{\nu}^{0})^{-1} N(x^0,y^0,\lambda^{0},{\nu}^{0}),
$$
where 
$$
M\triangleq
\left[\begin{array}{ccccccc}
\nabla_{y}^{2} \mathcal{L} & \left(\nabla_{y} p_{1}\right)^{\top} & \cdots & \left(\nabla_{y} p_{m}\right)^{\top} & \left(\nabla_{y} q\right)^{\top} \\
\lambda_{1} \nabla_{y} p_{1} & p_{1} & \cdots & 0  & 0 \\
\vdots & \vdots & \ddots & 0 & \vdots \\
\lambda_{m} \nabla_{y} p_{m} & 0 & \cdots & p_{m} & 0 \\
\nabla_y q & 0 & 0 & 0 & 0 
\end{array}\right]
$$
with $M(x^0, y^0)$ being nonsingular and 
$$
N\triangleq
\left[\nabla_{x y}^{2} \mathcal{L}^{\top}, \lambda_1\left(\nabla_{x} p_1\right)^{\top}, \cdots, \lambda_m\left(\nabla_{x} p_m\right)^{\top}, \nabla_{x} q^{\top} \right]^{\top}.
$$
\end{itemize}
\end{theorem}

\begin{remark}
If the lower-level optimization problem $P(x^0)$ is unconstrained, the requirements in Theorem \ref{thm1} reduce to that $\nabla_{y}^{2}g\left({y}^0, x^0\right)$  is positive definite.
\end{remark}

\begin{theorem}[\cite{Giorgi2018ATO,kojima1980strongly}]
\label{thm1+} 
Suppose that all requirements except the SCSC in Theorem \ref{thm1} are satisfied at $(y^0, \lambda^0, {\nu}^0)$ for $P\left(x^{0}\right)$. 
Then, 
\begin{itemize}
\item[(\romannumeral1)] $y^{0}$ is a locally unique local minimum of $P\left(x^{0}\right)$. The associated Lagrangian multipliers $\lambda^{0}$ and ${\nu}^{0}$ are unique.
\item[(\romannumeral2)] There exists $\epsilon > 0$ such that, there exists a unique Lipschitz continuous and once directional differentiable vector function
$$
z(x) \triangleq [y(x)^{\top}, \lambda(x)^{\top}, \nu(x)^{\top}]^{\top},
$$
which is defined on $\mathcal{B}(x^0,\epsilon)$, and
$y(x)$ is the locally unique local minimum of $P(x)$ with unique associated Lagrangian multipliers $\lambda(x)$ and $\nu(x)$.

\item[(\romannumeral3)] The LICQ hold at $y(x)$ for $P(x)$ for all $x \in \mathcal{B}(x^0,\epsilon)$.
\end{itemize}
\end{theorem}

\begin{lemma}
\label{lemma0}
Suppose that all requirements except the SCSC in Theorem \ref{thm1} hold at $(y^0, \lambda^0, {\nu}^0)$ for $P\left(x^{0}\right)$ (All requirements in Theorem \ref{thm1+} are satisfied).
Let $y(x)$ be the locally unique local minimum of $P(x)$ shown in Theorem \ref{thm1+}.
Then, 
\begin{itemize}
    \item[(\romannumeral1)] There exists $\beta >0$, such that for all $x \in \mathcal{B}(x^0,\beta)$, $p_{j}\left(x, y(x)\right)< 0$ for all $j \not\in J(x^0,y^0)$. 
    \item[(\romannumeral2)] There exists $\xi >0$ and $\delta>0$, such that for all $x^{\prime} \in \mathcal{B}(x^0,\xi)$ and $y^{\prime} \in \mathcal{B}(y(x^{\prime}),\delta)$, $p_{j}\left(x^{\prime}, y^{\prime}\right)<0$ for all $j \not\in J(x^0,y^0)$.
\end{itemize}

\end{lemma}

\begin{proof}
(\romannumeral1)
For any $j \not\in J(x^0,y^0)$, we have $p_{j}\left(x^0, y^0\right) < 0$. Then, there exists $\epsilon_1>0$ such that $p_{j}\left(x^0, y^0\right) \leq -\epsilon_1$. 
Since the function $p_j$ is continuous at $(x^0, y^0)$ and $y$ is continuous at $x^0$, we have that $p_j(x,y(x))$ is continuous at $x^0$.
Then, there exists $\beta_1 > 0$ such that, for all $x \in \mathcal{B}(x^0,\beta_1)$, we have $|p_{j}\left(x, y(x)\right) -p_{j}\left(x^0, y^0\right)| \leq \frac{1}{2}\epsilon_1$, and then $p_{j}\left(x, y(x)\right)\leq -\frac{1}{2}\epsilon_1$. 
 By selecting the smallest $\beta_1$ over all $j \not\in J(x^0,y^0)$ as $\beta$, (\romannumeral1) is shown.

(\romannumeral2)
Since $p_j$ is continuous at $(x^0, y^0)$, there exists $\delta_1 > 0$ such that, for all $(x^{\prime},y^{\prime}) \in \mathcal{B}((x^0, y^0),\delta_1)$, $|p_{j}\left(x^{\prime}, y^{\prime}\right) -p_{j}\left(x^0, y^0\right)| \leq \frac{1}{2}\epsilon_1$.
Since $y(x)$ is continuous at $x^0$, there exists $\xi_1>0$, for all $x \in \mathcal{B}(x^0,\xi_1)$, $\|y(x)-y(x^0)\|<\delta_1/4$.
Let $\xi_2 = \min\{\delta_1/4,\xi_1\}$. Then, for all $x^{\prime} \in \mathcal{B}(x^0,\xi_2)$ and $y^{\prime} \in \mathcal{B}(y(x^{\prime}),\delta_1/2)$, $\|(x^{\prime},y^{\prime})-(x_0,y_0)\|<\delta_1$. Then $|p_{j}\left(x^{\prime}, y^{\prime}\right) -p_{j}\left(x^0, y^0\right)| \leq \frac{1}{2}\epsilon_1$, and $p_{j}\left(x^{\prime}, y^{\prime}\right)\leq -\frac{1}{2}\epsilon_1$.

By selecting the smallest $\xi_2$ over all $j \not\in J(x^0,y^0)$ as $\xi$ and selecting the smallest $\delta_1/2$ over all $j \not\in J(x^0,y^0)$ as $\delta$, we have that, for all $x^{\prime} \in \mathcal{B}(x^0,\xi)$, we can find $\delta>0$ such that, when ${||y^{\prime}-y(x)||} \leq \delta$, $p_{j}\left(x, y^{\prime}\right)\leq -\frac{1}{2}\epsilon_1$ for all $j \not\in J(x^0,y^0)$. 
\end{proof}

\begin{remark}
Lemma \ref{lemma0} shows that the inactive constraints at $(x^0,y^0)$ are still inactive near $(x^0,y^0)$.
\end{remark}

\begin{lemma}
\label{thm4}
Suppose that all requirements in Theorem \ref{thm1} hold at $(y^0, \lambda^0, {\nu}^0)$ for $P\left(x^{0}\right)$. Define problem $\hat{P}\left(x\right)$ as:
$$
\begin{aligned}
\underset{y}{\arg\min }& \ g(x, y): \\
\text { s.t. } \ & p_{j}\left(x, y\right) \leq 0, j \in J(x^0,y^0), \\ 
& q\left(x, y\right) = 0.
\end{aligned}
$$
Then, the following
properties hold:
\begin{itemize}
    \item[(\romannumeral1)] All requirements and all conclusions in Theorem \ref{thm1} hold at $(y^0, {\lambda}^{0}, {\nu}^0)$ for ${P}\left(x^{0}\right)$, and also hold at $(y^0, {\lambda}^{0}_{J(x^0,y^0)}, {\nu}^0)$ for $\hat{P}\left(x^{0}\right)$. 
\end{itemize}
Let ${z}(x) \triangleq [{y}(x)^{\top}, {\lambda}(x)^{\top}, {\nu}(x)^{\top}]^{\top}$ be the unique continuously differentiable vector function in a neighborhood of $x^0$, such that $y(x)$ is a locally unique local minimum of $P(x)$ with unique associated Lagrangian multipliers $\lambda(x)$ and $\nu(x)$.
Let $\hat{z}(x) \triangleq [\hat{y}(x)^{\top}, \hat{\lambda}(x)^{\top}, \hat{\nu}(x)^{\top}]^{\top}$ be the unique continuously differentiable vector function in a neighborhood of $x^0$, such that $\hat{y}(x)$ is a locally unique local minimum of $\hat{P}(x)$ with unique associated Lagrangian multipliers $\hat{\lambda}(x)$ and $\hat{\nu}(x)$. 
\begin{itemize} 
\item[(\romannumeral2)] We have 
$$
\begin{aligned}
&\nabla_{x} y(x^0)=\nabla_{x} \hat{y}(x^0), \\
&\nabla_{x} {\nu}(x^0)=\nabla_{x} \hat{\nu}(x^0), \\
&\nabla_{x} \lambda_j(x^0)=\nabla_{x} \hat{\lambda}_j(x^0) \text{ when } j \in J(x^0,y^0),\\
&\nabla_{x} \lambda_j(x^0)=0 \text{ when } j \not\in J(x^0,y^0).
\end{aligned} 
$$
\end{itemize}
\end{lemma}

\begin{proof} (\romannumeral1)
Problem $\hat{P}\left(x^{0}\right)$ holds a same objective function and same equality constraints as problem $P\left(x^{0}\right)$.
The inequality constraints of $P\left(x^{0}\right)$ are those of $P\left(x^{0}\right)$ removing the inactive constraint. 
The LICQ, the SCSC, the SSOSC and the KKT conditions hold at $(y^0, {\lambda}^{0}, {\nu}^0)$ for ${P}\left(x^{0}\right)$. Then, it is easy to justify that the LICQ, the SCSC hold at $(y^0, {\lambda}^{0}_{J(x^0,y^0)}, {\nu}^0)$ for $\hat{P}\left(x^{0}\right)$. 
By the KKT conditions at $(y^0,{\lambda}^{0}, {\nu}^0)$ for ${P}\left(x^{0}\right)$, we have $\lambda_j=0$ for $j \not\in J(x^0,y^0)$, i.e., $p_{j}\left(x^0, y^0\right) < 0$. Then, the SSOSC and the KKT conditions hold at $(y^0, {\lambda}^{0}_{J(x^0,y^0)}, {\nu}^0)$ for $\hat{P}\left(x^{0}\right)$.
By Theorem \ref{thm1}, (\romannumeral1) holds. 
Then $y^{0}$ is a locally unique local minimum of $\hat{P}\left(x^{0}\right)$, and there exists $\hat{z}(x)=[\hat{y}(x)^{\top}, \hat{\lambda}(x)^{\top}, \hat{\nu}(x)^{\top}]^{\top}$ being the unique continuously differentiable vector function in a neighborhood of $x^0$, such that $\hat{y}(x)$ is a locally unique local minimum of $\hat{P}(x)$ with unique Lagrangian multipliers $\hat{\lambda}(x)$ and $\hat{\nu}(x)$. 

(\romannumeral2) All conclusions in Theorem \ref{thm1} hold at $(y^0, {\lambda}^{0}, {\nu}^0)$ for ${P}\left(x^{0}\right)$, then there exists ${z}(x)=[{y}(x)^{\top}, {\lambda}(x)^{\top}, {\nu}(x)^{\top}]^{\top}$ be the unique continuously differentiable vector function in a neighborhood of $x^0$, such that $y(x)$ is a locally unique local minimum of $P(x)$ with unique Lagrangian multipliers $\lambda(x)$ and $\nu(x)$. Next, we will show that, in a neighborhood of $x^0$,
$$
\begin{aligned}
&y(x)=\hat{y}(x), \\
&{\nu}(x)=\hat{\nu}(x), \\
&\lambda_j(x)=\hat{\lambda}_j(x) \text{ when } j \in J(x^0,y^0),\\
&\lambda_j(x)=0 \text{ when } j \not\in J(x^0,y^0).
\end{aligned} 
$$

Since $y(x)$ is a locally unique local minimum of $P(x)$, then there exists $\beta_3>0$, for any ${||x-x^0||}\leq \beta_3$, $y(x)$ is a local minimum of $P\left(x\right)$, i.e., there exists $\delta_3 > 0$ such that $g(x, y(x)) \leq g(x, y^{\prime})$, when ${||y^{\prime}-y(x)||} \leq \delta_3$, $p\left(x, y^{\prime}\right) \leq 0$, and $q\left(x, y^{\prime}\right) = 0$. 
Let $\beta_2$ be the $\xi$ and $\delta_3$ be the $\delta$ shown in Lemma \ref{lemma0}. 
Let $\beta_4=\min{\{\beta_2,\beta_3\}}$ and $\delta_4=\min{\{\delta_2,\delta_3\}}$. Then, for all ${||x-x^0||}\leq \beta_4$, we have $\delta_4$, such that the following two statements are satisfied:
\begin{itemize}
    \item[(a)] $g(x, y(x)) \leq g(x, y^{\prime})$ when ${||y^{\prime}-y(x)||} \leq \delta_4$, $p\left(x, y^{\prime}\right) \leq 0$, and $q\left(x, y^{\prime}\right) = 0$.
    \item[(b)] $p_{j}\left(x, y^{\prime}\right)\leq -\frac{1}{2}\epsilon_1 <0$ for all $j \not\in J(x^0,y^0)$ when ${||y^{\prime}-y(x)||} \leq \delta_4$.
\end{itemize}
The statement (b) show that set $\{y^{\prime}: {||y^{\prime}-y(x)||} \leq \delta_4\} \subset $ $\{y^{\prime}: p_{j}\left(x, y^{\prime}\right)\leq -\frac{1}{2}\epsilon_1 <0 \text{ for all } j \not\in J(x^0,y^0)\}$.
Then, $g(x, y(x)) \leq g(x, y^{\prime})$ when ${||y^{\prime}-y(x)||} \leq \delta_4$, $q\left(x, y^{\prime}\right) = 0$, and $p\left(x, y^{\prime}\right) \leq 0$, $j \in J(x^0,y^0)$. This means $y(x)$ is a local minimum of $\hat{P}\left(x\right)$ for all ${||x-x^0||}\leq \beta_4$. 

For ${||x-x^0||}\leq \beta_4$, $y(x)$ is a locally unique local minimum of ${P}\left(x\right)$. Assume that $y(x)$ is not a locally unique local minimum of $\hat{P}\left(x\right)$, i.e., for any $\phi>0$, there exists ${||y^{\prime}(\phi)-y(x)||}\leq \phi$ such that $y^{\prime}$ is a local minimum of $\hat{P}\left(x\right)$. We can set $0<\phi<\delta_4$, then $p_j(x,y^{\prime}(\phi))<0$ for all $j \not\in J(x^0,y^0)$. Then for any $\phi$, $y^{\prime}(\phi)$ is a local minimum of ${P}\left(x\right)$, which contradicts that $y(x)$ is a locally unique local minimum of ${P}\left(x\right)$. Thus, $y(x)$ is a locally unique local minimum of $\hat{P}\left(x\right)$. 
Let $z_1(x)\triangleq [{y}(x)^{\top}, {\lambda}^{\prime}(x)^{\top}, {\nu}(x)^{\top}]^{\top}$ defined on $\{x: {||x-x^0||}\leq \beta_4\}$, where ${\lambda}^{\prime}(x)^{\top}$ is a vector function which contains all $\lambda_j(x)$ when $j \in J(x^0,y^0)$. Then $z_1(x)$ is a continuously differentiable vector function, and it is easy to justify that the KKT conditions hold at $y(x)$ with Lagrangian multipliers $\lambda^{\prime}(x)$ and $\nu(x)$ for $\hat{P}(x)$.

Since $\hat{z}(x)=[\hat{y}(x)^{\top}, \hat{\lambda}(x)^{\top}, \hat{\nu}(x)^{\top}]^{\top}$ is the unique continuously differentiable vector function in a neighborhood of $x^0$, such that $\hat{y}(x)$ is a locally unique local minimum of $\hat{P}(x)$. Then for ${||x-x^0||}\leq \beta_4$, we have $\hat{z}(x)=z_1(x)$ and 
$$
\begin{aligned}
&y(x)=\hat{y}(x), \\
&{\nu}(x)=\hat{\nu}(x), \\
&\lambda_j(x)=\hat{\lambda}_j(x) \text{ when } j \in J(x^0,y^0).
\end{aligned} 
$$
Since $p_{j}\left(x, y(x)\right) <0$ when $j \not\in J(x^0,y^0)$, then $\lambda_j(x)=0$ by the KKT conditions.
Then,
$$
\begin{aligned}
&\nabla_{x} y(x^0)=\nabla_{x} \hat{y}(x^0), \\
&\nabla_{x} {\nu}(x^0)=\nabla_{x} \hat{\nu}(x^0), \\
&\nabla_{x} \lambda_j(x^0)=\nabla_{x} \hat{\lambda}_j(x^0) \text{ when } j \in J(x^0,y^0),\\
&\nabla_{x} \lambda_j(x^0)=0 \text{ when } j \not\in J(x^0,y^0).
\end{aligned} 
$$
\end{proof}

Note that the SCSC holds at $\hat{y}$ w.r.t. $\hat{\lambda}$ for $P(x)$ is equivalent to $J^{0}(x,\hat{y})=\emptyset$, then $J\left(x,\hat{y}\right) = J^{+}\left(x,\hat{y}\right)$.
We define the matrix functions 
$$
M_{+}(x^0,y^0) \triangleq 
\left[\begin{array}{ccccccc}
\nabla_{y}^{2} \mathcal{L}  & \nabla_{y} p_{J^{+}(x^0,y^0)}^{\top} & \nabla_{y} q^{\top} \\
\nabla_{y} p_{J^{+}(x^0,y^0)}  & 0 & 0 \\
\nabla_y q & 0 & 0 
\end{array}\right](x^0,y^0,\lambda^{0},{\nu}^{0}),
$$
and 
$$
N_{+}(x^0,y^0)\triangleq [\nabla_{x y}^{2} \mathcal{L}^{\top}, \nabla_{x} p_{J^{+}(x^0,y^0)}^{\top}, \nabla_{x} q^{\top}]^{\top}(x^0,y^0,\lambda^{0},{\nu}^{0}).
$$
Compute the gradient of $\hat{z}(x)$ as shown in Theorem \ref{thm1}. 
Since $\lambda_j>0$ for all $j \in J^{+}(x^0,y^0)$, we can cancel all $\lambda_j$ in $M$ and $N$.
Then we can get 
$$
\left[\begin{array}{c}
\nabla_{x} \hat{y}(x^0) \\
\nabla_{x} \hat{\lambda}(x^0) \\
\nabla_{x} \hat{\nu}(x^0) \\
\end{array}\right]
= -{M}_{+}^{-1}(x^0, y^0) {N}_{+}(x^0, y^0).
$$
By Lemma \ref{thm4}, the gradient of ${z}(x)$ is computed as:
\begin{equation}
\label{compute_g0}
\begin{array}{c}
\left[\begin{array}{c}
\nabla_{x} y(x^0) \\
\nabla_{x} \lambda_{J(x^0,y^0)}(x^0) \\
\nabla_{x} \nu(x^0)
\end{array}\right]
= -M_{+}(x^0, y^0)^{-1} N_{+}(x^0, y^0), \\ \nabla_{x} \lambda_{{J(x^0,y^0)}^C}(x^0)=0. \end{array}
\end{equation}

\begin{lemma}
\label{thm3}
Suppose that all requirements except the SCSC in Theorem \ref{thm1} are satisfied at $(y^0, \lambda^0, {\nu}^0)$ for $P\left(x^{0}\right)$. 
Then, 
\begin{itemize}
\item[(\romannumeral1)] $y^{0}$ is a locally unique local minimum of $P\left(x^{0}\right)$.
\item[(\romannumeral2)] There exists $\epsilon > 0$ such that, there exists a unique Lipschitz continuous vector function
$$
z(x) \triangleq [y(x)^{\top}, \lambda(x)^{\top}, \nu(x)^{\top}]^{\top},
$$
which is defined on $\mathcal{B}(x^0,\epsilon)$, and
$y(x)$ is the locally unique local minimum of $P(x)$ with unique associated Lagrangian multipliers $\lambda(x)$ and $\nu(x)$.
\end{itemize}
For a direction $d \in \mathbb{R}^{d_x}$, 
define $J^{0}_{+}(x^0,y^0, d)$ as set which contains all $j \in J^{0}(x^0,y^0)$ such that, there exists ${\epsilon}_0>0$, for any $0<{\epsilon}<{\epsilon}_0$, 
$p_j(x^0+\epsilon d, y(x^0+\epsilon d))=0$ and $\lambda_j(x^0+\epsilon d) > 0$.
Denote $J^{0}_{-}(x^0,y^0, d) \triangleq J^{0}(x^0,y^0) \setminus 
J^{0}_{+}(x^0,y^0, d)$, 
\begin{equation}
\label{md16}
M_{D}(x^0, y^0, d) \triangleq
\left[\begin{array}{ccccccc}
\nabla_{y}^{2} \mathcal{L}  & \nabla_{y} p_{J^{+}(x^0,y^0)}^{\top} & \nabla_{y} q^{\top} & \nabla_{y} p^{\top}_{J^{0}_{+}(x^0,y^0, d)}\\
\nabla_{y} p_{J^{+}(x^0,y^0)}  & 0 & 0 & 0\\
\nabla_y q & 0 & 0 & 0 \\
\nabla_{y} p_{J^{0}_{+}(x^0,y^0, d)} & 0 & 0 & 0
\end{array}\right](x^0,y^0,\lambda^{0},{\nu}^{0})
\end{equation}
and 
\begin{equation}
\label{nd17}
N_{D}(x^0, y^0,d) \triangleq
\left[\nabla_{x y}^{2} \mathcal{L}^{\top}, \nabla_{x} p_{J^{+}(x^0,y^0)}^{\top}, \nabla_{x} q^{\top}, \nabla_{x} p_{J^{0}_{+}(x^0,y^0, d)}^{\top} \right]^{\top}(x^0,y^0,\lambda^{0},{\nu}^{0}).
\end{equation}

\begin{itemize}
\item[(\romannumeral3)] 
The directional derivative of $z(x)$ at $x^0$ on any direction $d \in \mathbb{R}^{d_x}$ with $\|d\|=1$ exists and given by

\begin{equation}
\label{compute_g1}
\begin{array}{c}
\left[\begin{array}{c}
\nabla_{d} {y}(x^0) \\
\nabla_{d} \lambda_{J^{+}(x^0,y^0)}(x^0) \\
\nabla_{d} {\nu}(x^0) \\
\nabla_{d} {\lambda}_{J_+^0(x^0,y^0,d)}(x^0) \\
\end{array}\right]
= -M_{D}^{-1}(x^0, y^0, d) N_{D}(x^0, y^0,d) d, \\
\nabla_{d} {\lambda}_{J_-^0(x^0,y^0,d)}(x^0) =0, \\
\nabla_{d} {\lambda}_{J(x^0,y^0)^C}(x^0) =0, 
\end{array}
\end{equation}

where $M_{D}(x^0,y^0,d)$ is nonsingular. 
\end{itemize}

\end{lemma}

\begin{proof}
When the SCSC holds at $y^{0}$ w.r.t. $\lambda^{0}$, then we have $J(x^0,y^0)= J^{+}(x^0,y^0) $. This theorem is equivalent to Lemma \ref{thm4}. 
When the SCSC is not satisfied, $J^{0}(x^0,y^0)\neq \emptyset$.

(\romannumeral1)
By part (\romannumeral1) of Theorem \ref{thm1+}, the LICQ, the SSOSC and the KKT conditions hold at $(y^0, {\lambda}^{0}, {\nu}^0)$ for ${P}\left(x^{0}\right)$, (\romannumeral1) holds.

(\romannumeral2) By part (\romannumeral2) of Theorem \ref{thm1+}, the LICQ, the SSOSC and the KKT conditions hold at $(y^0, {\lambda}^{0}, {\nu}^0)$ for ${P}\left(x^{0}\right)$, (\romannumeral2) holds and the directional derivative of $z(x)$ at $x^0$ on any direction exists.

(\romannumeral3) 
By (\romannumeral2), the vector function $z(x) \triangleq [y(x)^{\top}, \lambda(x)^{\top}, \nu(x)^{\top}]^{\top}$ defined on $\mathcal{B}(x^0,\epsilon)$ is the local solution of $P(x)$.
By (\romannumeral3) of Theorem \ref{thm1+} the LICQ hold at $y(x)$ for $P(x)$ for all $x \in \mathcal{B}(x^0,\epsilon)$.
Then, according to Theorem 3 in \cite{Giorgi2018ATO}, $z(x) \triangleq [y(x)^{\top}, \lambda(x)^{\top}, \nu(x)^{\top}]^{\top}$ satisfies the KKT conditions for problem $P(x)$ for any $x \in \mathcal{B}(x^0,\epsilon)$, and $\lambda(x), \nu(x)$ are unique Lagrangian multipliers.
The directional derivative of $z(x)$ at $x^0$ on any direction exists.

(a) Consier $j \not\in J(x^0,y^0)$. From part (\romannumeral2) of Lemma \ref{lemma0}, we have that, in a small neighborhood of $(x^0,y^0)$,  $p_{j}\left(x, y\right)< 0$ for all $j \not\in J(x^0,y^0)$, i.e., the inactive constraints at $(x^0,y^0)$ are still inactive in a neighborhood of $(x^0,y^0)$.
Then, for $j \not\in J(x^0,y^0)$, there exists $\beta_0$ such that
$p_j(x^0+\beta d, y(x^0+\beta d))<0$ and $\lambda_j(x^0+\beta d)=0$ for all $\beta<\beta_0$.
Then, $\nabla_{d} {\lambda}_{J(x^0,y^0)^C}(x^0) =0$.

(b) Consider $j \in J^{0}_{-}\left(x^0,y^0, d\right)$, i.e., there does not exist ${\beta}_0>0$, for any ${\beta}<{\beta}_0$, 
$p_j(x^0+\beta d, y(x^0+\beta d))=0$ and $\lambda_j(x^0+\beta d)>0$. There are two possible cases for the direction $d$. The first case is that, there exists ${\beta}_0>0$, for any ${\beta}<{\beta}_0$, 
$p_j(x^0+\beta d, y(x^0+\beta d)) \leq 0$ and $\lambda_j(x^0+\beta d)=0$. The second case is that, for any ${\beta}_0>0$, we can always find ${\beta},\beta_1<{\beta}_0$ such that, $p_j(x^0+\beta d, y(x^0+\beta d))=0$, $\lambda_j(x^0+\beta d) > 0$, and $p_j(x^0+\beta d, y(x^0+\beta d)) \leq 0$, $\lambda_j(x^0+\beta d)=0$.
For the first case, we have $\nabla_{d} {\lambda}_{J_-^0(x^0,y^0,d)}(x^0) =0$.
For the second case, since the directional derivative of $y(x)$ and $\lambda(x)$ at $x^0$ on the direction exists, we have $\nabla_\beta p_j(x^0+\beta d, y(x^0+\beta d))=0$ and $\nabla_{d} {\lambda}_{J_-^0(x^0,y^0,d)}(x^0) =0$.
Thus, for $j \in J^{0}_{-}\left(x^0,y^0, d\right)$, $\nabla_{d} {\lambda}_{J_-^0(x^0,y^0,d)}(x^0) =0$.

(c) Consider $j \in J^{+}(x^0,y^0)$, we have $\lambda^0_j>0$. When $\beta$ is sufficiently small, we have $p_j(x^0+\beta d, y(x^0+\beta d))=0$ and $\lambda_j(x^0+\beta d)>0$.

(d) Consider $j \in J^{0}_{+}\left(x^0,y^0, d\right)$, $p_j(x^0+\beta d, y(x^0+\beta d))=0$ and $\lambda_j(x^0+\beta d) > 0$.

(e) Consider the KKT conditions at $z(x)$, we have 
$$
\left\{\begin{array}{l}
\nabla_{y} \mathcal{L}(y(x), \lambda(x), \nu(x), x)=0 \\
\lambda_{j}(x) p_{j}(y(x), x)=0, \ j=1, \ldots, m \\
q_j(y(x), x)=0, \ j=1, \ldots, n
\end{array}\right.
$$
for any $x \in \mathcal{B}(x^0,\epsilon)$.
Then, for any sufficiently small $\beta<\epsilon$, we have
$$
\left\{\begin{array}{l}
\nabla_{y} \mathcal{L}(y(x^0), \lambda(x^0), \nu(x^0), x^0)=0 \\
\lambda_{j}(x^0) p_{j}(y(x^0), x^0)=0, \ j=1, \ldots, m \\
q_j(y(x^0), x^0)=0, \ j=1, \ldots, n
\end{array}\right.
\text{ and  }
\left\{\begin{array}{l}
\nabla_{y} \mathcal{L}(y(x^0+\beta d), \lambda(x^0+\beta d), \nu(x^0+\beta d), x^0+\beta d)=0 \\
\lambda_{j}(x^0+\beta d) p_{j}(y(x^0+\beta d), x^0+\beta d)=0, \ j=1, \ldots, m \\
q_j(y(x^0+\beta d), x^0+\beta d)=0, \ j=1, \ldots, n.
\end{array}\right.
$$
Then, we have
$$
\left.\frac{\partial \nabla_{y} \mathcal{L}\left(y\left(x^{0}+\beta d\right), \lambda\left(x^{0}+\beta d\right), \nu\left(x^{0}+\beta d\right), x^{0}+\beta d\right)}{\partial \beta}\right|_{\beta=0}=0,
$$
$$
\left.\frac{\partial \lambda_{i}(x^0+\beta d) p_{j}(y(x^0+\beta d), x^0+\beta d)}{\partial \beta}\right|_{\beta=0}=0, \ j=1, \ldots, m,
$$
$$
\left.\frac{\partial q_j(y(x^0+\beta d), x^0+\beta d)}{\partial \beta}\right|_{\beta=0}=0, \ j=1, \ldots, n.
$$
Then, we get 
$$
\left[\begin{array}{ccccc}
\nabla_{y}^{2} \mathcal{L} & \left(\nabla_{y} p_{1}\right)^{\top} & \cdots & \left(\nabla_{y} p_{m}\right)^{\top} & \left(\nabla_{y} q\right)^{\top} \\
\lambda_{1} \nabla_{y} p_{1} & p_{1} & \cdots & 0 & 0 \\
\vdots & \vdots & \ddots & 0 & \vdots \\
\lambda_{m} \nabla_{y} p_{m} & 0 & \cdots & p_{m} & 0 \\
\nabla_{y} q & 0 & 0 & 0 & 0
\end{array}\right] 
\left[\begin{array}{c}
\nabla_{d} y \\
\nabla_{d} \lambda_1 \\
\vdots \\
\nabla_{d} \lambda_m \\
\nabla_{d} \nu
\end{array}\right] 
+ 
\left[\begin{array}{c}
\nabla_{x y}^{2} \mathcal{L} \\
\lambda_{1}\nabla_{x} p_{1} \\
\vdots \\
\lambda_{m}\nabla_{x} p_{m} \\
\nabla_{x} q
\end{array}\right] d =0.
$$
From (a)(b), we have $\nabla_{d} {\lambda}_{J(x^0,y^0)^C}(x^0) =0$ and $\nabla_{d} {\lambda}_{J_-^0(x^0,y^0,d)}(x^0) =0$. Then, the equation is reduced to
$$
\begin{aligned}
&\left[\begin{array}{ccccc}
\nabla_{y}^{2} \mathcal{L} & \nabla_{y} p_{J^{+}(x^0,y^0)}^{\top} & \nabla_{y} p_{J^{0}_{+}(x^0,y^0, d)}^{\top} & \nabla_{y} q^{\top} \\
\lambda_{J^{+}(x^0,y^0)} \nabla_{y} p_{J^{+}(x^0,y^0)} & p_{J^{+}(x^0,y^0)} & 0 & 0 \\
{\lambda}_{J_+^0(x^0,y^0,d)} \nabla_{y} p_{J^{0}_{+}(x^0,y^0, d)} & 0  & p_{J^{0}_{+}(x^0,y^0, d)} & 0 \\
\nabla_{y} q & 0 &  0 & 0
\end{array}\right] 
\left[\begin{array}{c}
\nabla_{d} y \\
\nabla_{d} \lambda_{J^{+}(x^0,y^0)} \\
\nabla_{d} {\lambda}_{J_+^0(x^0,y^0,d)} \\
\nabla_{d} \nu
\end{array}\right]  \\
&+ 
\left[\begin{array}{c}
\nabla_{x y}^{2} \mathcal{L} \\
\lambda_{J^{+}(x^0,y^0)}\nabla_{x} p_{J^{+}(x^0,y^0)} \\
{\lambda}_{J_+^0(x^0,y^0,d)}\nabla_{x} p_{J^{0}_{+}(x^0,y^0, d)} \\
\nabla_{x} q
\end{array}\right] d =0.
\end{aligned}
$$
From (c)(d), $p_{J^{+}(x^0,y^0)}=0$, $p_{J^{0}_{+}(x^0,y^0, d)}=0$, $\lambda_{J^{+}(x^0,y^0)}>0$, and ${\lambda}_{J_+^0(x^0,y^0,d)}>0$. Then, $\lambda_{J^{+}(x^0,y^0)}$ and ${\lambda}_{J_+^0(x^0,y^0,d)}$ are cancelled. We have
$$
\begin{array}{c}
\left[\begin{array}{c}
\nabla_{d} {y}(x^0) \\
\nabla_{d} \lambda_{J^{+}(x^0,y^0)}(x^0) \\
\nabla_{d} {\nu}(x^0) \\
\nabla_{d} {\lambda}_{J_+^0(x^0,y^0,d)}(x^0) \\
\end{array}\right]
= -M_{D}^{-1}(x^0, y^0, d) N_{D}(x^0, y^0,d) d
\end{array}
$$
where $M_{D}(x^0,y^0,d)$ and $N_{D}(x^0, y^0,d)$ is defined in \eqref{md16} and \eqref{nd17}. Since the LICQ holds, $M_{D}(x^0,y^0,d)$ is nonsingular.

\end{proof}

Consider Assumptions \ref{a1}, \ref{a2}, \ref{a3} are satisfied, $y^*(x)$ is the optimal solution of $P(x)$.
Define the matrix functions 
$$
M_{+}(x) \triangleq 
\left[\begin{array}{ccccccc}
\nabla_{y}^{2} \mathcal{L}  & \nabla_{y} p_{J^{+}(x,y^*(x))}^{\top} & \nabla_{y} q^{\top} \\
\nabla_{y} p_{J^{+}(x,y^*(x))}  & 0 & 0 \\
\nabla_y q & 0 & 0 
\end{array}\right](x,y^*(x),\lambda(x),\nu(x))
$$
and 
$$
N_{+}(x)\triangleq = [\nabla_{x y}^{2} \mathcal{L}^{\top}, \nabla_{x} p_{J^{+}(x,y^*(x))}^{\top}, \nabla_{x} q^{\top}]^{\top}(x,y^*(x)).
$$
For a direction $d \in \mathbb{R}^{d_x}$, define 
$J^{0}_{+}(x,y^*(x), d)$ as set which contains all $j \in J^{0}(x,y^*(x))$ such that, there exists ${\epsilon}_0>0$, for any $0<{\epsilon}<{\epsilon}_0$, 
$p_j(x+\epsilon d, y^*(x+\epsilon d))=0$ and $\lambda_j(x+\epsilon d) > 0$.
Denote $J^{0}_{-}(x,y^*(x), d) \triangleq J^{0}(x,y^*(x)) \setminus 
J^{0}_{+}(x,y^*(x), d)$. Define 
$$
M_{D}(x, d) \triangleq
\left[ 
\begin{array}{ccc}
    M_{+} & \begin{array}{ccc}
    \nabla_{y} p^{\top}_{J^{0}_{+}(x,y^*(x), d)} \\ 0 \\ 0 \\
    \end{array} \\
    \begin{array}{ccc}
    \nabla_{y} p_{J^{0}_{+}(x,y^*(x), d)} & 0 & 0 \\
    \end{array} & 0  \\
\end{array}\right](x,y^*(x))
$$
and 
$$
N_{D}(x,d) \triangleq \left[N_{+}^{\top}(x), \nabla_{x} p_{J^{0}_{+}(x,y^*(x), d)}^{\top}(x,y^*(x)) \right]^{\top}.
$$

\begin{theorem}[Full version of the combination of Theorems \ref{th0} and \ref{th1}]
\label{th2}
Suppose Assumptions \ref{a1}, \ref{a2}, \ref{a3} hold. Then,

\begin{itemize}
\item[(\romannumeral1)] The global minimum $y^{*}(x)$ of $P\left(x\right)$ exists and is unique. The KKT conditions hold at $y^{*}(x)$ with unique Lagrangian multipliers $\lambda(x)$ and $\nu(x)$.
\item[(\romannumeral2)] 
The vector function $z(x) \triangleq [y^{*}(x), \lambda(x), \nu(x)]$ is continuous and locally Lipschitz. 
\end{itemize}

\begin{itemize}
\item[(\romannumeral3)] 
The directional derivative of $z$ at $x$ on any direction $d \in \mathbb{R}^{d_x}$ with $\|d\|=1$ exists and given by
\begin{equation}
\label{M0}
\begin{array}{c}
\left[\begin{array}{c}
\nabla_{d} {y}^{*}(x) \\
\nabla_{d} \lambda_{J^{+}(x,y^{*}(x))}(x) \\
\nabla_{d} {\nu}(x) \\
\nabla_{d} {\lambda}_{J_+^0(x,y^{*}(x),d)}(x)
\end{array}\right]
=  -M_{D}^{-1}(x,d) N_{D}(x,d) d, \\
\nabla_{d} {\lambda}_{J_-^0(x,y^{*}(x),d)}(x) =0, \\
\nabla_{d} {\lambda}_{J(x,y^{*}(x))^C}(x)=0,
\end{array}
\end{equation}

where $M_{D}(x,d)$ is nonsingular.

\end{itemize}

\begin{itemize}
\item[(\romannumeral4)] If the SCSC holds at $y^{*}(x)$ w.r.t. $\lambda(x)$, ${z}$ is continuously differentiable at $x$ and the gradient is computed by
$$
\left[\begin{array}{c}
\nabla_{x} y^{*}(x) \\
\nabla_{x} \lambda_{J(x,y^{*}(x))}(x) \\
\nabla_{x} \nu(x)
\end{array}\right]
= -M_{+}^{-1}(x) N_{+}(x),
$$
$$\nabla_{x} {\lambda}_{J(x,y^{*}(x))^C}(x) =0,$$
where ${M}_{+}(x)$ is nonsingular.

\end{itemize}
\end{theorem}

\begin{proof}
(\romannumeral1) 
Firstly, we show $y^*(x)$ exists and is unique. 
Function $g(x,y)$ is $\mu$-strongly-convex w.r.t. $y$, and the feasible set $K\left(x\right)$ is convex and closed for all ${x} \in \mathbb{R}^{d_x}$. 
Case one: $K(x)$ is bounded. Then, $K(x)$ is a compact set and $g(x,y)$ is continuous, which imply that the minimum $y^*(x)$ exists.
Case two: $K(x)$ is not bounded. 
Function $g(x,y)$ is $\mu$-strongly-convex w.r.t. $y$, then
$$g(x,y) \geq g(x,y^0) + \nabla_y g(x,y^0)^{\top} (y-y^0)+\frac{\mu}{2}\|y-y^0\|^2.$$
Then, $\lim_{\|y\| \rightarrow \infty } g(x,y) = +\infty$. Then, for any real number $\alpha$, the set $\{x \mid g(x,y) \leq \alpha\}$ is closed and bounded, and there exists $\alpha$ such that $\{y \mid g(x,y) \leq \alpha\}$ is not empty. Then, $\{x \mid g(x,y) \leq \alpha\}$ is compact.
Then, the compactness of $\{y \mid g(x,y) \leq \alpha\}$ and the continuity of $g(x,y)$ imply that, the optimization problem $\min g(x,y) \text{ s.t. } y \in \{t \mid g(x,t) \leq \alpha\}$ is feasible and has a minimum, which is also the minimum of problem $P(x)$. Thus, the global minimum $y^*(x)$ exists.

The solution $y^{*}(x)$ is a global minimum of $P(x)$ implies it is a local minimum.
Assume $y^{*}(x)$ is not the unique local minimum of $P(x)$ and there exists $y^{\prime} \in K(x)$ is also a local minimum, then there exists a $0<\alpha<1$ such that $g(x,y^{*}(x)) \leq g(x, \alpha y^{*}(x) +(1-\alpha) y^{\prime})$ and $g(x,y^{\prime} \leq g(x, \alpha y^{*}(x) +(1-\alpha) y^{\prime})$, then $\alpha g(x,y^{*}(x)) + (1-\alpha)g(x, y^{\prime}) \leq g(x, \alpha y^{*}(x) +(1-\alpha) y^{\prime})$, which contradicts that $g(x,y)$ is $\mu$-strongly-convex w.r.t. $y$.
Thus, $y^{*}(x)$ is the unique local minimum, and then it is the unique global minimum.

The solution $y^{*}(x)$ is a global minimum of $P(x)$ implies that $y^{*}(x)$ is a local minimum.
By Assumptions \ref{a3}, the LICQ holds at $y^{*}(x)$ for $P(x)$.
Then, the KKT conditions hold at $y^{*}(x)$ with unique Lagrangian multipliers $\lambda(x)$ and $\nu(x)$ for $P(x)$ by (i) of Theorem \ref{thm1+}.

(\romannumeral2)(\romannumeral3) Firstly show that, for all ${x} \in \mathbb{R}^{d_x}$, the SSOSC holds at ${y}^{*}(x)$ with Lagrangian multipliers $\lambda(x)$ and $\nu(x)$ for $P(x)$.
By Assumption \ref{a2}, $g(x,y)$ is $\mu$-strongly-convex w.r.t. $y$, which implies that $\nabla_{y}^{2} g(x, {y}^{*}(x))$ is positive definite;
$p_j(x,y)$ is convex implies that $\nabla_{y}^{2} p_j(x, {y}^{*}(x))$ is positive semi-definite; $q_j(x,y)$ is affine implies  $\nabla_{y}^{2} q(x, {y}^{*}(x))=0$. 
The KKT condition in (\romannumeral1) implies that $\lambda(x) \geq 0$.
Then,
$$
\nabla_{y}^{2} \mathcal{L}\left({y}^{*}(x), {\lambda}(x), \hat{\nu}(x), x\right)=\nabla_{y}^{2} g(x, {y}^{*}(x))+ \lambda(x)^{\top} \nabla_{y}^{2} p(x, {y}^{*}(x))+ \nu(x)^{\top} \nabla_{y}^{2} q(x, {y}^{*}(x))
$$ 
is positive definite. Therefore, the SSOSC holds.

The KKT conditions, the LICQ, and the SSOSC hold at ${y}^{*}(x)$ with Lagrangian multipliers $\lambda(x)$ and $\nu(x)$ for $P(x)$.
By Lemma \ref{thm3}, for all ${x}^0 \in \mathbb{R}^{d_x}$, in a neighborhood of ${x}^0$, there exists a unique Lipschitz continuous vector function
$\hat{z}(x) \triangleq [\hat{y}^{\top}(x), \hat{\lambda}^{\top}(x), \hat{\nu}^{\top}(x)]^{\top}$ and
$\hat{y}(x)$ is a locally unique local minimum of $P(x)$. Then the local minimum $y^{*}(x)$ with unique Lagrangian multipliers $\lambda(x)$ and $\nu(x)$ for $P(x)$ is unique implies that $\hat{z}(x)=z(x)$. 

For any $x^0 \in \mathbb{R}^{d_x}$, the vector function
$
z(x) \triangleq [y^{*}(x)^{\top}, \lambda(x)^{\top}, \nu(x)^{\top}]^{\top}
$
is Lipschitz in a neighborhood of $x^0$ implies that, $z(x)$ is locally Lipschitz on $\mathbb{R}^{d_x}$. The computation of gradient is given in Lemma \ref{thm3}.

(\romannumeral4) The KKT conditions, the LICQ, and the SSOSC hold at ${y}^{*}(x)$ with Lagrangian multipliers $\lambda(x)$ and $\nu(x)$ for $P(x)$. By Lemma \ref{thm4}, the SCSC holds at $y^{*}(x^0)$ w.r.t. $\lambda(x^0)$ implies that, for all ${x}^0 \in \mathbb{R}^{d_x}$, in a neighborhood of ${x}^0$, there exists a unique continuously differentiable vector function
$\hat{z}(x) \triangleq [\hat{y}^{\top}(x), \hat{\lambda}^{\top}(x), \hat{\nu}^{\top}(x)]^{\top}$ and
$\hat{y}(x)$ is a locally unique local minimum of $P(x)$. Then the local minimum $y^{*}(x)$ with unique Lagrangian multipliers $\lambda(x)$ and $\nu(x)$ for $P(x)$ is unique implies that $\hat{z}(x)=z(x)$. The computation of gradient is given in Lemma \ref{thm4}.

\end{proof}

For a constrained lower-level optimization $P(x)$, for each direction $d$, paper \cite{ralph1995directional} computes $\nabla_d  y^{*}(x)$ by solving a quadratic programming problem. In Theorem \ref{th2}, we compute $\nabla_d  y^{*}(x)$ by \eqref{M0} and 
the set $J^{0}_{+}(x,y^*(x), d)$ can be determined by sampling ${\epsilon}_0$ in a small neighborhood.

\subsection{Proofs of Propositions \ref{prop2_0} and \ref{prop3}}

The following lemma provides a way to compute Clarke subdifferential of $y^{*}$ without the SCSC.

\begin{lemma}
\label{prop1}
Suppose all assumptions in Theorem \ref{th1} hold. 
Then,
\begin{equation}
\label{c_d}
\bar{\partial} y^{*}\left(x^{0}\right)=\operatorname{conv}\left\{ {w}^S(x^0) : S \subseteq J^{0}(x^0,y^*(x^0))\right\}.
\end{equation}
Here,
${w}^S(x^0)$ is obtained by extracting the first $d_x$ rows from matrix $-M^{S}_+(x^0,y^*(x^0))^{-1} N^{S}_+(x^0,y^*(x^0))$. When $S$ is not empty,
$$M^{S}_+ (x,y) \triangleq
\left[\begin{array}{ccccccc}
\nabla_{y}^{2} \mathcal{L}  & \nabla_{y} p_{J^{+}(x^0,y^*(x^0))}^{\top} & \nabla_{y} q^{\top} & \nabla_{y} p_S^{\top} \\
\nabla_{y} p_{J^{+}(x^0,y^*(x^0))}  & 0 & 0 & 0 \\
\nabla_y q & 0 & 0 & 0 \\
\nabla_{y} p_S & 0 & 0 & 0 
\end{array}\right](x,y),$$
and $$N^{S}_+ (x,y) \triangleq \left[\nabla_{x y}^{2} \mathcal{L}^{\top}, \nabla_{x} p_{J^{+}(x^0,y^*(x^0))}^{\top}, \nabla_{x} q^{\top}, \nabla_{x} p_{S}^{\top} \right]^{\top}(x,y).$$
When $S$ is empty, $M^{S}_+(x^0,y^*(x^0)) ={M}_{+}(x^0,y^*(x^0))$ and $N^{S}_+(x^0,y^*(x^0))= {N}_{+}(x^0,y^*(x^0))$, where
$$
M_{+}(x,y) = \left[\begin{array}{ccccccc}
\nabla_{y}^{2} \mathcal{L}  & \nabla_{y} p_{J^{+}(x^0,y^*(x^0))}^{\top} & \nabla_{y} q^{\top} \\
\nabla_{y} p_{J^{+}(x^0,y^*(x^0))}  & 0 & 0 \\
\nabla_y q & 0 & 0 
\end{array}\right](x,y),
$$
and
$$
N_{+}(x,y) = [\nabla_{x y}^{2} \mathcal{L}^{\top}, \nabla_{x} p_{J^{+}(x^0,y^*(x^0))}^{\top}, \nabla_{x} q^{\top}]^{\top}(x,y).
$$

\end{lemma}
Lemma \ref{prop1} is provided in papers \cite{dempe1998implicit,malanowski1985differentiability}, and also
can be derived from the directional derivative in part (iii) of Theorem \ref{th1}. 

\begin{lemma}
\label{lemma1}
Let $f: \mathbb{R}^{d_x} \longrightarrow \mathbb{R}$ a locally Lipschitz function. Then $f$ is Lipschitz continuous on any compact set $S \in \mathbb{R}^{d_x}$.
\end{lemma}
\begin{proof}
The function $f: \mathbb{R}^{d_x} \longrightarrow \mathbb{R}$ locally Lipschitz, then for any $x \in \mathbb{R}^{d_x}$, there are a Lipschitz constant $l(x)$ and $\epsilon(x) >0$, such that $l(x)$ is the Lipschitz constant of $f$ on $\mathcal{B}(x,\epsilon(x))$. For any compact set $S \in \mathbb{R}^{d_x}$, we have $S \subset {\cup}_{x \in S} \mathcal{B}(x,\epsilon(x))$, then ${\cup}_{x \in S} \mathcal{B}(x,\epsilon(x))$ is a open cover of $S$. The set $S$ is compact implies that, there exists a finite set $F$ such that $S \subset {\cup}_{x \in F} \mathcal{B}(x,\epsilon(x))$. Then, $\operatorname{max}_{x \in F}l(x)<\infty$ is a Lipschitz constant of $f$ on ${\cup}_{x \in F} \mathcal{B}(x,\epsilon(x))$, and thus is a Lipschitz constant of $f$ on $S$.
\end{proof}

\begin{lemma}
\label{lemma2}
Suppose all assumptions in Theorem \ref{th2} hold. 
Given a open ball $B$, if for each $j$, either 
$$j\in J^+(x,y^*(x))^C = J(x,y^*(x))^C \cup J^0(x,y^*(x)) \ \text{ for all } x \in B,$$
or 
$$j\in J^+(x,y^*(x)) \ \text{ for all } x \in B.$$
Then $y^{*}(x)$ is continuously differentiable on $B$.

\end{lemma}
\begin{proof}

For any $j$, 
$\text{ either } p_j(x,y^*(x)) \leq 0,  \lambda_j(x)=0 \text{ for all } x \in B, \text{ or } p_j(x,y^*(x))=0, \lambda_j(x)>0 \text{ for all } x \in B$.
Consider the set
$$
B^{\prime} \triangleq \{x^{\prime} \in B: \text{for all } j, \text{ either } p_j(x^{\prime},y^*(x^{\prime}))<0,  \lambda_j(x^{\prime})=0  \text{ or } p_j(x^{\prime},y^*(x^{\prime}))=0, \lambda_j(x^{\prime})>0\}.
$$

(i)
By part (\romannumeral3) of Theorem \ref{th2}, for any point $x^{\prime} \in B^{\prime}$, the SCSC holds at $y^{*}(x^{\prime})$, and $y^{*}$ is continuously differentiable at $x^{\prime}$.

(ii)
Consider the points $x \in B \setminus  B^{\prime}$, there exist $j$ such that $p_j(x,y^*(x))=0$, $\lambda_j(x)=0$, i.e., $J^{0}(x,y^*(x))$ is not empty.

(ii.a) Consider $j$ such that 
$p_j(x,y^*(x))\leq0$, $\lambda_j(x)=0$ for all $ x \in B$. By the definitions in Theorem \ref{th2}, $j \in J^{0}_{-}(x,y^*(x), d)$ is not added to the computation of $-M_{D}^{-1} N_{D}$ in \eqref{compute_g1} for any direction $d$. 

(ii.b)
Consider $j$ such that 
$p_j(x,y^*(x))=0$, $\lambda_j(x) > 0$  for all $x \in B$. By the definitions in Theorem \ref{th2}, $j \in J^{0}_{+}(x,y^*(x), d)$ is added to the computation of $-M_{D}^{-1} N_{D}$ in \eqref{compute_g1} for any direction $d$.


From (ii.a) and (ii.b), for the points $x \in B \setminus  B^{\prime}$, for any $d$, the computation of $-M_{D}^{-1} N_{D}$ in \eqref{compute_g1} is totally same, i.e.,
$\lim _{h \rightarrow 0} \frac{\left\|y^*(x+h)-y^*(x)+M_{D}^{-1}(x,d) N_{D}(x,d) h\right\|}{|h|}=0$. Thus, $y^{*}(x)$ is differentiable on $B$.
Then, $y^{*}(x)$ is differentiable on $B \setminus  B^{\prime}$. 
Moreover, The derivative $M_{D}^{-1}(x,d) N_{D}(x,d)$ is continuous for any $x$ and $d$.
Thus, $y^{*}(x)$ is continuously differentiable on $B$.

\end{proof}

\subsubsection{Proof of Proposition \ref{prop2_0}}

\begin{proof}[Proof of Proposition \ref{prop2_0}]
By the computation of gradient of $y^{*}$ in part (\romannumeral3) of Theorem \ref{th2}, $\Phi(x)$ and $y^{*}$ are differentiable on $\mathcal{B}(x^0,\epsilon)$, then they are twice-differentiable on $\mathcal{B}(x^0,\epsilon)$. When $\epsilon$ is sufficiently small, $| \|\nabla \Phi(x^0)\|- \|\nabla \Phi(x^{\prime})\| | < o(\epsilon)$ for any $x^{\prime} \in \mathcal{B}(x^0,\epsilon)$. Then, $|\|\nabla \Phi(x^0)\| - d(0, \bar{\partial}_{\epsilon} \Phi(x^0))|< o(\epsilon)$.
\end{proof}

\subsubsection{Proof of Proposition \ref{prop3}}

\begin{proposition}[Full version of Proposition \ref{prop3}]
\label{prop4}
Suppose Assumptions \ref{a1}, \ref{a2}, \ref{a3} hold. Consider $x^0 \in \mathbb{R}^{d_x}$ and $\epsilon>0$, there exists $x \in \mathcal{B}(x^0,\epsilon)$ such that $y^{*}$ is not continuously differentiable at $x$. 
Then, there exists at least one $j$, such that there exist $x^{\prime}$, $x^{\prime\prime} \in \mathcal{B}(x^0,\epsilon)$ with 
\begin{equation}
\begin{aligned}
\label{eq1600000}
j\in J^+(x^{\prime},y^*(x^{\prime}))^C = J(x^{\prime}, y^*(x^{\prime}))^C \cup J^0(x^{\prime}, y^*(x^{\prime})) \text{ and } \ j\in J^+(x^{\prime\prime},y^*(x^{\prime\prime}); 
\end{aligned}
\end{equation}
Define the set $I^\epsilon(x^0)$ which contains all $j$ that satisfy \eqref{eq1600000}. 
Define the set $I^{\epsilon}_{+}(x^0)$  which contains all $j$ such that, for any $x \in \mathcal{B}(x^0,\epsilon)$, 
$$
j\in J^+(x, y^*(x)); 
$$
Define the set $I^{\epsilon}_{-}(x^0)$  which contains all $j$ such that, for any $x \in \mathcal{B}(x^0,\epsilon)$, 
$$
j\in J(x, y^*(x))^C \cup J^0(x, y^*(x)).
$$
The set $G(x^{0},\epsilon)$ is defined as 
$$
G(x^{0},\epsilon)\triangleq\{ \nabla_{x} f\left(x^0, y^{*}\left(x^0\right)\right)+{w}^S(x^0)^{\top}  
\nabla_{y} f\left(x^0, y^{*}\left(x^0\right)\right): S \subseteq I^\epsilon(x^0)\}.
$$
Here,
${w}^S(x^0)$ is obtained by extracting the first $d_x$ rows from matrix $-M^{S}_{\epsilon}(x^0,y^*(x^0))^{-1} N^{S}_{\epsilon}(x^0,y^*(x^0))$, with
$$M^{S}_{\epsilon} \triangleq
\left[\begin{array}{ccccccc}
\nabla_{y}^{2} \mathcal{L} & \nabla_y p_{I^{\epsilon}_{+}(x^0)}^{\top} & \nabla_{y} q^{\top}  & \nabla_y p^{\top}_{S}\\
\nabla_y p_{I^{\epsilon}_{+}(x^0)} & 0 & 0 &0\\
\nabla_y q & 0 & 0 &0\\
\nabla_y p_{S} & 0 & 0 &0
\end{array}\right],$$
and 
$$N^{S}_{\epsilon} \triangleq \left[\nabla_{x y}^{2}\mathcal{L}^{\top}, \nabla_{x} p_{I^{\epsilon}_{+}(x^0)}^{\top}, \nabla_{x} q^{\top}, \nabla_{x} p_{S}^{\top} \right]^{\top}.$$ 
Consider $x^0 \in \mathbb{R}^{d_x}$, and assume there exists a sufficiently small $\epsilon>0$ such that, there exists $x \in \mathcal{B}(x^0,\epsilon)$ with that $y^{*}(x)$ is not continuously differentiable at $x$. 
Then, the following holds
\begin{itemize}
    \item[(\romannumeral1)] For any $g \in G\left(x^{0},\epsilon\right)$, there exists $x^{\prime} \in \mathcal{B}(x^0,\epsilon)$ such that 
    $\| g-\nabla \Phi(x^{\prime}) \|< o(\epsilon)$. For any $x^{\prime\prime} \in \mathcal{B}(x^0,\epsilon)$ and $y^*(x)$ is differentiable at $x^{\prime\prime}$, there exists $g \in G\left(x^{0},\epsilon\right)$ such that 
    $\| g-\nabla \Phi(x^{\prime\prime}) \|< o(\epsilon)$. 
    \item[(\romannumeral2)] For any $z \in \mathbb{R}^{d_x}$, $$| d(z, \operatorname{conv} G(x^{0},\epsilon)) - d(z, \bar{\partial}_{\epsilon} \Phi(x^0))|< o(\epsilon).$$
\end{itemize}
\end{proposition}

\begin{proof}[Proof of Proposition \ref{prop4}]
If for $x^0 \in \mathbb{R}^{d_x}$ and $\epsilon>0$, $y^{*}(x)$ is not continuously differentiable on the open ball $\mathcal{B}(x^0,\epsilon)$. 
Suppose that, for all $j$, points $x^{\prime}$, $x^{\prime\prime} \in \mathcal{B}(x^0,\epsilon)$ with 
\eqref{eq1600000}
do not exist. By Lemma \ref{lemma2}, $y^{*}(x)$ is continuously differentiable on $\mathcal{B}(x^0,\epsilon)$, which leads to contradiction. Then, there exists at least one $j$, such that there exist $x^{\prime}$, $x^{\prime\prime} \in \mathcal{B}(x^0,\epsilon)$ with \eqref{eq1600000}.

(\romannumeral1) Firstly, similar to the proof in Proposition \ref{prop2_0}, if the components of 
$$M^{S}_{\epsilon} =
\left[\begin{array}{ccccccc}
\nabla_{y}^{2} \mathcal{L} & \nabla_y p_{I^{\epsilon}_{+}(x^0)}^{\top} & \nabla_{y} q^{\top}  & \nabla_y p^{\top}_{S}\\
\nabla_y p_{I^{\epsilon}_{+}(x^0)} & 0 & 0 &0\\
\nabla_y q & 0 & 0 &0\\
\nabla_y p_{S} & 0 & 0 &0
\end{array}\right],$$
and 
$$N^{S}_{\epsilon}= \left[\nabla_{x y}^{2}\mathcal{L}^{\top}, \nabla_{x} p_{I^{\epsilon}_{+}(x^0)}^{\top}, \nabla_{x} q^{\top}, \nabla_{x} p_{S}^{\top} \right]^{\top}.$$ 
are same with the components of $M_{D}(x^0,d)$ and $N_{D}(x^0,d)$ shown in \eqref{M0},
i.e., the sets which are involved in the matrix computation: 
$$I^{\epsilon}_{+}(x^{\prime}) \cup S ={J^{0}_{+}(x^{\prime\prime},y^*(x^{\prime\prime}), d)} \cup J^{+}(x^{\prime\prime},y^*(x^{\prime\prime})),$$ 
then the difference $\|M^{S}_{\epsilon}(x^{\prime},y^*(x^{\prime}))^{-1} N^{S}_{\epsilon}(x^{\prime},y^*(x^{\prime})) - M_{D}(x^{\prime\prime},d))^{-1} N_{D}(x^{\prime\prime},d)\|< o(\epsilon)$ for all $x^{\prime}$ and $x^{\prime\prime} \in \mathcal{B}(x^0,\epsilon)$, since the function
${(M^{S}_{\epsilon})}^{-1} N^{S}_{\epsilon}$ is differentiable (refer to the proof in Proposition \ref{prop2_0}). 

(a) Suppose that, for any $x^{\prime} \in \mathcal{B}(x^0,\epsilon)$ and any $d \in \mathbb{R}^{d_x}$, there exists $g^S \in G\left(x^{0},\epsilon\right)$ such that 
$$I^{\epsilon}_{+}(x^{0}) \cup S ={J^{0}_{+}(x^{\prime},y^*(x^{\prime}), d)} \cup J^{+}(x^{\prime},y^*(x^{\prime}));$$ 
Then, for any $x^{\prime} \in \mathcal{B}(x^0,\epsilon)$ such that $y^{*}$ is differentiable at $x^{\prime}$, there exists $g^S \in G\left(x^{0},\epsilon\right)$ such that, for a direction $d \in \mathbb{R}^{d_x}$, $\|M^{S}_{\epsilon}(x^{0},y^*(x^{0}))^{-1} N^{S}_{\epsilon}(x^{0},y^*(x^{0})) - M_{D}(x^{\prime},d)^{-1} N_{D}(x^{\prime},d)\|< o(\epsilon)$. Then, $\|w^S(x^0)-y^*{(x^{\prime})}\|< o(\epsilon)$ and $\| g^S-\nabla \Phi(x^{\prime}) \|< o(\epsilon)$. 

(b) Suppose that, for any $g^S \in G\left(x^{0},\epsilon\right)$, there exists $x^{\prime} \in \mathcal{B}(x^0,\epsilon)$ and $d \in \mathbb{R}^{d_x}$ such that
$$I^{\epsilon}_{+}(x^{0}) \cup S ={J^{0}_{+}(x^{\prime},y^*(x^{\prime}), d)} \cup J^{+}(x^{\prime},y^*(x^{\prime})).$$ 
Since the set ${J^{0}_{+}(x^{\prime},y^*(x^{\prime}), d)} \subseteq  J^{0}(x^{\prime},y^*(x^{\prime}))$, from Lemma \ref{prop1} and its proof shown in \cite{dempe1998implicit,malanowski1985differentiability}, $g^S \in \bar{\partial} y^{*}\left(x^{\prime}\right)$ and exists $\left\{x^{j}\right\}$ such that $g^S = \lim _{j \rightarrow \infty} \nabla y^{*}\left(x^{j}\right):\left\{x^{j}\right\} \rightarrow x^{\prime}$ where $y^{*}$ is differentiable at $y^{j}$ for all $j$.
Then, there exists $x^{\prime\prime}$ in any small neighborhood of $x^{\prime}$, such that ${J^{0}_{+}(x^{\prime},y^*(x^{\prime}), d)} \cup J^{+}(x^{\prime},y^*(x^{\prime}))= J^{+}(x^{\prime\prime},y^*(x^{\prime\prime}))$ and $y^*$ is differentiable at $x^{\prime\prime}$.
Then, 
$$I^{\epsilon}_{+}(x^{0}) \cup S = J^{+}(x^{\prime\prime},y^*(x^{\prime\prime})),$$ 
and we have $\|w^S(x^0)-y^*{(x^{\prime\prime})}\|< o(\epsilon)$ and $\| g^S-\nabla \Phi(x^{\prime\prime}) \|< o(\epsilon)$ with 
$x^{\prime\prime} \in \mathcal{B}(x^0,\epsilon)$. Then, the part (i) of the proposition is shown.

Now, we can prove (\romannumeral1) by showing that the assumptions of (a) and (b) hold. 
Define $J^{N}(x,y^*(x)) \triangleq \{j: j \not\in {J^{+}(x,y^*(x))} \cup {J^{0}(x,y^*(x))} \}$. Then, for any $x^{*} \in \mathcal{B}(x^0,\epsilon)$ and any $d \in \mathbb{R}^{d_x}$, 
$
I^{\epsilon}(x^0) \cup I_+^{\epsilon}(x^0) \cup I_-^{\epsilon}(x^0) =J^{0}({x^*,y^*(x^{*})}) \cup J^{+}({x^*,y^*(x^{*})}) \cup J^{N}(x^*,y^*(x^*))= J^{0}_{-}({x^*,y^*(x^{*}),d}) \cup J^{0}_{+}({x^*,y^*(x^{*}),d}) \cup J^{+}({x^*,y^*(x^{*})}) \cup J^{N}(x^*,y^*(x^*))
$,
since they both contain all constraints $p_j$. Note that the intersection of two in $I^{\epsilon}(x^0)$, $I_+^{\epsilon}(x^0)$ and $I_-^{\epsilon}(x^0)$ is empty, the intersection of two in $J^{0}_{-}({x^*,y^*(x^{*}),d})$, $J^{0}_{+}({x^*,y^*(x^{*}),d})$, $J^{+}({x^*,y^*(x^{*})})$ and $J^{N}(x^*,y^*(x^*))$ is empty, and $J^{0}({x^*,y^*(x^{*})}) = J^{0}_{+}({x^*,y^*(x^{*}),d}) \cup J^{0}_{-}({x^*,y^*(x^{*}),d})$.

\begin{itemize}
\item [(1)] Consider $j \in I_+^{\epsilon}(x^0)$, i.e., $p_j(x,y^*(x))=0$ and $\lambda_j(x) > 0$ for all $x \in \mathcal{B}(x^0,\epsilon)$. 
Then $j \in {J^{+}(x^{\prime},y^*(x^{\prime}))}$ for any $x^{\prime} \in \mathcal{B}(x^0,\epsilon)$.
Then,
$j \in {J^{0}_{+}(x^{\prime},y^*(x^{\prime}), d)} \cup {J^{+}(x^{\prime},y^*(x^{\prime}))}$ for any $d \in \mathbb{R}^{d_x}$ and any $x^{\prime} \in \mathcal{B}(x^0,\epsilon)$.
Thus, $$I_+^{\epsilon}(x^0) \subseteq {J^{0}_{+}(x^{\prime},y^*(x^{\prime}), d)} \cup {J^{+}(x^{\prime},y^*(x^{\prime}))}$$ for any $d \in \mathbb{R}^{d_x}$ and any $x^{\prime} \in \mathcal{B}(x^0,\epsilon)$.

\item [(2)] Consider $j \in I_-^{\epsilon}(x^0)$. If given $x \in \mathcal{B}(x^0,\epsilon)$, we have $p_j(x,y^*(x))<0$ and $\lambda_j(x) = 0$, then $j \not\in {J^{+}(x,y^*(x))} \cup {J^{0}(x,y^*(x))} $, i.e., $j \in J^{N}(x,y^*(x))$.
If given $x \in \mathcal{B}(x^0,\epsilon)$, we have $p_j(x,y^*(x))=0$, $\lambda_j(x) = 0$ and this holds for any $x \in \mathcal{B}(x^0,\epsilon)$, then $j \in {J^{0}_{-}(x,y^*(x), d)}$ for any $d \in \mathbb{R}^{d_x}$.
If given $x \in \mathcal{B}(x^0,\epsilon)$, we have $p_j(x,y^*(x))=0$, $\lambda_j(x) = 0$ but this does not hold for all set $\mathcal{B}(x^0,\epsilon)$, then $j \in {J^{0}_{-}(x,y^*(x), d)}$ for any $d \in \mathbb{R}^{d_x}$.
Then, $$I_-^{\epsilon}(x^0) \subseteq  {J^{0}_{-}(x,y^*(x), d)} \cup J^{N}(x,y^*(x)),$$
for any $x \in \mathcal{B}(x^0,\epsilon)$ and any $d \in \mathbb{R}^{d_x}$.

\item [(3)] Consider $j \in I^{\epsilon}(x^0)$.
When $\epsilon>0$ is sufficiently small such that, there exists $x^{*} \in \mathcal{B}(x^0,\epsilon)$ such that $p_j(x^{*},y^*(x^{*}))=0$ and $\lambda_j(x^{*}) = 0$ for all $j \in I^{\epsilon}(x^0)$.
Then, we have $$I^{\epsilon}(x^0) \subseteq J^{0}({x^*,y^*(x^{*})}).$$ Then, for any $S \subseteq  I^{\epsilon}(x^0)$, $S \subseteq J^{0}({x^*,y^*(x^{*})})$. 
Also, $${J^{+}(x^*,y^*(x^*))} \subseteq I_+^{\epsilon}(x^0).$$
This is true because that, assume that $j \in {J^{+}(x^*,y^*(x^*))}$ and $j \not\in I_+^{\epsilon}(x^0)$, then 
$j \in I^{\epsilon}(x^0)$, then $p_j(x^{*},y^*(x^{*}))=0$ and $\lambda_j(x^{*}) = 0$, and $j \in {J^{0}(x^*,y^*(x^*))}$ which contradicts  $j \in {J^{+}(x^*,y^*(x^*))}$.

\end{itemize}

From (1) and (3), we have that, there exists $x^{*} \in \mathcal{B}(x^0,\epsilon)$, for any $d \in \mathbb{R}^{d_x}$,
\begin{equation}
\label{eq222}
J^{+}(x^*, y^*(x^*))  \subseteq I_+^{\epsilon}(x^0)  \subseteq  {J^{+}(x^*,y^*(x^*))}\cup {J_+^{0}(x^*,y^*(x^*),d)}.
\end{equation}
From Lemma \ref{prop1} and the proof of the lemma shown in \cite{dempe1998implicit,malanowski1985differentiability}, for any set $A \subseteq  J^{0}({x^*,y^*(x^{*})})$, we can find a direction $d$ such that $A = {J^{0}_{+}(x^*,y^*(x^*), d)}$. Here, we can let $A= S \cup I_+^{\epsilon}(x^0) \setminus J^{+}(x^*,y^*(x^*)) $. From \eqref{eq222}, $A \subseteq J^{0}({x^*,y^*(x^{*})})$. 
Then, $S \cup I_+^{\epsilon}(x^0) \setminus J^{+}(x^*,y^*(x^*)) ={J^{0}_{+}(x^*,y^*(x^*), d)}$.
Then, for any $g^S \in G\left(x^{0},\epsilon\right)$, 
there exists $x^{*}\in \mathcal{B}(x^0,\epsilon)$ and $d \in \mathbb{R}^{d_x}$ such that
$$I^{\epsilon}_{+}(x^{0}) \cup S ={J^{0}_{+}(x^{*},y^*(x^{*}), d)} \cup J^{+}(x^{*},y^*(x^{*})).$$ 
The assumption of (b) is shown.

From (2), if $j \in I_-^{\epsilon}(x^0)$, for any $x \in \mathcal{B}(x^0,\epsilon)$ and any $d \in \mathbb{R}^{d_x}$, $j \in  {J^{0}_{-}(x,y^*(x), d)} \cup J^{N}(x,y^*(x))$. Then, $j \not\in {J^{0}_{+}(x,y^*(x), d)} \cup {J^{+}(x,y^*(x))} $. Then, for any set $C \subseteq I_-^{\epsilon}(x^0)$, $C \cap {(J^{0}_{+}(x,y^*(x), d)} \cup {J^{+}(x,y^*(x)))}$ is empty.
Moreover, from (1), $I_+^{\epsilon}(x^0) \subseteq {J^{0}_{+}(x,y^*(x), d)} \cup {J^{+}(x,y^*(x))}$.
Then,
for any $x \in \mathcal{B}(x^0,\epsilon)$ and any $d \in \mathbb{R}^{d_x}$, there exists
$g^S \in G\left(x^{0},\epsilon\right)$ such that 
$$I^{\epsilon}_{+}(x^{0}) \cup S ={J^{0}_{+}(x,y^*(x), d)} \cup J^{+}(x,y^*(x));$$ 
The assumption of (a) is shown. Then, the proof of the part (i) of the proposition is done.

(\romannumeral2) From the definition of the Clarke $\epsilon \text {-subdifferential}$, it is easy to see
$\bar{\partial}_{\epsilon} \Phi(x^0)=\operatorname{conv} \{\nabla \Phi(x^{\prime}): x^{\prime} \in \mathcal{B}(x^0,\epsilon), y^* $ is differentiable at  $x^{\prime}\}$. Then, for any $g_0 \in \bar{\partial}_{\epsilon} \Phi(x^0)$, there exist $g_1$, $g_2 \in \{\nabla \Phi(x^{\prime}): x^{\prime} \in \mathcal{B}(x^0,\epsilon), y^* $ is differentiable at  $x^{\prime}\}$, such that $g_0 = \theta g_1 + (1-\theta) g_2$ and $0 \leq \theta \leq 1$.
From (\romannumeral1), we can find $g_1^{\prime}$, $g_2^{\prime} \in G\left(x^{0},\epsilon\right)$, such that $\| g_1-g_1^{\prime} \|< o(\epsilon)$ and $\| g_2-g_2^{\prime} \|< o(\epsilon)$. Then, 
$$
g_0^{\prime} = \theta g_1^{\prime} + (1-\theta) g_2^{\prime} \in \operatorname{conv} G(x^{0},\epsilon),$$ and for any $z \in \mathbb{R}^{d_x}$, 
$$
\begin{aligned}
| \|z- g_0\|-\|z- g_0^{\prime}\| | &= | \|z- (\theta g_1 + (1-\theta) g_2) \|-\|z- (\theta g_1^{\prime} + (1-\theta) g_2^{\prime})\| | \\ &\leq | \|\theta (z- g_1) + (1-\theta) (z-g_2) \|-\|(\theta (z- g_1^{\prime})) + (1-\theta) (z-g_2^{\prime})\| | \\
&\leq  \|\theta (z- g_1) + (1-\theta) (z-g_2) -(\theta (z- g_1^{\prime})) - (1-\theta) (z-g_2^{\prime})\| \\ 
&\leq \|\theta (z- g_1) -(\theta (z- g_1^{\prime})) \| +\|(1-\theta) (z-g_2) - (1-\theta) (z-g_2^{\prime})\| \\ &< o(\epsilon).
\end{aligned}
$$
Let $l= d(z, \bar{\partial}_{\epsilon} \Phi(x^0))$. Then, for any $\sigma >0$, there exists $g_0 \in \bar{\partial}_{\epsilon} \Phi(x^0))$ such that
$l \leq \|z-g_0 \| < l+\sigma$. Then, there exists $g_0^{\prime} \in  \operatorname{conv} G(x^{0},\epsilon))$ with $| \|z- g_0\|-\|z- g_0^{\prime}\| | < o(\epsilon)$.
Then, $l-o(\epsilon) \leq \|z-g_0^{\prime} \| < l+\sigma+o(\epsilon)$. Since $\|z-g_0^{\prime} \| \geq d(z, \operatorname{conv} G(x^{0},\epsilon))$, we have $d(z, \operatorname{conv} G(x^{0},\epsilon)) = \inf \{\|z-a\| \mid a \in \operatorname{conv} G(x^{0},\epsilon)\}< l+\sigma+o(\epsilon)$.
Then, $d(z, \operatorname{conv} G(x^{0},\epsilon)) < d(z, \bar{\partial}_{\epsilon} \Phi(x^0))+\sigma+o(\epsilon)$ for any $\sigma>0$, i.e., $d(z, \operatorname{conv} G(x^{0},\epsilon)) \leq d(z, \bar{\partial}_{\epsilon} \Phi(x^0))+o(\epsilon)$.

Similar to the proof of $d(z, \operatorname{conv} G(x^{0},\epsilon)) \leq d(z, \bar{\partial}_{\epsilon} \Phi(x^0))+o(\epsilon)$, from (\romannumeral1), we can also get  $d(z, \bar{\partial}_{\epsilon} \Phi(x^0)) \leq d(z, \operatorname{conv} G(x^{0},\epsilon)) +o(\epsilon)$. Then, $| d(z, \operatorname{conv} G(x^{0},\epsilon)) - d(z, \bar{\partial}_{\epsilon} \Phi(x^0))|< o(\epsilon)$. Then, the proof of the part (ii) of Proposition \ref{prop4} is done.

\end{proof}

Note that Proposition \ref{prop3} is included in Proposition \ref{prop4}, then Proposition \ref{prop3} is shown.

\subsection{Proof of Theorem \ref{th_converge}}

We start part (\romannumeral1) of the proof of Theorem \ref{th_converge} from Lemma \ref{lemma3}, and show part (\romannumeral2) of the proof from Lemma \ref{lemma4}.

\begin{lemma}
\label{lemma3}
Let $\emptyset \neq C \subset \mathbb{R}^{n}$ be compact and convex and $\beta \in(0,1)$ and $0 \notin C$.
If $u, v \in C$ and $\|u\| = d(0, C)$, then $\langle v, u\rangle \geq \|u\|^{2}$.
\end{lemma}

\begin{proof}
Suppose that $\|u\| = d(0, C)$ and $\langle v, u\rangle < \|u\|^{2}$. Then, there exists $0<\theta<1$ such that $w=\theta v + (1-\theta) u$ and $\|w\|<\|u\|$. Then, $w \in C$ since $C$ is convex. Then $\|u\| \ne d(0, C)$, which has contradiction.
\end{proof}

\begin{proof}[Proof of part (\romannumeral1) of Theorem \ref{th_converge}]

Firstly, consider $y^{*}$ is differentiable at $x^k$. We have $g^{k}= \min \{ \|g\|: g \in \operatorname{conv} G(x^{k},\epsilon_k) \}$ and $\|g^k\| >0$.
We have $ \operatorname{conv} G(x^{k},\epsilon_k) $ is closed, bounded and convex, then is convex and compact.
Since $ \operatorname{conv} G(x^{k},\epsilon_k) $ is closed, $g^k \in \operatorname{conv} G(x^{k},\epsilon_k) $ and $\|g^k \|= d(0, \operatorname{conv} G(x^{k},\epsilon_k))$. 
By Lemma \ref{lemma3}, $\langle v, g_k\rangle \geq \|g^k\|^{2}$ for any $v \in \operatorname{conv} G(x^{k},\epsilon_k)$.
Since $y^{*}$ is differentiable at $x^k$, we have $\nabla \Phi(x^k) \in \operatorname{conv} G(x^{k},\epsilon_k)$.  Then, $\langle \nabla \Phi(x^k), g_k\rangle \geq \|g^k\|^{2}$. Since $y^{*}$ is differentiable at $x^k$, $\nabla \Phi$ is differentiable at $x^k$, then $\Phi$ is twice-differentiable at $x^k$.
Then, when $t$ is sufficiently small,
$$
\begin{aligned}
\Phi(x^{k}-t g^{k})=& \Phi(x^{k})-t \langle \nabla \Phi(x^{k}), g^{k} \rangle +o(t^2) \\
\leq&\Phi(x^{k}) - t\|g^{k}\|^{2} +o(t^2) \\
<& \Phi(x^{k}) -\beta t\|g^{k}\|^{2},
\end{aligned}
$$
for any $0<\beta<1$. Then, the line search has a positive solution $t_k$. 

Secondly, consider $y^{x}$ is not differentiable at $x^k$. From Lemma \ref{prop1}, Proposition \ref{prop3} and the proof of Proposition \ref{prop3} (also shown in \cite{dempe1998implicit,malanowski1985differentiability}), $\Phi(x)$ is a piecewise function by finite number of twice-differentiable functions $\{\Phi_i(x)\}_{1 \leq i \leq m}$, and for each $1 \leq i \leq m$, $\nabla\Phi_i(x^k) \in \operatorname{conv} G(x^{k},\epsilon_k)$.
Then, when $t$ is sufficiently small, for all $1 \leq i \leq m$,
$$
\begin{aligned}
\Phi_i(x^{k}-t g^{k})=& \Phi_i(x^{k})-t \langle \nabla \Phi_i(x^{k}), g^{k} \rangle +o(t^2) \\
\leq &\Phi_i(x^{k}) - t\|g^{k}\|^{2} +o(t^2) \\
<& \Phi_i(x^{k}) -\beta t\|g^{k}\|^{2}\\
=& \Phi(x^{k}) -\beta t\|g^{k}\|^{2},
\end{aligned}
$$
for any $0<\beta<1$.
Then, we have $\Phi(x^{k}-t g^{k})<\Phi(x^{k}) -\beta t\|g^{k}\|^{2}$. Then, the line search has a positive solution $t_k$. 

\end{proof}

\begin{lemma}
\label{lemma4}
If $\liminf\limits_{k \rightarrow \infty}{\max \left\{\|x^k-x\|,\|g^k\|,\epsilon_k\right\}}=0$ with $| \|g^k\|-d(0, \bar{\partial}_{\epsilon_k} \Phi(x^k))|<o(\epsilon_k)$ for sufficiently small $\epsilon_k$, then $0 \in \bar{\partial} \Phi(x)$. 
\end{lemma}

\begin{proof}
If $\liminf\limits_{k \rightarrow \infty}{\max \left\{\|x^k-x\|,\|g^k\|,\epsilon_k\right\}}=0$, then the sequence $\{x^{k}\}$ has a subsequence $\{x^{k_n}\}$ such that $x^{k_n} \rightarrow x$, $\|g^k\| \rightarrow 0$ and $\epsilon_k \rightarrow 0$. Since $| \|g^k\|-d(0, \bar{\partial}_{\epsilon} \Phi(x^k))|<o(\epsilon_k)$ and $\epsilon_k \rightarrow 0$, we have $d(0, \bar{\partial}_{\epsilon_k} \Phi(x^k)) \rightarrow 0$, i.e., there exists a sequence $\{h^k\}$ with $h^k \in \bar{\partial}_{\epsilon_k} \Phi(x^k)$ and $\|h^k\| \rightarrow 0$, which implies $0 \in \bar{\partial} \Phi(x)$.
\end{proof}

\begin{proof}[Proof of parts (\romannumeral2)-(\romannumeral4) of Theorem \ref{th_converge}]
(\romannumeral2.a) From part (\romannumeral1) of Theorem \ref{th_converge}, the line search has a non-zero solution $t_k$, we have 
$\Phi(x^{k+1})< \Phi(x^{k}) -\beta t_k\|g^{k}\|^{2}$, Then,
$$\sum_{k=1}^{\infty} \beta t_k\|g^{k}\|^{2} < \sum_{k=1}^{\infty} \Phi(x^{k})-\Phi(x^{k+1}) \leq \Phi(x^{1})-a<\infty,$$
where $a$ is a lower bound of $\Phi(x)$ on $\mathbb{R}^{d_x}$.
Since $x^{k+1}=x^{k}-t g^{k}$, we have 
\begin{equation}
\label{smallinf}
    \sum_{k=1}^{\infty} \|x^{k+1}-x^{k} \| \|g^k \|<\infty.
\end{equation}

(\romannumeral2.b) We now show $\lim_{k\rightarrow\infty} \nu_{k} = 0$ and $\lim_{k\rightarrow\infty} \epsilon_{k} = 0$. Assume this does not hold. There are $k_1$, $\hat{\nu}$ and $\hat{\epsilon}$, such that ${\nu}_k=\hat{\nu}$, ${\epsilon}_k=\hat{\epsilon}$ and $\|g^k\|>\hat{\nu}$ for all $k>k_1$.
Then, by \eqref{smallinf}, we have $\sum_{k=1}^{\infty} \|x^{k+1}-x^{k} \|<\infty$, and then $x^k$ converges to a point $\bar{x}$ and $t_k\|g^{k}\| \rightarrow 0$ which means $t_k \rightarrow 0$. 

Similar to the proof of part (i) of Theorem \ref{th_converge}, when $t$ is sufficiently small,
\begin{equation}
\label{t_solution}
\begin{aligned}
\Phi(x^{k}-t g^{k})=& \Phi(x^{k})-t \langle \nabla \Phi_i(x^{k}), g^{k} \rangle +\alpha t^2 \\
\leq&\Phi(x^{k}) - t\|g^{k}\|^{2} +\alpha t^2 \\
<& \Phi(x^{k}) -\beta t\|g^{k}\|^{2},
\end{aligned}
\end{equation}
where $\Phi(x)$ is a piecewise function by finite number of twice-differentiable functions $\{\Phi_i(x)\}_{1 \leq i \leq m}$, and $\alpha$ is the upper bound of $\{\| \nabla^2\Phi_i(x)\| \}_{1 \leq i \leq m}$ on the set $\operatorname{conv} \{x^k\}$.
Since $x^k \rightarrow \bar{x}$, there exists a bounded and closed set $A$ such that $\operatorname{conv} \{x^k\} \subset A$. Since $\nabla^2\Phi_i(x)$ is continuous on the compact set $A$, $\| \nabla^2\Phi_i(x) \|$ is bounded.
Then, the upper bound $\alpha$ exists. 

From \eqref{t_solution}, there exists $t_0$ such that, for all $t<t_0$ and 
$- t\|g^{k}\|^{2} +\alpha t^2 
< -\beta t\|g^{k}\|^{2}$, i.e., $t<\frac{(1-\beta)\|g^{k}\|^{2}}{\alpha}$, we have $t$ satisfies the inequality $\Phi(x^{k}-t g^{k})< \Phi(x^{k}) -\beta t\|g^{k}\|^{2}$.
Since $\|g^k\|>\hat{\nu}$ for all $k>k_1$, then if $t<t_0$ and $t<\frac{(1-\beta)\hat{\nu}^{2}}{\alpha}$, we have $t<\frac{(1-\beta)\|g^{k}\|^{2}}{\alpha}$ for all $k>k_1$.
Then, the line search
$t_{k} = \sup \{t \in\{\gamma, \gamma^{2}, \ldots\}: \Phi(x^{k}-t g^{k})< \Phi(x^{k}) -\beta t\|g^{k}\|^{2}\}$ has solution $t_{k} \geq \gamma^{N}$ with $\gamma^{N} \leq \min\{t_0, \frac{(1-\beta)\hat{\nu}^{2}}{\alpha}\} <\gamma^{N-1}$ for all $k>k_1$.
Then, $t_k$ does not converge to $0$, which contradicts \eqref{smallinf}.
Then, $\lim_{k\rightarrow\infty} \nu_{k} = 0$ and $\lim_{k\rightarrow\infty} \epsilon_{k} = 0$.

(\romannumeral3)
Since $\lim_{k\rightarrow\infty} \nu_{k} = 0$, then for any arbitrary small $\nu>0$ and any $k_0$, there exists $k>k_0$ such that $\|g_k\|<\nu$, then $\liminf\limits_{k \rightarrow \infty}\|g^k\| = 0$.
Since $| \|g^k\|-d(0, \bar{\partial}_{\epsilon} \Phi(x^k))|<o(\epsilon_k)$ and $\epsilon_k \rightarrow 0$, we have $\liminf\limits_{k\rightarrow\infty} d(0, \bar{\partial} \Phi(x^k))= 0$.

(\romannumeral4)
We now have $\lim_{k\rightarrow\infty} \nu_{k} = 0$, $\lim_{k\rightarrow\infty} \epsilon_{k} = 0$,
and $\liminf\limits_{k \rightarrow \infty}\|g^k\| = 0$.
Let $\bar{x}$ is a limit point of $\{x^k\}$. If $x^k$ converges to $\bar{x}$, then $\|x^k-x\| \rightarrow 0$.
So we have
$\liminf\limits_{k \rightarrow \infty}{\max \left\{\|x^k-x\|,\|g^k\|,\epsilon_k\right\}}=0$.
By Proposition \ref{prop3}, if $\Phi$ is not differentiable on the ball $\mathcal{B}(x^0,\epsilon)$, for sufficiently small $\epsilon_k$,
we have $| \|g^k\|-d(0, \bar{\partial}_{\epsilon_k} \Phi(x^k))|= | d(0, \operatorname{conv} G(x^{k},\epsilon_k)) -d(0, \bar{\partial}_{\epsilon_k} \Phi(x^k))|<o(\epsilon_k)$. 
By Proposition \ref{prop2}, if $\Phi$ is not differentiable on the ball $\mathcal{B}(x^0,\epsilon)$, $g^k=\bar{\partial} \Phi(x^k))$, then $| \|g^k\|-d(0, \bar{\partial}_{\epsilon_k} \Phi(x^k))|<o(\epsilon_k)$.
Then, by Lemma \ref{lemma4}, we have $0 \in \bar{\partial} \Phi(x)$. 

Consider $x^k$ does not converge to $\bar{x}$.
Note that $\bar{x}$ is a limit point of $\{x^k\}$. 
Since $\lim_{k\rightarrow\infty} \epsilon_{k} = 0$, we just need to show $\liminf\limits_{k \rightarrow \infty}{\max \left\{\|x^k-x\|,\|g^k\|\right\}}=0$, then $0 \in \bar{\partial} \Phi(x)$ by Lemma \ref{lemma4}. Assume that  $\liminf\limits_{k \rightarrow \infty}{\max \left\{\|x^k-x\|,\|g^k\|\right\}}>0$.
Since $\bar{x}$ is a limit point of $\{x^k\}$, for any $\hat{v}>0$ and $\hat{k}$, there exists an infinite set $K(\hat{k},\hat{v}) \triangleq \{k: k\geq \hat{k}, \|x^k-\bar{x}\|<\hat{v} \}$.
Since $\liminf\limits_{k \rightarrow \infty}{\max \left\{\|x^k-x\|,\|g^k\|\right\}}>0$,
there exists $\hat{v}>0$ and $\hat{k}$ such that $\|g^k\|>\hat{v}$ for all $k \in K(\hat{k},\hat{v})$. From \eqref{smallinf}, we have $ \sum_{k \in K(\hat{k},\hat{v})} \|x^{k+1}-x^{k} \| \|g^k \|<\infty$, then $ \sum_{k \in K(\hat{k},\hat{v})} \|x^{k+1}-x^{k}\|<\infty$. Since $x^k$ does not converge to $\bar{x}$, there exists $\epsilon>0$ such that, for each $k \in K(\hat{k},\hat{v})$ with $\|x^{k}-\bar{x}\|\leq \hat{v}/2$, there exists $k^{\prime} >k$ satisfying $\|x^{k}-x^{k^{\prime}}\|>\epsilon$ and $x^i \in K(\hat{k},\hat{v})$ for all $k \leq i< k^{\prime}$.
Then, we have $\epsilon<\|x^{k^{\prime}}-x^{k}\| \leq \sum_{i=k}^{k^{\prime}-1}\|x^{i+1}-x^{i}\|$. Since $ \sum_{i \in K(\hat{k},\hat{v})} \|x^{i+1}-x^{i} \| <\infty$, we can select a sufficiently large $k$ with $ \sum_{i=k}^{\infty} \|x^{i+1}-x^{i} \| <\epsilon$. Then, there is a contradiction. Thus, $\liminf\limits_{k \rightarrow \infty}{\max \left\{\|x^k-x\|,\|g^k\|\right\}}=0$.

Thus, whether $x^k$ converges to the limit point $\bar{x}$ or not, we have $0 \in \bar{\partial} \Phi(x)$. The part (iii) of Theorem \ref{th_converge} is shown.

\end{proof}

\subsection{Proofs of Propositions \ref{prop2}}

\begin{proof}[Proof of Proposition \ref{prop2}]

(\romannumeral1) By part (\romannumeral2) of Theorem \ref{th2}, vector function $z(x) \triangleq [y^{*}(x)^{\top}, \lambda(x)^{\top}, \nu(x)^{\top}]^{\top}$ is locally Lipschitz.
Lemma \ref{lemma1} implies that $z(x)$ is Lipschitz continuous function on any compact set $S \in \mathbb{R}^{d_x}$.
Then, $y^{*}(x)$, $\lambda_j(x)$ and $p_j(x)$ for all $j$ are Lipschitz continuous function on any compact set.
Since $f$ and $p$ are continuously differentiable, then $\Phi(x)=f\left(x, y^{*}(x)\right)$ and $p(x,y^*(x)$ are locally Lipschitz. Thus, $\Phi(x)$ and $f\left(x, y^{*}(x)\right)$ are Lipschitz continuous function on any compact set. Then part (i) is proven.

For all $j \in J(x^0,\hat{y})$, either $\lambda_j(x^0) > l_{\lambda_j}(x^0,\epsilon) \epsilon$ or $p_j(x^0,y^*(x^0)) < -l_{p_j}(x^0,\epsilon) \epsilon$. Firstly, if $\lambda_j(x^0) > l_{\lambda_j}(x^0,\epsilon) \epsilon$, for any $x \in \mathcal{B}(x^0,\epsilon)$, $|\lambda_j(x)-\lambda_j(x^0)| \leq l_{\lambda_j}(x^0,\epsilon) \|x-x^0\| \leq l_{\lambda_j}(x^0,\epsilon) \epsilon$. Then, for any $x \in \mathcal{B}(x^0,\epsilon)$, $\lambda_j(x) > 0$, also $p_j(x^0,y^*(x^0)) =0$, which is obtained by the KKT conditions. Then $j \in J^{+}(x,y^*(x))$.
Secondly, if $p_j(x^0,y^*(x^0)) < -l_{p_j}(x^0,\epsilon) \epsilon$, then for any $x \in \mathcal{B}(x^0,\epsilon)$, $| p_j(x^0,y^*(x^0))- p_j(x,y^*(x))| \leq l_{p_j}(x^0,\epsilon) \|x-x^0 \| \leq l_{p_j}(x^0,\epsilon) \epsilon$. Then, for any $x \in \mathcal{B}(x^0,\epsilon)$, $p_j(x,y^*(x)) <0$, and $j \not\in J(x,y^*(x))$. Thus, for any $x \in \mathcal{B}(x^0,\epsilon)$, $J^{0}(x,y^*(x))$ is empty, which implies the SCSC holds at $y^{*}(x)$ w.r.t. $\lambda(x)$ for any $x \in \mathcal{B}(x^0,\epsilon)$. By part (\romannumeral3) of Theorem \ref{th2}, $y^{*}(x)$ is differentiable on $\mathcal{B}(x^0,\epsilon)$. Then $\Phi(x)$ is differentiable on $\mathcal{B}(x^0,\epsilon)$.

\end{proof}

\end{document}